\documentclass[twoside]{article}
\usepackage[accepted]{aistats2026}
\usepackage{natbib}

\usepackage[utf8]{inputenc} %
\usepackage[T1]{fontenc}    %
\usepackage{hyperref}       %
\usepackage{url}            %
\usepackage{booktabs}       %
\usepackage{amsfonts}       %
\usepackage{nicefrac}       %
\usepackage{xcolor}         %
\usepackage{amssymb}
\usepackage{amsthm}
\usepackage{esvect}
\usepackage{amsmath}
\usepackage{mathtools}
\usepackage{cleveref}

\Crefname{ALC@unique}{Line}{Lines}
\newcounter{myalg}
\AtBeginEnvironment{algorithmic}{\refstepcounter{myalg}}
\makeatletter
\@addtoreset{ALC@unique}{myalg}
\makeatother

\usepackage{algorithm}
\usepackage{algorithmic}
\usepackage[old]{old-arrows}
\usepackage[left,inline,final]{showlabels}
\usepackage{etoc}
\usepackage{setspace}
\usepackage{wrapfig}

\usepackage{tikz}

\usepackage{color}

\usepackage{bm}
\usepackage[font=small,labelfont=bf]{caption}
\usepackage{thm-restate}
\usepackage{mathrsfs}
\usepackage{xspace}
\usepackage{enumitem}
\definecolor{mygreen}{rgb}{0.0, 0.5, 0.0}
\definecolor{myorange}{rgb}{0.55, 0.62, 1}

\definecolor{niceRed}{RGB}{190,38,38}
\definecolor{Red2}{RGB}{219, 50, 54}
\definecolor{mgreen}{RGB}{160, 200, 140}
\definecolor{blueGrotto}{RGB}{5,157,192}
\definecolor{limeGreen}{HTML}{81B622}
\definecolor{myellow}{rgb}{0.88,0.61,0.14}
\definecolor{darkGreen}{HTML}{2E8B57}
\definecolor{navyBlueP}{HTML}{03468F}
\definecolor{Sepia}{HTML}{7F462C}
\definecolor{red2}{HTML}{1F462C}
\definecolor{orange2}{HTML}{FF8000}
\definecolor{mgray}{HTML}{ABB3B8}
\definecolor{lgray}{HTML}{E5E8E9}
\definecolor{myPurple}{RGB}{175,0,124}
\definecolor{mypurple2}{rgb}{0.8,0.62,1}
\definecolor{royalBlue}{HTML}{057DCD}
\definecolor{mpink}{HTML}{FC6C85}
\definecolor{lblue}{RGB}{74,144,226}
\definecolor{peagreen}{RGB}{152,193,39}
\definecolor{typ_navy}{HTML}{001f3f}
\definecolor{typ_blue}{HTML}{0074d9}
\definecolor{typ_aqua}{HTML}{7fdbff}
\definecolor{typ_teal}{HTML}{39cccc}
\definecolor{typ_eastern}{HTML}{239dad}
\definecolor{typ_purple}{HTML}{b10dc9}
\definecolor{typ_fuchsia}{HTML}{f012be}
\definecolor{typ_maroon}{HTML}{85144b}
\definecolor{typ_red}{HTML}{ff4136}
\definecolor{typ_orange}{HTML}{ff851b}
\definecolor{typ_yellow}{HTML}{ffdc00}
\definecolor{typ_olive}{HTML}{3d9970}
\definecolor{typ_green}{HTML}{2ecc40}
\definecolor{typ_lime}{HTML}{01ff70}
\definecolor{newgreen}{HTML}{83c702}
\definecolor{newpurp}{RGB}{97,96,121}
\definecolor{Andrea}{RGB}{204, 0, 153}

\hypersetup{
	colorlinks = true,
	linkcolor = typ_blue,
	citecolor = darkGreen,
	linktocpage = true,
	urlcolor = darkGreen
}

\newtheorem{lemma}{Lemma}
\newtheorem{assumption}{Assumption}
\newtheorem{theorem}{Theorem}
\newtheorem{corollary}{Corollary}
\newtheorem{definition}{Definition}

\theoremstyle{definition}

\newcommand{\tto}{\rightrightarrows}
\newcommand{\E}{\mathbb{E}}

\newcommand{\Reals}{\mathbb{R}}
\newcommand{\Naturals}{\mathbb{N}}
\newcommand{\Mcal}{\mathcal{M}}
\newcommand{\Xcal}{\mathcal{X}}
\newcommand{\Fcal}{\mathcal{F}}
\newcommand{\Hcal}{\mathcal{H}}
\newcommand{\Bcal}{\mathcal{B}}
\newcommand{\cl}{\textup{cl}}

\newcommand{\quotes}[1]{``#1''}

\newcommand{\argmin}{\operatornamewithlimits{argmin}}
\newcommand{\argmax}{\operatornamewithlimits{argmax}}

\newcommand{\Prob}{\mathbb{P}}
\newcommand{\Gr}{\textup{Gr}}

\makeatletter
\newcommand{\mapstoto}{\mathpalette\@mapstoto\relax}
\newcommand*{\@mapstoto}[2]{%
    \mathrel{%
      \vcenter{%
         \vbox{%
            \baselineskip\z@skip
            \lineskip\z@
            \ialign{##\cr$#1\mapstochar\varrightarrow$\cr
            $#1\mapstochar\varrightarrow$\cr}%
         }%
      }%
   }%
}
\makeatother

\newcommand{\ie}{\emph{i.e.},\xspace}
\newcommand{\eg}{\emph{e.g.},\xspace}

\crefname{assumption}{Assumption}{Assumptions}

\begin{document}

\twocolumn[

\aistatstitle{Pure Exploration with Infinite Answers}

\aistatsauthor{ Riccardo Poiani \And Martino Bernasconi \And  Andrea Celli }

\aistatsaddress{ Bocconi University \And  Bocconi University \And Bocconi University } ]

\etocdepthtag.toc{mtchapter}
\etocsettagdepth{mtchapter}{subsection}
\etocsettagdepth{mtappendix}{none}

\begin{abstract}
    We study pure exploration problems in which the set of correct answers is possibly infinite. For example, such problems arise when regressing a continuous function on the means of the bandit or when learning Nash equilibria by querying noisy values of the payoff matrix.
    We derive an instance-dependent lower bound for these problems. 
    By analyzing it, we discuss why existing methods (\ie Sticky Track-and-Stop) for finite answer problems fail at being asymptotically optimal in this more general setting. 
    Finally, we present a framework, Sticky-Sequence Track-and-Stop, which generalizes both Track-and-Stop and Sticky Track-and-Stop, and that enjoys asymptotic optimality.   
    Due to its generality, our analysis also highlights special cases where existing methods enjoy optimality.
\end{abstract}

\section{INTRODUCTION}

In \emph{pure exploration} problems, an agent sequentially interacts with a set of $K \in \mathbb{N}$ probability distributions denoted by $\bm\nu = \left( \nu_k \right)_{k \in [K]} \in \mathcal{Q}$ modeling the outcome of $K$ different experiments, where $\mathcal{Q}$ is an arbitrary set of problems. The main goal of the agent is answering a given question about these distributions as efficiently as possible, \ie using the least possible amount of samples. 
Let $\mathcal{X}$ be an answer space for the question at hand; then, for each possible $\bm\nu \in \mathcal{Q}$, a set-valued function (a.k.a.~\emph{correspondence} \citep{aubin1999set}) maps each possible instance $\bm\nu$ to a set of correct answers $\mathcal{X}^{\star}(\bm\nu) \subseteq \mathcal{X}$. The agent is then given a maximum risk parameter $\delta \in (0,1)$ and has to return a correct answer $x \in \mathcal{X}^{\star}(\bm\nu)$ with probability at least $1-\delta$, while minimizing the number of interactions.

This framework models a broad range of settings, with the most extensively studied being the \emph{Best-Arm Identification} (BAI) problem \citep{even2002pac}. Here, the answer space is $\Xcal=[K]$ and the unique correct answer is described by the single-valued correspondence $\mathcal{X}^{\star}(\bm\nu) = \argmax_{k \in [K]} \mu_k$, where $\bm\mu = \left( \mu_k \right)_{k \in [K]}$ denotes the means of the distributions in $\bm\nu$.\footnote{For any $n\in \Naturals$ we denote by $[n]$ the set $\{1,\ldots, n\}$.} In the seminal work by \citet{garivier2016optimal}, the authors derived an information-theoretic lower bound showing that, in the \emph{unstructured} bandit setting, any algorithm requires at least a certain number of samples in order to identify the best arm with high probability.
Furthermore, the authors proposed the Track-and-Stop (TaS) algorithm, which achieves optimal sample complexity rates in the high confidence regime of $\delta \rightarrow 0$. The key idea behind TaS is to exploit \emph{oracle weights}, which are probability distributions over arms that represents the optimal sampling strategy for an algorithm with full knowledge of the instance $\bm\nu$.
TaS mimics this oracle algorithm by tracking the oracle weights of an empirical estimate of the instance $\bm\mu$.
Subsequent work has extended these asymptotic optimality results to BAI problems with additional structure on the instance means $\bm\mu$ \citep{moulos2019optimal,degenne2020gamification,kocak2020best,jourdan2021efficient,russac2021b,vannella2023best,carlsson2024pure,poiani2024optimal,poiani2024best,russo2025pure,tuynman2025batch,lazzaro2025fixed}. 
More generally, several works have shown how to leverage these techniques to build asymptotically optimal algorithms for arbitrarily structured problems where there is a \emph{single} correct answer \citep{degenne2019non,menard2019gradient,juneja2019sample,wang2021fast,kaufmann2021mixture}.
For instance, asymptotically optimal results are available for the broad class of \emph{sequential partition identification problems}, where each instance $\bm\nu$ belongs to an element of a partition of the set of instances $\mathcal{Q}$, and the goal lies in identifying the index of the partition that $\bm\nu$ belongs to.
A key aspect of most of these studies is that they rely on the \emph{well-behaved} nature of oracle weights as functions of the instance.

These properties no longer hold in the broader setting of problems with \emph{multiple correct answers}, where the correspondence $\mathcal{X}^{\star}(\bm\mu)$ denoting the set of correct answers is no longer single-valued. 
Asymptotic optimality for such problems has been studied by \citet{degenne2019pure} in the case in which the set of possible answers $\mathcal{X}$ is finite.\footnote{This formulation also models problems like $\epsilon$-best arm identification, further explored in  \citet{kocak2021epsilon,jourdan2022choosing,jourdan2023varepsilon}.} 
Specifically, the statistical lower bound is now expressed as a minimum over all the multiple correct answers of the lower bound for single-answer problems. Let $\mathcal{X}_F(\bm\nu) \subseteq \mathcal{X}^{\star}(\bm\nu)$ denote the subset of correct answers that attain this minimum. Intuitively, when there are multiple correct answers, some answers might be statistically easier to identify than others.
$\Xcal_F(\bm\nu)$ is precisely the set of the ``easiest'' correct answers. %
However, the presence of multiple answers in $\mathcal{X}_F(\bm\nu)$ introduces topological challenges that hinder the direct application of the Track-and-Stop algorithm.
\citet{degenne2019pure} solve this issue by introducing the Sticky-Track-and-Stop (Stiky-TaS) algorithm, which first identifies a statistically convenient correct answer, \ie one that belongs to $\mathcal{X}_F(\bm\nu)$, and then sticks to it by tracking its corresponding oracle weights (which, for a fixed answer, exhibit the ``nice'' properties required to prove asymptotic optimality). %
Crucially, both the way Sticky-TaS selects an answer and its ability to stick to it heavily depend on the fact that $|\mathcal{X}|$ is finite.

In this work, we drop the assumption that $|\mathcal{X}|$ is finite and study the more general setting in which the answer set $\mathcal{X}$ may be infinite.  Our analysis adopts an asymptotically optimal perspective, focusing on the challenges that arise when the set of correct answers $\mathcal{X}^{\star}(\bm\mu)$ is infinite. 
This model captures fundamental applications that are currently underexplored and not yet fully understood in the bandit literature, such as the problem of regressing a continuous function of the bandit means. To give a concrete application example, consider a pricing problem where each arm corresponds to posting a certain price, and the observed reward is a Bernoulli variable indicating whether the item was purchased. A company may wish to estimate, up to accuracy $\epsilon > 0$, the revenue associated with the optimal price, or the maximum revenue gap that exists between the optimal and the worst price. In these cases, both the answer space and the space of correct answers are infinite sets.

\subsection{Contributions}

We focus on what we call \emph{regular} pure exploration problems. Intuitively, these are problems in which alternative models for an answer $x \in \Xcal $ are \quotes{stable} for nearby answers $x'$ (\ie belonging to some neighborhood of $x$).
This definition (see \Cref{sec:ass} for the precise statement) imposes only minimal and natural requirements. Indeed, it encompasses all cases where the correspondence $\bm\mu \mapstoto \Xcal^{\star}(\bm\mu)$ is continuous (\Cref{theo:ass-conv-and-continuity}). This continuity property is satisfied in fundamental problems such as the problem of regressing an arbitrary continuous function of $\bm\mu$.

After introducing regular pure exploration problems, we present an asymptotic lower bound on the number of samples that are required to identify a correct answer in $\Xcal^{\star}(\bm\mu)$ (\Cref{theo:lb}). %
Then, in \Cref{sec:property}, we analyze the properties of this lower bound. In particular, we study continuity properties of the oracle weights, the lower bound itself, and the mapping $\mathcal{X}_F(\bm\mu)$, and we analyze their algorithmic implications for infinite-answer problems. We argue that the presence of infinite answers  makes it impossible to select and track the empirical oracle weights for a \emph{single} correct answer in $\mathcal{X}_F(\bm\nu)$. This undermines the core argument behind asymptotic optimality of Sticky-TaS.

We address this challenge in \Cref{sec:opt}. In particular, we show that it is not necessary to select and stick to a single correct answer. Instead, it suffices to track a sequence of empirical oracle weights associated with a sequence of answers that converges to some (potentially unknown a priori) correct answer in $\mathcal{X}_F(\bm\nu)$. Building on this, we introduce a general framework (Sticky-Sequence Track-and-Stop) which, when equipped with a method for selecting a converging sequence of answers, achieves asymptotic optimality guarantees (\Cref{theo:sticky-seq}).

The main challenge here is constructing a converging sequence of answers by only exploiting a sequence of sets that converges to the unknown set $\mathcal{X}_F(\bm\nu)$. To present challenges and connections with existing algorithms (\ie TaS and Sticky-TaS), we show how this can be achieved according to different topological properties of $\Xcal$ and $\Xcal_F(\bm\nu)$, which we group in four main scenarios.
\textbf{(i)} If $\Xcal_F(\bm\nu)$ is single-valued for all problems in $\mathcal{Q}$, then both TaS and Sticky-TaS already implement a converging sequence of answers, thereby achieving optimality. \textbf{(ii)} When $\Xcal \subset \Reals$, this property is lost by TaS, but the total order on the reals ensures that the sequence of answers selected by Sticky-TaS converges to some $x \in \Xcal_F(\bm\mu)$. \textbf{(iii)} If $|\Xcal_F(\bm\nu)|$ is finite but, for example, $\Xcal \subset \Reals^2$, then neither TaS nor Sticky-TaS guarantees the convergence property. However, this can be ensured by a simple rule that selects the next answer as the closest to the previous one within a suitable confidence region.
\textbf{(iv)} In the general case where the only information available is that $\Xcal \subset \Reals^d$, we propose an algorithm that progressively discretizes the answer space while guiding the selection of answers according to the history of the previously selected ones.

\paragraph{Additional Related Works}
In pure exploration literature, to the best of our knowledge, there exists two works that deals with infinite answer problems \ie \citet{deep2024asymptotically,osogami2025optimal}. In particular, \cite{deep2024asymptotically} focus on unstructured $K=2$ bandit problems with the goal of estimating an interval of width at most $\epsilon$ that contains $\mu_1 - \mu_2$ with high probability.  \cite{osogami2025optimal}, instead, study the problem of estimating the value of the optimal arm, up to an accuracy $\epsilon$, in unstructured bandit settings. Our work is more general in the sense that we deal with correct answer correspondences that comprehend the ones considered in \citet{deep2024asymptotically,osogami2025optimal}. Furthermore, structured bandits are also allowed within our framework.

\section{PRELIMINARIES}

\paragraph{Mathematical Background} 
We denote by $\Delta_n$ the $n$-dimensional simplex. Given $\mathcal{X} \subseteq \Reals^d$ and $x \in \mathcal{X}$, we denote by $\mathcal{B}_{\rho}({x})= \{ x' \in \mathcal{X}: \|x - x' \| \le \rho  \}$ the ball of radius $\rho$ around ${x}$ and for a set $A\subset \Xcal$ we denote by $\Bcal_\rho(A)$ the union of $\Bcal_\rho(x)$ over all $x\in A$. Given a set $\mathcal{X}$, we denote by $\cl(\mathcal{X})$ its closure. %
Now, let $\mathcal{X} \subseteq \Reals^{n_x}$ and $\mathcal{Y} \subseteq \Reals^{n_y}$, we denote by $C: \mathcal{X} \rightrightarrows \mathcal{Y}$ a set-valued function (i.e., \emph{correspondence}) that maps each element $x \in \mathcal{X}$ to a (non-empty) subset $C(x) \subset \mathcal{Y}$ \citep{aubin1999set}. The correspondence $C$ is \emph{upper hemicontinuous} if, for all $x \in \mathcal{X}$, and for every open set $\mathcal{V} \subset \mathcal{Y}$ such that $C(x) \subset \mathcal{V}$, there exists a neighbourhood $\mathcal{U}$ of $x$ such that $C(x')$ is a subset of $\mathcal{V}$ for all $x' \in \mathcal{U}$. Furthermore, $C$ is \emph{lower hemicontinuous} if, for all $x \in \mathcal{X}$, and for every open set $\mathcal{V} \subset \mathcal{Y}$ such that $C(x) \cap \mathcal{V} \ne \emptyset$, there exists a neighbourhood $\mathcal{U}$ of $x$ such that $C(x') \cap \mathcal{V} \ne \emptyset$ for all $x' \in \mathcal{U}$. Finally, a correspondence $C$ is \emph{continuous} if it is upper and lower hemicontinuous.

\paragraph{Learning Model}
The learner has $K \in \mathbb{N}$ possible choices, each associated with a probability distribution $\nu_k$ over $\Reals$. We denote by $\bm{\nu}=(\nu_k)_{k\in[K]}$ the vector of distributions which we refer to as the bandit model. 
Let $\bm\mu=(\mu_k)_{k\in[K]}$ be the vector of means of the distributions $\nu_k$, where each $\mu_k=\mathbb{E}_{R\sim \nu_k}[R]$ is the expected value under distribution $\nu_k$.
In this work, we consider distributions that belong to a canonical exponential family \citep{cappe2013kullback}.\footnote{These distributions include, \eg Gaussian distributions with known variance and Bernoulli distributions. We refer the interested reader to \Cref{sec:exponential_fam} for additional details on canonical exponential families. For concentration purposes, we also require that these distributions are sub-Gaussian.} 
Conveniently, since these distributions are fully characterized by their means, we may, with a slight abuse of notation, refer to $\bm\mu$ as the bandit model.
We denote by $\Theta$ the interval defining the possible means for any arm $\mu$. We make the standard assumption that the exponential family is regular and bounded, meaning that $\Theta$ is a closed interval strictly contained in an open one (see, \eg \cite{degenne2019non,poiani2024best}).
Furthermore, to represent (possible) additional structure on the bandit model, we assume knowledge of a set $\Mcal\subseteq\Theta^K$ defining the set of admissible bandits models that the learner could face, \ie $\bm\mu\in\Mcal$. 
We consider a possibly infinite answer space $\mathcal{X} \subseteq{\Reals^d}$. For each bandit model $\bm\mu \in \Theta^K$, the set of correct answers for $\bm\mu$ is represented by a correspondence $\Xcal^\star:\Theta^K\tto \Xcal$.  The learner interacts repeatedly with the bandit model. During each round $t\in\Naturals$, it selects an action $A_t\in[K]$ and observes an outcome $R_t\sim\nu_{A_t}$. Let $\Fcal_t=\sigma((A_s,R_s)_{s=1}^t)$ be the $\sigma$-field generated by the observations up to time $t$. A learning algorithm takes in input a risk parameter $\delta\in(0,1)$ and is composed of
(i) a $\Fcal_{t-1}$-measurable sampling rule that selects the next action $A_t\in [K]$, 
(ii) a  stopping rule $\tau_\delta$, which is a stopping time with respect to $(\Fcal_t)_{t\in\Naturals}$, and 
(iii) a $\Fcal_{\tau_\delta}$-measurable decision rule that selects a final decision $\hat x_{\tau_\delta}\in \Xcal$.
We say that an algorithm is \emph{$\delta$-correct} if $\Prob_{\bm\mu}(\hat x_{\tau_{\delta}}\notin \Xcal^\star(\bm\mu))\le\delta$ for all $\bm\mu\in\Mcal$. In words, $\delta$-correct algorithms provide, for each $\bm\mu \in \mathcal{M}$, an answer $\hat{x}_{\tau_\delta} \in \mathcal{X}^{\star}(\bm\mu)$ among the correct ones with probability at least $1-\delta$. Among the class of $\delta$-correct algorithms, we look for those minimizing the expected stopping time, that is, $\mathbb{E}_{\bm{\mu}} \left[ \tau_\delta \right] = \sum_{k \in [K]} \mathbb{E}_{\bm{\mu}}[N_{k}(\tau_\delta)]$,
where $N_{k}(t)$ denotes the (random) number of samples collected for action $k \in [K]$ up to time $t$. In the following, we will use $\bm{N}(t)$ to denote the vector $\left( N_1(t), \dots, N_k(t) \right)$ and $\hat{\bm\mu}(t)$ to denote the empirical estimate of $\bm\mu$ at time $t$, \ie $\hat{\mu}_k(t) = N_k(t)^{-1} \sum_{s=1}^t R_s \bm{1}\{A_t=k \}$.

\paragraph{Alternative Models}
For each $x\in\Xcal$, the set of \emph{alternative models} $\lnot x$ is defined as
\(
 \lnot x=\{ \bm{\lambda} \in \mathcal{M}: x \notin \Xcal^{\star}(\bm{\lambda}) \}
\) (see, \eg \cite{degenne2019non,degenne2019pure}).
In words, $\lnot x$ is the set that contains all the bandit models $\bm\lambda$ for which $x$ is \emph{not} a correct answer for $\bm\lambda$.\footnote{W.l.o.g., we assume that $\lnot x \ne \emptyset$ for all $x \in \Xcal$. Indeed, if there exists $\bar{x} \in \Xcal$ such that $\lnot \bar{x} = \emptyset$, then a $\delta$-correct algorithm can trivially return $\bar{x}$ for any $\bm{\mu} \in \mathcal{M}$ without even interacting with the environment.
}
We generalize this concept to any subset of answers $\widetilde{\Xcal} \subseteq \Xcal$. Specifically,
    \(
        \lnot\widetilde{\Xcal} = \{ \bm{\lambda} \in \mathcal{M}: \forall x \in \widetilde{\Xcal}: x \notin \Xcal^{\star}(\bm{\lambda}) \}.
    \)
The set $\lnot \widetilde{\Xcal}$ extends the notion of alternative models to any arbitrary collection of answers, as it requires that \emph{each} answer $x \in \widetilde{\Xcal}$ is not correct for $\bm\lambda$. It directly follows that $\lnot \widetilde{\Xcal} \subseteq \lnot x$ for any $\widetilde{\Xcal}$ such that $x \in \widetilde{\Xcal}$. This generalization plays a crucial role in defining \emph{regular pure exploration problems}, which we introduce in \Cref{sec:ass}.

\paragraph{Divergences}
Finally, we introduce some divergences that are commonly used in pure exploration problems (see \eg \cite{degenne2019non,degenne2019pure}). Intuitively, as we will see in the next section, they are helpful in statistically identifying the correct answers. Let $d(p,q)$ be the KL divergence between distributions with means $p$ and $q$, respectively. Then, for $\bm\mu \in \Theta^K, \Lambda \subseteq \mathcal{M}$ and $\bm\omega \in \Reals^K$, we define the following:
\begin{align*}
     & D(\bm{\mu}, \bm{\omega}, \Lambda) = \inf_{\bm{\lambda} \in \Lambda} \sum_{k \in [K]}  \omega_k d(\mu_k, \lambda_k) 
     \\ & D(\bm{\mu}, \Lambda) = \sup_{\bm{\omega} \in \Delta_K} D(\bm{\mu}, \bm{\omega}, \Lambda) \\ & D(\bm{\mu}) = \sup_{x \in \Xcal^{\star}(\bm{\mu})} D(\bm{\mu}, \lnot x).
\end{align*}
Let us introduce $\Xcal_F(\bm\mu) = \argmax_{x \in \Xcal^{\star}(\bm\mu)} D(\bm\mu, \neg x)$. We will later see that $\Xcal_F(\bm\mu)$ represents the subset of correct answers for $\bm\mu$ that are the \quotes{easiest} to identify. Moreover, for any $\bm\mu \in \Theta^K$ and any set  $\Lambda \subseteq \mathcal{M}$, we denote by $\bm\omega^*(\bm\mu,\Lambda)$ the argmax over $\bm\omega$ of $D(\bm\mu, \bm\omega, \Lambda)$. Finally, $\bm\omega^*(\bm\mu) = \bigcup_{x \in \Xcal_F(\bm\mu)} \bm\omega^*(\bm\mu, \neg x)$ denotes the oracle weights for $\bm\mu$.

\section{LOWER BOUND}

\subsection{Regular Pure Exploration Problems}\label{sec:ass}

We now define the class of \emph{regular pure exploration problems}, which are characterized by a set of regularity assumptions detailed below. All the problems considered in this paper belong to this class.

\begin{assumption}[Compactness]\label{ass:compact}
    $\Xcal$ is compact and $\bm\mu \mapstoto \Xcal^{\star}(\bm\mu)$ is compact-valued.
\end{assumption}

\begin{assumption}[Identifiability]\label{ass:identify}
    For all $\bm{\mu} \in \mathcal{M}$, there exists $\bar{x} \in \Xcal^{\star}(\bm{\mu})$ such that $\bm{\mu} \notin \cl(\lnot \bar{x})$. 
\end{assumption}

\begin{assumption}[Continuity of $D(\bm{\mu}, \bm{\omega}, \lnot \mathcal{B}_{\rho}(x))$ to $D(\bm{\mu}, \bm{\omega}, \lnot x)$]\label{ass:conv}
For all sufficiently small $\epsilon > 0$, there exists $\rho > 0$ such that, for all $\bm{\mu} \in \Theta^K$, $\bm\omega \in \Delta_K$, $x \in \Xcal$, it holds that $\lnot \mathcal{B}_{\rho}(x)\ne \emptyset$ and $D(\bm{\mu}, \bm{\omega}, \lnot \mathcal{B}_{\rho}(x)) - D(\bm{\mu}, \bm{\omega}, \lnot x) \le \epsilon$. 
\end{assumption}

\Cref{ass:compact} imposes mild regularity conditions, namely compactness, on both the answer space and the correct answer correspondence $\Xcal^\star(\bm\mu)$. \Cref{ass:identify}, instead, is necessary for learnability. When this assumption does not hold, the sample complexity is infinite, even in settings with a finite number of possible answers. This follows directly from the lower bound by \citet{degenne2019pure} (see \Cref{subsec:ass-identify} for further discussion). Finally, at first glance, \Cref{ass:conv} may appear to be a purely technical condition. However, as our analysis reveals, for a subset $\widetilde{\Xcal} \subseteq \Xcal^{\star}(\bm\mu)$, the quantity $D(\bm\mu, \bm\omega, \lnot \widetilde{\Xcal})$ can be related to the complexity of distinguishing $\bm\mu$ from all the models in $\mathcal{M}$ for which \emph{none} of the answers in $\widetilde{\mathcal{X}}$ are correct.
Intuitively, \Cref{ass:conv} implies that distinguishing $\bm\mu$ from $\lnot x$ becomes arbitrarily similar to distinguishing $\bm\mu$ from $\lnot \mathcal{B}_{\rho}(x)$ whenever $\rho$ is small.
In this sense, to demonstrate that \Cref{ass:conv} holds, it suffices to prove a form of \quotes{smoothness} in the alternative models when switching from $\lnot x$ to $\lnot \mathcal{B}_{\rho}(x)$, \ie for all $\bm\lambda \in \lnot x$, there exists $\tilde{\bm\lambda} \in \lnot \mathcal{B}_{\rho}(x)$ such that $\bm\lambda \approx \tilde{\bm\lambda}$ for $\rho \to 0$ (see \Cref{lemma:suff-cond-ass-conv}).
As one might expect, we can show that \Cref{ass:conv} holds whenever $\bm\mu \mapstoto \Xcal^{\star}(\bm\mu)$ is continuous (see \Cref{subsec:ass-conv} for the proof).

\begin{restatable}[Continuous Correspondence Implies \Cref{ass:conv}]{theorem}{asscontinuity}\label{theo:ass-conv-and-continuity}
    If $\bm\mu \mapstoto \Xcal^{\star}(\bm\mu)$ is continuous, and $\mathcal{M}$ and $\Xcal$ are compact sets, then, \Cref{ass:conv} holds. %
\end{restatable}

\paragraph{Uniform vs Local Continuity}
It is important to observe that we require \Cref{ass:conv} to hold uniformly over $\Theta\times\Delta_K\times\Xcal$.
If \Cref{ass:conv} only required local continuity, then \Cref{theo:ass-conv-and-continuity} would follow almost directly from the reasoning presented earlier.
Showing uniform continuity is more challenging, as we need to show the existence of a $\rho$ for which the conditions hold uniformly across all choices of $\bm\lambda$.
At this point, one might wonder why we relied on \Cref{ass:conv} rather than directly assuming that $\Xcal^{\star}(\bm\mu)$ is continuous. The key point is that \Cref{ass:conv} allows us to fully generalize prior results for the finite-answers setting. Indeed,  \citet{degenne2019pure} allow $\Xcal^{\star}(\bm\mu)$ to be discontinuous in $\bm\mu$. Nevertheless, \Cref{ass:conv} always holds for problems with finite possible answers (see \Cref{app:exe-finite}), showing that we can properly generalize \cite{degenne2019pure}. 

\paragraph{Examples} We now give some examples of regular pure exploration problems. First, as anticipated above, we can deal with arbitrary finite answer problems (\Cref{app:exe-finite}). Secondly, given $\epsilon > 0$ and any continuous function $f: \Theta \to \Xcal$ (\eg the maximum), we can consider the problem of estimating $f(\bm\mu)$ up to an accuracy level $\epsilon$.\footnote{Note that, thanks to this result, we are able to show that regular exploration problems, under the assumption of a compact domain, encompasses the problem studied in \cite{deep2024asymptotically,osogami2025optimal}.}
In this case we have $\Xcal^{\star}(\bm\mu) = \{ x \in \Xcal: \|f(\bm\mu) - x \|_{\infty} \le \epsilon \}$. In \Cref{app:regression} we prove that, \Cref{ass:compact,ass:identify,ass:conv} holds in these problems.\footnote{We introduced the problem using the $\ell_\infty$-norm, but other norms could also be considered.} 
Furthermore, given two distinct learning problems defined by the correspondences $\Xcal^{\star}_1$ and $\Xcal^{\star}_2$ for which \Cref{ass:compact,ass:identify,ass:conv} hold, one can prove that \Cref{ass:compact,ass:identify,ass:conv} hold for the learning problem defined by the product correspondence $\Xcal^{\star}(\bm\mu) = \{ (x_1, x_2): x_1 \in \Xcal^{\star}_1(\bm\mu), x_2 \in \Xcal^{\star}_2(\bm\mu) \}$ (see \Cref{app:circuit}). Thus, we can combine arbitrary finite answer problems with regression problems. 
Finally, we show in \Cref{app:nash} that the problem of learning an $\epsilon$-Nash equilibrium in two player stochastic zero-sum games \citep[\eg][]{zhou2017identify,maiti2023instance} can be formulated as a regular pure exploration problem.

\subsection{Sample Complexity Lower Bound}

We now present the lower bound for infinite-answer problems, whose proof is deferred to \Cref{app:lb}.

\begin{restatable}[Lower Bound]{theorem}{lowerbound}\label{theo:lb}
    For any $\bm{\mu} \in \mathcal{M}$, and any $\delta$-correct algorithm it holds that:
    \begin{align}
        \liminf_{\delta \rightarrow 0} \frac{\E_{\bm{\mu}}[\tau_\delta]}{\log(1/\delta)} \ge T^*({\bm{\mu}}) = \frac{1}{D(\bm{\mu})}.
    \end{align}
\end{restatable}

\Cref{theo:lb} provides an asymptotic lower bound on $\E_{\bm\mu}[\tau_\delta]$ that holds for any $\delta$-correct algorithm. 
By explicitly writing $D(\bm\mu)={\sup_{x\in\Xcal^\star(\bm\mu)}\sup_{\bm\omega\in\Delta_k}\inf_{\bm\lambda\in\lnot x}\sum_{k\in[K]}\omega_kd(\mu_k,\lambda_k)}$ we see that the lower bound of \Cref{theo:lb} is expressed as a max-min game, where the max player chooses both a correct answer $x$ within $\Xcal^{\star}(\bm\mu)$ and a strategy $\bm\omega$ over the arm space, and the min player chooses an alternative model $\bm\lambda$ for which $x$ is not a correct answer.
For this reason, answers in $\Xcal_F(\bm\mu)$ can be regarded as the statistically easiest correct answers to be identified (formally, we will prove that the $\sup$ over $\Xcal^\star(\bm\mu)$ is actually attained and therefore $\Xcal_F(\bm\mu)$ is well defined). Finally, we mention that \Cref{theo:lb} nicely generalizes the lower bound for multiple (but finite) correct answers of \cite{degenne2019pure}. We note that the proof of the lower bound of \cite{degenne2019pure} was explicitly using the fact that $\Xcal^{\star}(\bm\mu)$ is finite. Proving \Cref{theo:lb} thus required ad-hoc arguments to extend the result to our setting. For space constraints, we provide further details on this point in \Cref{app:lb}.

\section{ON REGULAR PURE EXPLORATION PROBLEMS}\label{sec:property}
We now present properties of the divergences and discuss their implication for existing algorithms, \ie Track-and-Stop \citep{garivier2016optimal} and Sticky Track-and-Stop \citep{degenne2019pure}. 
\begin{restatable}[Continuity]{lemma}{continuity}\label{lemma:continuity}
    The following holds:
    \begin{enumerate}[label=(\roman*)]
        \item The function $( \bm{\mu}, \bm{\omega}, x) \rightarrow D(\bm{\mu}, \bm{\omega}, \neg x)$ is continuous over $\Theta^K \times \Delta_K \times \Xcal$.
        \item The function $(\bm{\mu}, x) \rightarrow D(\bm{\mu}, \neg x)$ is continuous over $\Theta^K \times \Xcal$ and $(\bm{\mu},x) \rightrightarrows \bm{\omega}^*(\bm{\mu}, \neg x)$ is upper hemicontinuous and compact-valued.
        \item The function $(\bm{\mu}, \bm{\omega}) \rightarrow \max_{x \in \mathcal{X}^{\star}(\bm{\mu})} D(\bm{\mu}, \bm{\omega}, \neg x)$ is continuous over $\Theta^K \times \Delta_K$ and $(\bm{\mu}, \bm{\omega}) \rightrightarrows \argmax_{x \in \mathcal{X}^{\star}(\bm{\mu})} D(\bm{\mu}, \bm{\omega}, \neg x)$ is upper hemicontinuous and compact-valued.
        \item The function $\bm{\mu} \rightarrow D(\bm{\mu})$ is continuous over $\Theta^K$. Moreover, $\bm{\mu} \rightrightarrows \bm{\omega}^{\star}(\bm{\mu})$ and $\bm{\mu} \rightrightarrows \Xcal_F(\bm{\mu})$ are upper hemicontinuous and compact-valued over $\mathcal{S}$.
    \end{enumerate}
\end{restatable}

It is well-known that results analogous to \Cref{lemma:continuity} play a crucial role in the design of optimal algorithms.
For instance, \citet{degenne2019pure} exploited similar results for the case of problems with finite sets of answers.

Although properties \emph{(ii)} and \emph{(iii)} play a crucial role in designing optimal algorithms, their derivation is relatively standard. In contrast, it is interesting to consider properties \emph{(i)} and \emph{(iv)}, for which we need to make at least three important considerations.
First, proving point \emph{(i)} requires novel arguments as, contrary to \cite{degenne2019pure}, we have to guarantee that $D(\bm\mu, \bm\omega, \neg x)$ is jointly continuous on the product space $\Theta^K \times \Delta_K \times \Xcal$. The joint continuity is crucial here, since it is then used to prove all the other claims within \Cref{lemma:continuity}. On a technical level, the main idea to prove \emph{(i)} is combining the joint continuity of $(\bm\mu, \bm\omega) \to D(\bm\mu, \bm\omega, \neg x)$ for all $x \in \Xcal$ (see \cite[Theorem~4]{degenne2019pure}) together with \Cref{ass:conv}. To this end, it is important to highlight again that \Cref{ass:conv} holds uniformly across the domain. This allows us to prove joint continuity by using that $D(\bm\mu, \bm\omega, \neg x)$ is separately continuous in $\bm\mu$ and $\bm\omega$ for a fixed $x$.
Second, it is interesting to highlight that, in point \emph{(iv)}, the continuity of $\bm\mu \to D(\bm\mu)$ and, more surprisingly, the upper hemicontinuity of $\Xcal_F(\bm\mu)$ do not require the continuity of $\bm\mu \mapstoto \Xcal^{\star}(\bm\mu)$ but only the joint continuity of $D(\bm\mu, \bm\omega, \lnot x)$. Indeed, when studying $\max_{x \in \Xcal^{\star}(\bm\mu)} D(\bm\mu, \neg x)$, if $\Xcal^{\star}(\bm\mu)$ is not continuous, we cannot directly apply Berge's maximum theorem to prove the continuity of $D(\bm\mu)$. 
However, we can observe that $\max_{x \in \Xcal^{\star}(\bm\mu)} D(\bm\mu, \neg x) = \max_{x \in \Xcal} D(\bm\mu, \neg x)$ (\Cref{lemma:consequence-ass-compact}) and then use Berge's theorem ($\Xcal$ is a constant correspondence and thus trivially continuous).

\subsection{Failure of Sticky Track-and-Stop with infinite answers}

\begin{figure*}[t]
\centering
\tikzset{every picture/.style={line width=0.75pt}} %

\begin{tikzpicture}[x=0.75pt,y=0.75pt,yscale=-1,xscale=1]
\input{settings/colors}
\draw    (60,200) -- (597,200) ;
\draw [shift={(600,200)}, rotate = 180] [fill={rgb, 255:red, 0; green, 0; blue, 0 }  ][line width=0.08]  [draw opacity=0] (8.93,-4.29) -- (0,0) -- (8.93,4.29) -- cycle    ;
\draw  [fill={rgb, 255:red, 110; green, 155; blue, 155 }  ,fill opacity=0.22 ] (70,90) -- (190,90) -- (190,190) -- (70,190) -- cycle ;
\draw  [fill={rgb, 255:red, 110; green, 155; blue, 155 }  ,fill opacity=0.22 ] (200,90) -- (320,90) -- (320,190) -- (200,190) -- cycle ;
\draw  [fill={rgb, 255:red, 110; green, 155; blue, 155 }  ,fill opacity=0.22 ] (330,90) -- (450,90) -- (450,190) -- (330,190) -- cycle ;
\draw  [fill=typ_blue, fill opacity=0.38 ] (100,100) .. controls (120,90) and (170.18,105.89) .. (170,120) .. controls (169.82,134.11) and (149.02,149.31) .. (160,160) .. controls (165.99,165.83) and (153.56,177.19) .. (137.77,181.46) .. controls (124.6,185.02) and (109.1,183.64) .. (100,170) .. controls (80,140) and (80,110) .. (100,100) -- cycle ;
\draw  [fill=typ_blue  ,fill opacity=0.38 ] (250,110) .. controls (270,100) and (290.58,107.09) .. (290,120) .. controls (289.42,132.91) and (295.05,144.38) .. (300,150) .. controls (304.95,155.62) and (270,190) .. (250,160) .. controls (230,130) and (230,120) .. (250,110) -- cycle ;
\draw  [fill=typ_blue  ,fill opacity=0.38 ] (385,115) .. controls (400.11,106.29) and (420.18,107.09) .. (420,120) .. controls (419.82,132.91) and (410.42,135.71) .. (415,140) .. controls (419.58,144.29) and (405,185) .. (385,155) .. controls (365,125) and (369.89,123.71) .. (385,115) -- cycle ;
\draw  [fill=typ_yellow  ,fill opacity=1 ] (119.02,120) .. controls (129.07,119.88) and (143.97,124.62) .. (139.02,130) .. controls (134.07,135.38) and (134.07,134.88) .. (139.02,140) .. controls (143.97,145.12) and (129.07,149.88) .. (119.02,150) .. controls (108.97,150.12) and (108.97,120.12) .. (119.02,120) -- cycle ;
\draw  [fill=typ_yellow  ,fill opacity=1 ] (259.02,120) .. controls (269.07,119.88) and (283.97,124.62) .. (279.02,130) .. controls (274.07,135.38) and (274.07,134.88) .. (279.02,140) .. controls (283.97,145.12) and (269.07,149.88) .. (259.02,150) .. controls (248.97,150.12) and (248.97,120.12) .. (259.02,120) -- cycle ;
\draw  [fill=typ_yellow  ,fill opacity=1 ] (390,120) .. controls (400.05,119.88) and (414.95,124.62) .. (410,130) .. controls (405.05,135.38) and (405.05,134.88) .. (410,140) .. controls (414.95,145.12) and (400.05,149.88) .. (390,150) .. controls (379.95,150.12) and (379.95,120.12) .. (390,120) -- cycle ;
\draw    (170,120) ;
\draw [shift={(170,120)}, rotate = 0] [color={rgb, 255:red, 0; green, 0; blue, 0 }  ][fill={rgb, 255:red, 0; green, 0; blue, 0 }  ][line width=0.75]      (0, 0) circle [x radius= 2.01, y radius= 2.01]   ;
\draw    (300,150) ;
\draw [shift={(300,150)}, rotate = 0] [color={rgb, 255:red, 0; green, 0; blue, 0 }  ][fill={rgb, 255:red, 0; green, 0; blue, 0 }  ][line width=0.75]      (0, 0) circle [x radius= 2.01, y radius= 2.01]   ;
\draw    (420,120) ;
\draw [shift={(420,120)}, rotate = 0] [color={rgb, 255:red, 0; green, 0; blue, 0 }  ][fill={rgb, 255:red, 0; green, 0; blue, 0 }  ][line width=0.75]      (0, 0) circle [x radius= 2.01, y radius= 2.01]   ;
\draw  [fill={rgb, 255:red, 110; green, 155; blue, 155 }  ,fill opacity=0.22 ] (456,90) -- (576,90) -- (576,190) -- (456,190) -- cycle ;
\draw  [fill=typ_blue  ,fill opacity=0.38 ] (515,115) .. controls (519.87,111.57) and (539.18,118.12) .. (540,125) .. controls (540.82,131.88) and (540.05,139.38) .. (545,145) .. controls (549.95,150.62) and (513.53,160.43) .. (505,145) .. controls (496.47,129.57) and (510.13,118.43) .. (515,115) -- cycle ;
\draw  [fill=typ_yellow  ,fill opacity=1 ] (516,120) .. controls (526.05,119.88) and (540.95,124.62) .. (536,130) .. controls (531.05,135.38) and (531.05,134.88) .. (536,140) .. controls (540.95,145.12) and (526.05,149.88) .. (516,150) .. controls (505.95,150.12) and (505.95,120.12) .. (516,120) -- cycle ;
\draw    (545,145) ;
\draw [shift={(545,145)}, rotate = 0] [color={rgb, 255:red, 0; green, 0; blue, 0 }  ][fill={rgb, 255:red, 0; green, 0; blue, 0 }  ][line width=0.75]      (0, 0) circle [x radius= 2.01, y radius= 2.01]   ;
\draw [color={rgb, 255:red, 255; green, 0; blue, 0 }  ,draw opacity=1 ]   (140,125) ;
\draw [shift={(140,125)}, rotate = 45] [color={rgb, 255:red, 255; green, 0; blue, 0 }  ,draw opacity=1 ][line width=0.75]    (-5.59,0) -- (5.59,0)(0,5.59) -- (0,-5.59)   ;
\draw [color={rgb, 255:red, 255; green, 0; blue, 0 }  ,draw opacity=1 ]   (140,145) ;
\draw [shift={(140,145)}, rotate = 45] [color={rgb, 255:red, 255; green, 0; blue, 0 }  ,draw opacity=1 ][line width=0.75]    (-5.59,0) -- (5.59,0)(0,5.59) -- (0,-5.59)   ;
\draw [color={rgb, 255:red, 255; green, 0; blue, 0 }  ,draw opacity=1 ]   (280,125) ;
\draw [shift={(280,125)}, rotate = 45] [color={rgb, 255:red, 255; green, 0; blue, 0 }  ,draw opacity=1 ][line width=0.75]    (-5.59,0) -- (5.59,0)(0,5.59) -- (0,-5.59)   ;
\draw [color={rgb, 255:red, 255; green, 0; blue, 0 }  ,draw opacity=1 ]   (280,145) ;
\draw [shift={(280,145)}, rotate = 45] [color={rgb, 255:red, 255; green, 0; blue, 0 }  ,draw opacity=1 ][line width=0.75]    (-5.59,0) -- (5.59,0)(0,5.59) -- (0,-5.59)   ;
\draw [color={rgb, 255:red, 255; green, 0; blue, 0 }  ,draw opacity=1 ]   (410,125) ;
\draw [shift={(410,125)}, rotate = 45] [color={rgb, 255:red, 255; green, 0; blue, 0 }  ,draw opacity=1 ][line width=0.75]    (-5.59,0) -- (5.59,0)(0,5.59) -- (0,-5.59)   ;
\draw [color={rgb, 255:red, 255; green, 0; blue, 0 }  ,draw opacity=1 ]   (410,145) ;
\draw [shift={(410,145)}, rotate = 45] [color={rgb, 255:red, 255; green, 0; blue, 0 }  ,draw opacity=1 ][line width=0.75]    (-5.59,0) -- (5.59,0)(0,5.59) -- (0,-5.59)   ;
\draw [color={rgb, 255:red, 255; green, 0; blue, 0 }  ,draw opacity=1 ]   (535,125) ;
\draw [shift={(535,125)}, rotate = 45] [color={rgb, 255:red, 255; green, 0; blue, 0 }  ,draw opacity=1 ][line width=0.75]    (-5.59,0) -- (5.59,0)(0,5.59) -- (0,-5.59)   ;
\draw [color={rgb, 255:red, 255; green, 0; blue, 0 }  ,draw opacity=1 ]   (535,145) ;
\draw [shift={(535,145)}, rotate = 45] [color={rgb, 255:red, 255; green, 0; blue, 0 }  ,draw opacity=1 ][line width=0.75]    (-5.59,0) -- (5.59,0)(0,5.59) -- (0,-5.59)   ;

\draw (516,202.4) node [anchor=north west][inner sep=0.75pt]    {$t_{4}$};
\draw (121,202.4) node [anchor=north west][inner sep=0.75pt]    {$t_{1}$};
\draw (256,202.4) node [anchor=north west][inner sep=0.75pt]    {$t_{2}$};
\draw (386,202.4) node [anchor=north west][inner sep=0.75pt]    {$t_{3}$};
\draw (172,123.4) node [anchor=north west][inner sep=0.75pt]  [font=\small]  {$x_{t_{1}}{}$};
\draw (302,153.4) node [anchor=north west][inner sep=0.75pt]  [font=\small]  {$x_{t_{2}}{}$};
\draw (422,123.4) node [anchor=north west][inner sep=0.75pt]  [font=\small]  {$x_{t_{3}}{}$};
\draw (547,148.4) node [anchor=north west][inner sep=0.75pt]  [font=\small]  {$x_{t_{4}}{}$};

\end{tikzpicture}
\caption{Even though the sets $\Xcal_t$ (in \textbf{\textcolor{typ_blue}{blue}}) are progressively shrinking toward $\Xcal_F(\bm\mu)$ (in \textbf{\textcolor{myellow!90!}{yellow}}), the answers selected $x_t$ could oscillate between one of the two correct answers marked by the \textbf{\textcolor{red}{red}} crosses.%
} 
\label{fig:failureTAS}
\end{figure*}

We now discuss why the Sticky-TaS algorithm \citep{degenne2019pure} is not optimal in the infinite-answer setting. In particular, while conditions similar to those in \Cref{lemma:continuity} were sufficient to establish its optimality in the finite-answer case, they are no longer sufficient when the answer space is infinite. This discussion will underscore the fundamental differences between the finite and infinite-answer settings.
\begin{algorithm}[!t]
	\caption{Sticky-TaS \citep{degenne2019pure}}
	\label{alg:TAS}
	\begin{algorithmic}[1]
        \REQUIRE{Total order over $\Xcal$, exploration funct. $g(t)$}
		\STATE{\textbf{Sampling Rule}}
        \STATE{$C_t = \{ \bm\mu' \in \Mcal: D(\hat{\bm\mu}(t), \bm{N}(t), \bm\mu') \le \log(g(t)) \} $}
        \STATE{$\Xcal_t = \bigcup_{\bm\mu' \in C_t} \Xcal_F(\bm\mu')$}
        \STATE{Pick $x_t \in \Xcal_t$ according to the total order over $\Xcal$}\label{line:pick-xt}
        \STATE{Compute $\bm\omega(t) \in \argmax_{\bm\omega \in \Delta_K} D(\hat{\bm\mu}(t), \bm\omega, \neg x_t)$}
        \STATE{Let $\tilde{\bm\omega}(t)$ be the projection of $\bm\omega(t)$ onto $\Delta_K^{\epsilon_t} = \Delta_K \cap [\epsilon_t, 1]^K$}   
        \STATE{$A_{t} \in \argmax_{k \in [K]} \sum_{s=1}^t \omega_k(s) - N_k(t)$}
        \STATE{\textbf{Stopping Rule}}
        \STATE{
        Stop if $\beta_{t,\delta} < \max_{x \in \Xcal^{\star}(\hat{\bm\mu}(t))} D(\hat{\bm\mu}(t), \bm{N}(t), \neg x)$
        }
        \STATE{\textbf{Recommendation Rule}}
        \STATE{
        $\hat{x}_{\tau_\delta} \in \argmax_{x \in \Xcal^{\star}(\hat{\bm\mu}(t))} D(\hat{\bm\mu}(t), \bm{N}(t), \neg x)$  
        }
	\end{algorithmic}
\end{algorithm}

First, we first recall how Sticky-TaS works (pseudocode in \Cref{alg:TAS}). During each round $t \in \mathbb{N}$, the algorithm defines a confidence region $C_t$ around $\hat{\bm\mu}(t)$, using a suitable exploration function $g: \mathbb{N} \to \mathbb{R}$. Then, it computes a set of candidate answers $\Xcal_t$ using models in $C_t$ and selects an answer $x_t \in \Xcal_t$ according to a pre-specified total order over $\Xcal$. Once this is done, it computes the empirical oracle weights for answer $x_t$ and it applies the C-Tracking \citep{garivier2016optimal} sampling rule on the sequence $\{ \tilde{\bm\omega}(s) \}_{s=1}^t$, where $\tilde{\bm\omega}(s)$ denotes the $l_{\infty}$ projection of $\bm\omega(s)$ onto $\Delta_K^{\epsilon_t}$ and $\epsilon_t = (4(t+K^2))^{-1/2}$. The algorithm then stops using the condition $\beta_{t,\delta} < \max_{x \in \Xcal^{\star}(\hat{\bm\mu}(t))} D(\hat{\bm\mu}(t), \bm{N}(t), \neg x)$ and returns any answer that attains the argmax in $D(\hat{\bm\mu}(t), \bm{N}(t), \neg x)$. Here, $\beta_{t,\delta}$ is a calibrated threshold \citep{kaufmann2021mixture} which ensures $\delta$-correctness.

The upper hemicontinuity of $\bm\mu \mapstoto \Xcal_F(\bm\mu)$ (point \emph{(iv)} of \Cref{lemma:continuity}) ensures that $\Xcal_t\subseteq \Bcal_\epsilon(\Xcal_F(\bm\mu))$ for any $\epsilon > 0$ and sufficiently large $t$ (indeed, $C_t$ will concentrate around $\bm\mu$ under a good event). 
This property holds regardless of whether $\Xcal$ is finite or infinite. However, when $\Xcal$ is finite, we can essentially identify $\Bcal_\epsilon(\Xcal_F(\bm\mu))$ with $\Xcal_F(\bm\mu)$ and, therefore, we can say that $\Xcal_t=\Xcal_F(\bm\mu)$ for sufficiently large $t$. Thus, Sticky TaS sticks to a fixed $x \in \Xcal_t$ thanks to the pre-specified total order of $\Xcal$.
However, when $\Xcal$ is not finite, one can only ensure that $\Xcal_t \subseteq \mathcal{B}_{\epsilon}(\Xcal_F(\bm\mu))$. As a consequence, the algorithm can fail at \emph{sticking} to a single answer as the total order over $\Xcal$ might select answers that progressively disappear from $\Xcal_t$ (thus breaking the optimality proof of Sticky-TaS). Furthermore, a total order over $\Xcal$ might even select answers in $\Xcal_t$ in a way that prevents any sort of converging behavior to a single answer in $\Xcal_F(\bm\mu)$.
In \Cref{fig:failureTAS} we present an example of this problem and in the next section we show how to circumvent it.\footnote{One might argue that, although $x_t$ is not converging, the corresponding oracle weights $\bm\omega(t)$ could still lead to optimality. In \Cref{app:experiments} we show with a theoretical example and an experiment that this is not the case.}

\section{AN OPTIMAL FRAMEWORK}\label{sec:opt}

In order to address the upper hemicontinuity issue discussed above, the main idea is selecting a sequence of candidate answers $x_t \in \Xcal$ such that, under a good event, $x_t$ progressively \emph{converges} to some answer $\bar{x} \in \Xcal_F(\bm\mu)$.\footnote{One might be tempted to solve the convergence issue of $x_t$ by blindly discretizing the answer space, and then applying Sticky-TaS on the resulting finite answers problem. In \Cref{app:experiments}, we argue that this naive method does not attain asymptotic optimality.}
Indeed, we will show that this guarantees asymptotically optimal rates. The rest of this section is structured as follows. In \Cref{sec:algo}, we introduce a general framework (Sticky Sequence TaS) that works with any \emph{converging selection rule} (\Cref{def:conv-answer}), and we present its theoretical guarantees. In \Cref{sec:conv}, we show how to implement converging selection rules.

\subsection{Sticky Sequence Track-and-Stop}\label{sec:algo}
Sticky Sequence Track-and-Stop shares the same pseudocode of Sticky Track-and-Stop (\ie \Cref{alg:TAS}) with one major difference. Specifically, rather than using a pre-specified total order over $\Xcal$ to select an answer $x_t \in \Xcal_t$ (\Cref{line:pick-xt}), it uses a \emph{convergent selection rule}, which is defined as follows.

\begin{definition}[Convergent selection rule]\label{def:conv-answer}
A selection rule is said to be \emph{convergent} if the sequence $\{x_t\}_{t\in\Naturals }$ it generates satisfies the following property: for every $\epsilon > 0$, there exists a time $T_{\epsilon} \in \mathbb{N}$ such that, for all $t \geq T_{\epsilon}$, under the good event $\mathcal{E}(t) = \bigcap_{s \ge h(t)}^{t} \{ \bm\mu \in C_s \}$, there exists $\bar{x} \in \Xcal_F(\bm\mu): \| x_s - \bar{x} \| \le \epsilon,~\forall s \ge h(t)=\lceil \sqrt{t}\rceil$.
\end{definition}

Intuitively, a convergent selection rule guarantees that, under the good event, the sequence $\{ x_t \}_t$ stays close to a correct answer $\bar x\in\Xcal_F(\bm\mu)$. This is not a property guaranteed by \emph{any} selection rule. For example, the selection rule employed by Sticky-TaS is not a convergent one when the answer set $\Xcal$ is infinite, but it is convergent when $\Xcal$ is finite (we provide more details on this in \Cref{sec:conv}).
We discuss in \Cref{sec:conv} how to implement selection rules that satisfies \Cref{def:conv-answer}. First, we prove that Sticky Sequence Track-and-Stop is asymptotically optimal whenever $\{x_t\}_t$ satisfies \Cref{def:conv-answer}.

\begin{restatable}{theorem}{convanswer}\label{theo:sticky-seq}
    Sticky Sequence Track-and-Stop, equipped with a convergent selection rule, is $\delta$-correct and asymptotically optimal, \ie $\limsup_{\delta \to 0} \frac{\E_{\bm\mu}[\tau_\delta]}{\log(1/\delta)} \le T^*(\bm\mu)$.
\end{restatable}

\paragraph{Proof Sketch of \Cref{theo:sticky-seq}} 
First, we observe that the stopping and recommendation rules lead to a $\delta$-correct algorithm for any sampling rule even when the answer space is infinite, \ie for any sampling rule $(A_t)_{t\ge1}$ we have that $\mathbb{P}_{\bm{\mu}} (\hat{x}_{\tau_\delta} \notin \Xcal^\star(\bm{\mu}))  \le \delta$ for all $\bm{\mu} \in \mathcal{M}$ (\Cref{lemma:correctness}). This result follows from standard concentration arguments. We now discuss how to prove asymptotic optimality. The proof, as usual in the literature, proceeds by analyzing the behavior of the algorithm under the sequence of events $\{ \mathcal{E}(t) \}_t$ defined above. Specifically, we show that (i) $\sum_{t=0}^{\infty}\mathbb{P}_{\bm\mu}(\mathcal{E}(t)^c)$ is finite and (ii) there exists $T_0(\delta)$ such that $\mathcal{E}(t) \subseteq \{ \tau_\delta \le t \}$ for any $t \ge T_0(\delta)$, where $T_0(\delta)$ is such that $T_0(\delta) / \log(1/\delta) \to T^*(\bm\mu)$ for $\delta \to 0$. Indeed, whenever these two conditions hold, one can prove asymptotic optimality with standard arguments. While (i) is well-known to be finite (see \cite{degenne2019pure}), the crucial part is proving (ii). To this end, we recall that the algorithm stops as soon as it holds  $\max_{x \in \Xcal_F(\hat{\bm\mu}(t))}D(\hat{\bm\mu}(t), \bm{N}(t), \neg x) \ge {\beta_{t,\delta}}$. 
Now, suppose that under the good event $\mathcal{E}(t)$ it holds:
\begin{equation}\label{eq:tmp1}
\max_{x \in \Xcal_F(\hat{\bm\mu}(t))}D(\hat{\bm\mu}(t), \bm{N}(t), \neg x) \gtrsim tD(\bm\mu).
\end{equation}
Let $T_0(\delta)$ be the first $t \in \mathbb{N}$ such that $tD(\bm\mu) \gtrsim \beta_{t,\delta}$ is satisfied. Then, for $t \ge T_0(\delta)$, we have $\mathcal{E}_t \subseteq \{ \tau_\delta \le t \}$. Furthermore, it is also possible to show that $T_0(\delta) / \log(1/\delta) \to T^*(\bm\mu)$ for $\delta \to 0$. Thus, it only remains to show that \Cref{eq:tmp1} holds under the good event $\mathcal{E}(t)$, which is the key step in the analysis.
To this end, we will use three main ingredients. First, under event $\mathcal{E}(t)$, $\hat{\bm\mu}(t)\approx \bm\mu$ and $\bm\mu'(s)\approx \bm\mu$ for all $s \ge h(t)$. Second, C-tracking guarantees $\bm N(t) \approx \sum_{s=1}^t\bm\omega(s)$ (\Cref{lemma:tracking}). Finally, and most importantly, $x_s$ converges to some $\bar x\in\Xcal^\star(\bm\mu)$. %
These three key properties, along with the continuity results from \Cref{lemma:continuity}, allow us to prove that:%
\begin{align*}
     \hspace{-0.3cm}\max_{x \in \Xcal^{\star}(\hat{\bm\mu}(t))}\hspace{-0.2cm} D(\hat{\bm\mu}(t), \bm{N}(t), \neg x) %
    &\gtrsim D\left(\bm\mu, \textstyle{\sum\limits_{s = 1}^t \bm\omega(s), \neg \bar{x}}\right) \\
    &\ge \sum_{s =h(t)}^t D(\bm\mu, \bm\omega(s), \neg \bar{x}) \\ %
    &\gtrsim  \sum_{s =h(t)}^t D(\bm\mu'(s), \bm\omega(s), \neg {x}_s) 
\end{align*}
where, (i) in the first step, we used $\bm N(t)\approx \sum_{s=1}^t\bm\omega(s)$, $\hat{\bm\mu}(t)\approx \bm\mu$, and $\bar{x} \in \Xcal^{\star}(\bm\mu)$ (ii), in the second one, we moved the inf inside the summation and we used the fact that $D \ge 0$ to get rid of the first $h(t)$ steps, and (iii) in the last one the convergence of $x_s$ to $\bar{x}$ and $\bm\mu'(s) \approx \bm\mu$ for all $s$ such that $s \ge h(t)$. Finally, we note that $D(\bm\mu'(s),\bm\omega(s),\lnot x_s)=D(\bm\mu(s)')$ and thus that $D(\bm\mu'(s),\bm\omega(s),\lnot x_s)\gtrsim D(\bm\mu)$, thus concluding the proof.

We conclude by highlighting two important properties.

\paragraph{Generalization of existing algorithms} It is important to highlight that Sticky Sequence Track-and-Stop generalizes both TaS and Sticky-TaS. Precisely, if $x_t$ is chosen as a point in $ \Xcal_F(\hat{\bm\mu}(t))$, we obtain the TaS algorithm. While if $x_t$ is selected according to a pre-specified total order over $\Xcal$, then we obtain Sticky-TaS. In other words, both TaS and Sticky-TaS are selecting a sequence of answers $x_t$ and collecting data to (eventually) identify $x_t$ as a correct answer. However, as we discuss in the next section, the main point is that these selection rules can fail to generate a converging sequence, and therefore may fail to guarantee optimality in general settings with infinite answers.

\paragraph{Importance of convergence} 
The proof sketch above critically relies on the fact that, for all $s \ge h(t)$, the points $x_s$ remain close to a \emph{fixed} element $\bar{x} \in \Xcal_F(\bm\mu)$.
If this was not the case, \eg there exists $\bar{x}_s \in \mathcal{X}_F(\bm\mu)$ such that $\bar{x}_s \approx x_s$ (this is always the case for both TaS and Sticky-TaS), then we would not be able to achieve the same result, as in step three we are \quotes{forced} to select a single answer within $\Xcal_F(\bm\mu)$ for \emph{all} $s \ge h(t)$.\footnote{In \Cref{thm:not-conv} (\Cref{app:not-conv}), we provide  guarantees for the case in which the sequence of answers is not a converging one. Given the previous remark, we note that \Cref{thm:not-conv} provides theoretical guarantees on the performance of TaS and Sticky-TaS for all cases where they fail to generate a converging sequence.}

\subsection{Algorithms for Converging Sequences}\label{sec:conv}
In this section, we discuss how to develop selection rules for the candidate answer $x_t$ that ensure that  $\{x_t\}_t$ is a converging sequence. 
Since, under the good event, we have $\Xcal_t \subseteq \mathcal{B}_{\epsilon}(\Xcal_F(\bm\mu))$, the problem reduces to the following: given a sequence of sets $\{ \Xcal_t \}_t$ such that $\Xcal_t \to \Xcal_F(\bm\mu)$, how can we select a sequence of points $x_t \in \Xcal_t$ such that $x_t \to \bar{x}$ for some $\bar{x} \in \Xcal_F(\bm\mu)$?
We now discuss several solutions to this problem depending on the different topological properties of $\Xcal$ and $\Xcal_F(\bm\mu)$. These different cases highlight (i) the challenges in building algorithms that satisfy \Cref{def:conv-answer} and (ii) existing cases where existing methods, such as TaS and Sticky-TaS are asymptotically optimal. We remark that, in the last scenario that we cover at the end of this section, we consider the general setup where $\Xcal$ is any subset of $\mathbb{R}^d$, hence providing a general procedure for building a convergent sequence.  

\paragraph{When $\Xcal_F(\mu)$ is single-valued}  First, when $\bm\mu \mapstoto \Xcal_F(\bm\mu)$ is single-valued for all $\bm\mu \in \Theta^K$, one can easily pick any answer within $\Xcal_t$ to obtain a converging sequence. The underlying reason is that whenever $\bm\mu \mapstoto \Xcal_F(\bm\mu)$ is single-valued, then $\bm\mu \mapstoto \Xcal_F(\bm\mu)$ can be seen as a continuous function of $\bm\mu$.\footnote{Indeed, every upper hemicontinuous and single-valued correspondence is a continuous function.} Then, since, under the good event, $\Xcal_t$ progressively converges to $\Xcal_F(\bm\mu)$, we have that picking any element $x_t \in \Xcal_t$ leads to a converging sequence (\Cref{lemma:unique-conv-seq}).
As a consequence, \Cref{theo:sticky-seq} implies that both TaS and Sticky-TaS are asymptotically optimal whenever $|\Xcal_F(\bm\mu)|$ is unique. The optimality of TaS for multiple-answer problems where $\Xcal_F(\bm\mu)$ is single-valued and $|\Xcal|$ is finite was proved also in \citet{degenne2019pure}. \Cref{theo:sticky-seq} extends the result to infinite-answer problems.

\paragraph{When $\mathcal{X} \subset \mathbb{R}$}
Next, we study the case where $\Xcal \subset \mathbb{R}$ (\eg this is the case when the goal is to estimate the optimal arm up an $\epsilon > 0$ accuracy). 
In this setting, selecting any $x_t \in \Xcal_t$ clearly does not lead to a converging sequence. 
Consider, for instance, the case $x_t \in \Xcal_F(\hat{\bm\mu}(t))$ (\ie the way in which TaS selects candidate answers) and $|\Xcal_F(\bm\mu)| = 2$. Since, the map $\Xcal_F(\bm\mu)$ is only upper hemicontinuous, it might be the case that for models $\bm\mu'$ arbitrarily close to $\bm\mu$, $\Xcal_F(\bm\mu')$ only contains one of these two answers. As a consequence, since $x_t \in \Xcal_F(\hat{\bm\mu}(t))$ and $\hat{\bm\mu}(t) \approx \bm\mu$, one might select different answers at each step $t$, which implies that $\{x_t\}_t$ does not converge.\footnote{This was already spotted in \cite{degenne2019pure}.}
However, a converging sequence can be obtained by picking $x_t \in \argmin_{x \in \Xcal_t} x$ (alternatively, we could also take the max). It is easy to see that, whenever $\Xcal_t \subseteq \mathcal{B}_{\epsilon}(\Xcal_F(\bm\mu))$, then $|x_t - \argmin_{x \in \Xcal_F(\bm\mu)} x| \le \epsilon$ (\Cref{lemma:conv-real}). In this sense, there exists a total order over $\Xcal$ such that the resulting sequence is a converging one. As a consequence, from \Cref{theo:sticky-seq}, we have that Sticky-TaS is optimal whenever $\Xcal \subset \mathbb{R}$. %
However, as we will highlight in the next paragraph, this simple fix fails in dimension 2 and higher.

\paragraph{When $|\Xcal_F(\bm\mu)|$ is finite} Now, consider the case where $|\Xcal_F(\bm\mu)|$ is finite (but $\Xcal$ is possibly infinite).
For the same arguments that we presented above, TaS fails at being optimal in this case (\eg if $|\Xcal_F(\bm\mu)| = 2$, but for points around $\bm\mu$ only one of these two answers is correct). Nevertheless, we now argue that even Sticky-TaS fails at generating a convergent sequence, \ie we cannot simply use a total order as above for picking answers. Consider $\mathcal{X} \subseteq \mathbb{R}^2$ and $\Xcal_F(\bm\mu)=\{x_1, x_2\}$. Then, from the upper hemicontinuity of $\Xcal_F(\bm\mu)$, we have that, for sufficiently large $t$, on $\mathcal{E}(t)$, $\Xcal_t \subseteq \mathcal{B}_{\epsilon}(x_1) \cup \mathcal{B}_{\epsilon}(x_2)$ and $\mathcal{B}_{\epsilon}(x_1) \cap \mathcal{B}_{\epsilon}(x_2) = \emptyset$. Then, selecting $x_t$ using, \eg the lexicographic order over $\mathbb{R}^2$ might lead to oscillating behaviors between points in $\mathcal{B}_{\epsilon}(x_1)$ and $\mathcal{B}_{\epsilon}(x_2)$.\footnote{Consider again \Cref{fig:failureTAS}, and suppose that the only answers in $\Xcal_F(\bm\mu)$ are the red crosses.}
To fix this issue, we can resort to the following selection rule: $x_t \in \argmin_{x \in \Xcal_t} \| x - x_{t-1} \|_{\infty}$. We prove in  \Cref{lemma:conv-finite} that this leads to a converging sequence. Indeed, since for sufficiently large $t$, $\Xcal_t \in \bigcup_{x \in \Xcal_F(\bm\mu)} \mathcal{B}_{\epsilon}(x)$ and $\mathcal{B}_{\epsilon}(x_1) \cap \mathcal{B}_{\epsilon}(x_2) = \emptyset$, by staying close to the previously selected point, we know that the algorithms will select points that are always within $\mathcal{B}_{\epsilon}(\bar{x})$ for some $\bar{x} \in \Xcal_F(\bm\mu)$. %
Interestingly, unlike in the previous cases, this procedure does not rely on exact knowledge of the points to which the algorithm will converge.

\paragraph{All the other cases} Finally, we consider the most general case in which $\Xcal \subset \mathbb{R}^d$. Here, one might be tempted to select again $x_t \in \argmin_{x \in \Xcal_t} \| x - x_{t-1} \|_{\infty}$ in the hope of having again a convergent sequence. 
Nevertheless, suppose that $\Xcal_F(\bm\mu)$ is, \eg the boundary of the unitary ball centered in some $x \in \mathcal{X}$. Then, for sufficiently large $t$, we only have that $\Xcal_t \subseteq \mathcal{B}_{\epsilon}(\Xcal_F(\bm\mu))$. As a result, $x_t$ might ``wander'' indefinitely around $\Xcal_F(\bm\mu)$ thus preventing convergence. To solve this issue, we propose an algorithm that progressively discretizes the answer space $\mathcal{X}$.\footnote{For space constraints, its pseudocode can be found in the appendix (\Cref{alg:general}; \Cref{app:adaptive}); here, we outline only the general idea.}
The key idea it to combine a \emph{progressive discretization} of $\Xcal$, using balls with a vanishing radius $\rho_t$ over time, with a mechanism that incorporates the \emph{history} of previously selected points.
During each iteration, the algorithm constructs and maintains a set $\mathcal{H}_t$ which is composed of tuples $(\bar{x}_s, \rho_s)_{s=1}^{t}$. 
Specifically, each element $(\bar{x}_s, \rho_s)$ represents a region $\mathcal{B}_{\rho_s}(\bar{x}_s)$ in the answer space  in which the algorithm is \quotes{conducting the search} of an answer within $\Xcal_F(\bm\mu)$. More precisely, for all $s \in [1,t]$, the elements in $\mathcal{H}_t$ are such that $\mathcal{B}_{\rho_s}(\bar{x}_s) \cap \mathcal{B}_{\rho_{s-1}}(\bar{x}_{s-1}) \cap \Xcal_t \ne \emptyset$, and $\bar{x}_s$ is the center of a ball radius $\rho_s$ that belongs to a \emph{uniquely} identified finite cover of $\Xcal$, which we denote by $\mathcal{P}_s$. Since $\rho_{s}<\rho_{s-1}$, this implies that, over time, the regions $\mathcal{B}_{\rho_s}(\bar{x}_s)$ become progressively smaller. Given this setup, the algorithm simply selects $x_t$ to be any element within $\mathcal{B}_{\rho_t}(\bar{x}_t) \cap \Xcal_t$.
To guide the search toward previously selected points, the set $\mathcal{H}_t$ is recursively constructed at each iteration as follows. The algorithm first selects the point $(\bar{x}, \bar{\rho})$ in $\mathcal{H}_{t-1}$ with smaller radius and for which it holds that $\mathcal{B}_{\rho_s}(\bar{x}_s) \cap \Xcal_t \ne \emptyset$ for all $(\bar{x}_s, \rho_s) \in \mathcal{H}_{t-1}$ such that $\rho_{s} < \bar{\rho}$ (indeed, as iteration progresses, previously selected balls might not contain any candidate answers within the updated $\Xcal_t$). Then, the algorithm constructs $\mathcal{H}_t$ starting from $(\bar{x}, \bar{\rho})$ in order to satisfy $\mathcal{B}_{\rho_{s}}(\bar{x}_s) \cap \mathcal{B}_{\rho_{s-1}}(\bar{x}_{s-1}) \cap \Xcal_t \ne \emptyset$. Such a procedure can be proven to generate a converging sequence (\Cref{lemma:conv-general}). 
The reason why this happens is that, due to the upper hemicontinuity of $\Xcal_F(\bm\mu)$, $\Xcal_t \in \mathcal{B}_{\epsilon}(\Xcal_F(\bm\mu))$ for sufficiently large $t$. This implies that, for sufficiently small $\epsilon$ and sufficiently large $t$, balls of radius $\mathcal{B}_{\epsilon}(\bar{x})$ that belong to $\mathcal{P}_{{\epsilon}}$ and that do not intersect $\Xcal_F(\bm\mu)$ will not belong to $\mathcal{H}_t$. As a consequence, the algorithm has found a region of radius ${\epsilon}$ that contains an element of $\Xcal_F(\bm\mu)$ and this region will remain in $\mathcal{H}_t$ under the good event.

\section{FUTURE WORK}\label{sec:conclusion}
Our study paves the way for several future research avenues. For example, one might investigate what happens outside of regular pure exploration problems. Are these problems learnable or is \Cref{ass:conv} necessary for finite sample complexity? 
Furthermore, while our algorithm is statistically optimal, we observe that it is not computationally efficient. This limitation is somewhat expected, given that even Sticky-TaS \citep{degenne2019pure} (which deals with a narrower class of problems) has analogous limitations.
A promising direction for future work is to investigate whether efficient algorithms can be developed for \emph{specific} classes of infinite-answer pure exploration problems (\eg the regression of a continuous function $f(\bm\mu)$). Indeed, despite Sticky-TaS being inefficient, there exist algorithms that are both efficient and optimal for the $\varepsilon$-best arm identification problem \citep{jourdan2023varepsilon}. Another possible approach is to focus on improving computational complexity by relaxing asymptotic optimality guarantees and instead targeting $\beta$-optimal algorithms, which are usually more efficient \citep{russo2016simple, qin2017improving, jourdan2022top}.

\subsubsection*{Acknowledgements}
Funded by the European Union. Views and opinions expressed are however those of the author(s) only and do not necessarily reflect those of the European Union or the European Research Council Executive Agency. Neither the European Union nor the granting authority can be held responsible for them.

Work supported by the Cariplo CRYPTONOMEX grant and by an ERC grant (Project 101165466 — PLA-STEER).

\bibliographystyle{plainnat}
\bibliography{bibliography}

\newpage
\onecolumn
\aistatstitle{Supplementary Material: Pure Exploration with Infinite Answers}
\appendix
\thispagestyle{empty}
\etocdepthtag.toc{mtappendix}
\etocsettagdepth{mtchapter}{none}
\etocsettagdepth{mtappendix}{subsection}

\begin{spacing}{0.5}
\tableofcontents
\end{spacing}

\newpage
\section{Examples of Regular Pure Exploration Problems}\label{app:exemples}

In this section, we present several general classes of problems for which \Cref{ass:compact,ass:identify,ass:conv} hold. Specifically, the first class of problems is that of regression of continuous functions (\Cref{app:regression}), the second one encompasses all finite-answers problems (\Cref{app:exe-finite}), and the third one consists of combinations of problems that meet \Cref{ass:compact,ass:identify,ass:conv} (\Cref{app:circuit}). The last class allows one to combine regression of continuous functions together with finite-answers problems. 

Some results of this section rely on \Cref{theo:ass-conv-and-continuity}, which is proved in \Cref{app:assumptions}.
 
\subsection{Regression of Continuous Functions}\label{app:regression}

Consider a continuous function $f: \Theta^K \rightarrow \Xcal$ and let $\epsilon > 0$. Suppose that both $\Theta^K$ and $\Xcal$ are compact Euclidean sets. The goal of the learner is to compute an $\epsilon$-accurate estimate of $f(\bm{\mu})$ with high-probability, namely, $\Prob_{\bm{\mu}}( \| f(\bm{\mu}) - x_{\tau_\delta}\|_{\infty} > \epsilon) \le \delta$ for all $\bm{\mu} \in \Mcal$. Thus, $\Xcal^\star(\bm{\mu}) = \{ x \in \Xcal: \|f(\bm{\mu}) - x \|_{\infty} \le \epsilon \}$.

\begin{lemma}
    \Cref{ass:compact} holds, namely $\mathcal{X}$ is compact and $\bm\mu \mapstoto \Xcal^{\star}(\bm\mu)$ is compact-valued.
\end{lemma}
\begin{proof}
    The set $\Xcal$ is compact by assumption, and $\mathcal{X}^{\star}(\bm\mu) = \{x \in \mathcal{X}: \|f(\bm\mu) - x \|_\infty \le \epsilon \}$ is compact as well.
\end{proof}

We now continue by proving that \Cref{ass:identify} holds.

\begin{lemma}
    \Cref{ass:identify} holds.
    Namely, for all $\bm{\mu} \in \Mcal$, there exists $\bar x \in \Xcal^\star(\bm{\mu})$ such that $\bm{\mu} \notin \cl(\lnot \bar x)$
\end{lemma}
\begin{proof}
    Let $\mathcal{Q}_x = \{ \bm\mu \in \mathcal{M}: \| f(\bm\mu) - x\|_{\infty} < \epsilon \}$. We first prove that $\mathcal{Q}_x \cap \cl(\neg x) = \emptyset$.    
    We recall that a closure of a generic set $\mathcal{S}$ contains all the points in $\mathcal{S}$ together with all the limit points. Now, for all $\bm\mu\in \lnot x$ we have $\|f(\bm\mu)-x\|>\epsilon$, hence $\mathcal{Q}_x \cap \neg x = \emptyset$.
    It remains to prove that no limit point of $\neg x$  belongs to $\mathcal{Q}_x$. Suppose, by contradiction, that $\bm\mu$ is a limit point of $\cl(\neg x)$ and that $\bm\mu \in \mathcal{Q}_x$. Then, then there exists a sequence $\bm\mu_n$ such that $\bm\mu_n \in \neg x$ and $\bm\mu_n\to\bm\mu$. However, by continuity of $f$ we would have $\|f(\bm\mu_n)-x\|_\infty\to\|f(\bm\mu)-x\|_\infty<\epsilon$, thus showing that $\bm\mu_n$ is not in $\lnot x$ for sufficiently large $n$.

    Now, for any $\bm{\mu} \in \Mcal$, let $\bar x = f(\bm{\mu})$. Then, $\bar x \in \Xcal^\star(\bm{\mu})$ since $\|f(\bm{\mu}) - \bar x\|_{\infty} = 0 \le \epsilon$. Furthermore, $\bm{\mu} \notin \cl(\lnot \bar x)$ (since $\bm\mu \in \mathcal{Q}_{\bar{x}}$ and $\mathcal{Q}_{\bar x} \cap \cl(\neg \bar x) = \emptyset$), which concludes the proof.
\end{proof}

Finally, we continue by proving \Cref{ass:conv}. We prove that by simply showing that $\mathcal{X}^{\star}(\bm\mu)$ is a continuous correspondence. 

\begin{lemma}\label{lemma:i-star-cont}
    The correspondence $\bm\mu\mapstoto\Xcal^\star(\bm{\mu})$ is continuous over $\Theta^K$.
\end{lemma}
\begin{proof}
    To prove this statement, we use a sufficient condition for the continuity of the composition of a correspondence and a continuous function. In particular, we can use \Cref{th:composition_correspondence}. Indeed the correspondence $\Xcal^\star(\bm\mu)$ can be seen as $\Xcal^\star(\bm\mu)=\cup_{u\in \Bcal_{\epsilon}(0)}\{f(\bm\mu)+u\}$. Thus, $\Xcal^\star:\Theta^K\rightrightarrows\Xcal$ is continuous.
\end{proof}

\begin{lemma}
\Cref{ass:conv} holds, namely for all $\epsilon > 0$ sufficiently small, there exists $\rho > 0$ such that for all $\bm{\mu} \in \Theta^K$, $\bm\omega \in \Delta_K$, $x \in \mathcal{X}$, it holds that $\neg \mathcal{B}_{{\rho}}(x)\ne \emptyset$ and $D(\bm{\mu}, \bm{\omega}, \neg \mathcal{B}_{\rho}(x)) - D(\bm{\mu}, \bm{\omega}, \neg x) \le \epsilon$.
\end{lemma}
\begin{proof}
    Apply \Cref{lemma:i-star-cont} and \Cref{theo:ass-conv-and-continuity}.
\end{proof}

\subsection{Finite Set of Answers}\label{app:exe-finite}
We show that \Cref{ass:compact,ass:identify,ass:conv} hold in the setting where the answer space is finite \cite{degenne2019pure}. Specifically, we consider $\Mcal \subseteq \Theta^K$ and let $|\Xcal| < +\infty$, and we consider an arbitrary correspondence $\Xcal^{\star}: \Theta^K \rightrightarrows \Xcal$ that models the set of correct answers for each bandit model. Without loss of generality, we only assume that \Cref{ass:identify} holds. Otherwise, one cannot obtain finite sample-complexity results (see \Cref{subsec:ass-identify}). That being said, we now prove that \Cref{ass:compact} and \Cref{ass:conv} hold.
 
\begin{lemma}
    \Cref{ass:compact} holds; namely $\Xcal$ is compact and $\Xcal^*(\bm\mu)$ is compact-valued.
\end{lemma}
\begin{proof}
    $\Xcal$ is compact since it is finite. Therefore, $\Xcal^{\star}(\bm\mu)$ is compact as well since $\Xcal^{\star}(\bm\mu) \subseteq \Xcal$. 
\end{proof}

\begin{lemma}
    \Cref{ass:conv} holds; namely, for all $\epsilon > 0$ sufficiently small, there exists $\rho > 0$ such that for all $\bm{\mu} \in \Theta^K$, $\bm\omega \in \Delta_K$, $x \in \mathcal{X}$, it holds that $\neg \mathcal{B}_{{\rho}}(x)\ne \emptyset$ and $D(\bm{\mu}, \bm{\omega}, \neg \mathcal{B}_{\rho}(x)) - D(\bm{\mu}, \bm{\omega}, \neg x) \le \epsilon$.
\end{lemma}
\begin{proof}
    We first show that there exists $\rho > 0$ such that $\neg \mathcal{B}_{\rho}(x) \ne \emptyset$.
    Since $\Xcal$ is finite, for any $x \in \Xcal$, there exists $\rho_x > 0$ such that $\mathcal{B}_{\rho_x}(x) = \{x\}$. Then, taking ${\rho} = \min_{x \in \Xcal} \rho_x$, we obtain that $\neg \mathcal{B}_{\rho}(x) = \neg x,~\forall x \in \Xcal$. Furthermore, $\neg x \ne \emptyset$ holds by design otherwise one could simply always return such $x$.

    Now, taking ${\rho} = \min_{x \in \Xcal} \rho_x$, for all $\epsilon > 0$ it follows that:
    \begin{align*}
        D(\bm\mu, \bm\omega, \neg \mathcal{B}_\rho(x)) - D(\bm\mu, \bm\omega, \neg x) = 0 \le \epsilon,
    \end{align*}
    thus concluding the proof.
\end{proof}

\subsection{Composition of Learning Problems}\label{app:circuit}
We now consider the composition of learning problems, where, for each bandit $\bm\mu \in \Theta^K$ the set of correct answer $\Xcal^{\star}(\bm\mu)$ is the product correspondence of two correspondences $\Xcal_1^{\star}(\bm\mu) \times \Xcal_2^{\star}(\bm\mu)$, where $\Xcal^{\star}_1(\bm\mu) \subseteq \Xcal_1$, $\Xcal^{\star}_2(\bm\mu) \subseteq \Xcal_2$, and $\Xcal = \Xcal_1 \times \Xcal_2$. Specifically, we consider
\begin{align*}
    \Xcal^{\star}(\bm\mu) = \{x=(x_1, x_2) \in \Xcal: x_1 \in \Xcal^{\star}_1(\bm\mu) \text{ and } x_2 \in \Xcal^{\star}_2(\bm\mu) \} = \Xcal^{\star}_1(\bm\mu) \times \Xcal^{\star}_2(\bm\mu).
\end{align*}

For the sake of the example, consider the case where we want to estimate the $\min_{k \in [K]} \mu_k$ up to an $\epsilon$-factor, together with the problem of finding an arm $k$ that satisfies $\mu_k > \argmax_{j \in [K]} \mu_k - \epsilon$. This problem is a joint combination of a regression problem together with a multiple (but finite) correct answers one. As we show in this section, this sort of problem is regular. 

Indeed, suppose that \Cref{ass:compact}, \Cref{ass:identify} and \Cref{ass:conv} hold for the learning problems defined by $\Xcal^{\star}_1$ and $\Xcal^{\star}_2$ independently. Furthermore, let $\neg_1, \neg_2$ be the alternative set correspondences related to $\mathcal{X}^{\star}_1$ and $\mathcal{X}^{\star}_2$, respectively. Similarly, let $\neg$ be the correspondence related to the product correspondence $\Xcal^{\star}$. Suppose that $\neg_1 x_1 \ne \emptyset$ for all $x_1 \in \mathcal{X}_1$ and $\neg_2 x_2 \ne \emptyset$ for all $x_2 \in \mathcal{X}_2$, so that $\neg x \ne \emptyset$ as well. 

We now show that the learning problem defined by $\Xcal^{\star}$ inherits \Cref{ass:compact}, \Cref{ass:identify} and \Cref{ass:conv} directly from $\Xcal^{\star}_1$ and $\Xcal^{\star}_2$.

\begin{lemma}
    \Cref{ass:compact} holds; namely, $\Xcal$ is compact and $\Xcal^{\star}$ is compact-valued.
\end{lemma}
\begin{proof}
    Given two compact sets $A,B$, $A \times B$ is compact.
    Thus, since $\Xcal_1, \Xcal_2$ are compact, $\Xcal = \Xcal_1 \times \Xcal_2$ is compact as well. Furthermore, since $\Xcal^{\star}(\bm\mu) = \Xcal^{\star}_1(\bm\mu) \times \Xcal^{\star}_2(\bm\mu)$, and $\Xcal^{\star}_1(\bm\mu), \Xcal^{\star}_2(\bm\mu)$ are compact, then $\Xcal^{\star}(\bm\mu)$ is compact as well.
\end{proof}

\begin{lemma}
    \Cref{ass:identify} holds; namely for all $\bm\mu \in \mathcal{M}$, there exists $\bar{x} \in \mathcal{X}^{\star}(\bm\mu)$ such that $\bm\mu \notin \cl(\neg \bar{x})$. 
\end{lemma}
\begin{proof}
    For all $\bm\mu \in \mathcal{M}$, there exists $\bar{x}_1 \in \mathcal{X}_1^{\star}(\bm\mu)$ such that $\bm\mu \notin \cl(\neg_1 \bar{x}_1)$, and, moreover, there exists $\bar{x}_2 \in \mathcal{X}^{\star}_2(\bm\mu)$ such that $\bm\mu \notin \cl(\neg_2 \bar{x}_2)$.

    Now, consider $\bar{x} = (\bar{x}_1, \bar{x}_2) \in \mathcal{X}^{\star}(\bm\mu)$. We verify that $\bar{x} \notin \cl(\neg \bar{x})$. Specifically,
    \begin{align*}
        \cl(\neg \bar{x}) & = \cl\left( \{ \bm\lambda \in \mathcal{M}: \bar{x} \notin \mathcal{X}^{\star}(\bm\mu) \} \right) \\ &
        = \cl\left( \{  \bm\lambda \in \mathcal{M}: \bar{x}_1 \notin \mathcal{X}_1^{\star}(\bm\mu) \text{ or } \bar{x}_2 \notin \mathcal{X}^{\star}_{2}(\bm\mu)  \} \right) \\ & 
        = \cl\left( \{ \bm\lambda \in \mathcal{M}: \bar{x}_1 \notin \mathcal{X}^{\star}_1(\bm\mu) \} \cup \{ \bm\lambda \in \mathcal{M}: \bar{x}_2 \notin \mathcal{X}^{\star}_2(\bm\mu) \} \right) \\ & 
        = \cl(\neg_1 \bar{x}_1) \cup \cl(\neg_2 \bar{x}_2),
    \end{align*}
    where in the last step we used that, for any two sets $A,B$, $\cl(A \cup B) = \cl(A) \cup \cl(B)$.

    Now, since $\bm\mu \notin \cl(\neg_1 \bar{x}_1)$ and $\bm\mu \notin \cl(\neg_2 \bar{x}_2)$, it follows that $\bm\mu \notin \cl(\neg \bar{x})$, thus concluding the proof.
\end{proof}

\begin{lemma}
    \Cref{ass:conv} holds; namely, for all $\epsilon > 0$ sufficiently small, there exists $\rho > 0$ such that $\neg \mathcal{B}_{\rho}(x) \ne \emptyset$ and for all $\bm\mu \in \Theta^K, \bm\omega \in \Delta_K, x \in \mathcal{X}$, it holds that $D(\bm\mu, \bm\omega, \neg \mathcal{B}_{\rho}(x)) - D(\bm\mu, \bm\omega, \neg x) \le \epsilon$.
\end{lemma}
\begin{proof}
    We first show that there exists $\rho$ such that $\neg \mathcal{B}_\rho(x) \ne \emptyset,~\forall x \in \mathcal{X}$.
    Given $x=(x_1,x_2)$, we have:
    \begin{align*}
        \neg \mathcal{B}_{\rho}(x) & = \{ \bm\lambda \in \mathcal{M}: \forall \bar{x} \in \mathcal{B}_{\rho}(x), \bar{x} \notin \mathcal{X}^{\star}(\bm\lambda) \} \\ &
        = \{ \bm\lambda \in \mathcal{M}: \forall \bar{x}=(\bar{x}_1, \bar{x}_2) \in \mathcal{B}_{\rho}(x), \bar{x}_1 \notin \mathcal{X}^{\star}_1(\bm\lambda) \text{ or } \bar{x}_2 \notin \mathcal{X}^{\star}_2(\bm\lambda) \} \\ &
        = \{ \bm\lambda \in \mathcal{M}: \forall \bar{x}_1 \in \mathcal{B}_{\rho}(x_1), \bar{x}_1 \notin \mathcal{X}^{\star}_1(\bm\lambda)\} \cup \{ \bm\lambda \in \mathcal{M}: \forall \bar{x}_2 \in \mathcal{B}_{\rho}(x_2), \bar{x}_2 \notin \mathcal{X}^{\star}_2(\bm\lambda)\} \\ & 
        = \left( \neg_1 \mathcal{B}_\rho(x_1) \right) \cup \left( \neg_2 \mathcal{B}_{\rho}(x_2) \right).
    \end{align*}
    Now, since \Cref{ass:conv} holds for $\mathcal{X}^{\star}_1, \mathcal{X}^{\star}_2$, there exists $\bar{\rho}_1, \bar{\rho}_2 > 0$ such that $\neg_1 \mathcal{B}_{\rho}(x_1), \neg_2 \mathcal{B}_\rho(x_2)$ are both non empty for all $\rho \le \min\{ \bar{\rho}_1, \bar{\rho}_2 \}$.\footnote{Notice that one could take also the maximum of $\bar{\rho}_1, \bar{\rho}_2$, since $\neg \mathcal{B}_{\rho}(x)$ is the union of the two $\neg_1 \mathcal{B}_{\rho}(x_1)$ and $\neg_2 \mathcal{B}_\rho(x_2)$.}
    
    Now, we continue by analyzing the difference in the divergences. Proceeding as above, we have that:
    \begin{align*}
        & \neg x = \neg_1 x_1 \cup \neg_2 x_2 \\ 
        & \neg \mathcal{B}_{\rho}(x) = \neg_1 \mathcal{B}_{\rho}(x_1) \cup \neg_2 \mathcal{B}_{\rho}(x_2).
    \end{align*}
    Now, let $\epsilon > 0$, and set ${\rho} = \min \{ {\rho}_1, \rho_2 \}$, where $\rho_1$ and $\rho_2$ are such that:
    \begin{align}
        & D(\bm\mu, \bm\omega, \neg_1 \mathcal{B}_{\rho_1}(x_1)) - D(\bm\mu, \bm\omega, \neg_1 x_1) \le \epsilon \quad \forall \bm\mu \in \Theta^K, \bm\omega \in \Delta_K, x_1 \in \mathcal{X}_1 \label{eq:ass-4-proof-circuit} \\
        & D(\bm\mu, \bm\omega, \neg_2 \mathcal{B}_{\rho_2}(x_2)) - D(\bm\mu, \bm\omega, \neg_2 x_2) \le \epsilon \quad \forall \bm\mu \in \Theta^K, \bm\omega \in \Delta_K, x_2 \in \mathcal{X}_2
    \end{align}
    
    Then, let $\bm\mu \in \Theta^K$, $\bm\omega \in \Delta_K$, $x=(x_1, x_2) \in \mathcal{X}$.
    For readability, since $\bm\mu, \bm\omega$ are fixed, we omit them from the notation $D(\cdot, \cdot, \cdot)$ in the rest of this proof. Suppose without loss of generality that $D(\neg_1 x_1) \le D(\neg_2 x_2)$.
    It holds that:
    \begin{align*}
        D(\neg \mathcal{B}_{\rho}(x)) - D(\neg x) & = D(\neg_1 \mathcal{B}_{\rho}(x_1) \cup \neg_2 \mathcal{B}_{\rho}(x_2) ) - D(\neg_1 x_1 \cup \neg_2 x_2) \\ & 
        = D(\neg_1 \mathcal{B}_{\rho}(x_1) \cup \neg_2 \mathcal{B}_{\rho}(x_2) ) - \min\{ D(\neg_1 x_1), D(\neg_2 x_2) \} \\ & 
        = D(\neg_1 \mathcal{B}_{\rho}(x_1) \cup \neg_2 \mathcal{B}_{\rho}(x_2) ) - D(\neg_1 x_1) \\ & 
        \le D(\neg_1 \mathcal{B}_{\rho}(x_1) ) - D(\neg_1 x_1) \\
        & \le  D(\neg_1 \mathcal{B}_{\rho_1}(x_1) ) - D(\neg_1 x_1) \\ & \le \epsilon
    \end{align*}
    where (i) in the first step we used the decomposition of $\neg$ into $\neg_1$ and $\neg_2$, (ii) in the second one, \Cref{lemma:split-optim}, (iii) in the third one $D(\neg_1 x_1) \le D(\neg_2 x_2)$, (iv) in the forth one $\neg_1 \mathcal{B}_\rho(x_1) \subseteq \neg_1 \mathcal{B}_\rho(x_1) \cup \neg_2 \mathcal{B}_\rho(x_2)$, and (v) in the last step \Cref{eq:ass-4-proof-circuit}.
\end{proof}

\subsection{$\epsilon$-Nash Equilibrium in Two-player Zero-sum Games}\label{app:nash}
We now consider the problem of learning an $\epsilon$-Nash equilibrium point in a two player zero-sum stochastic matrix game \citep{zhou2017identify,maiti2023instance,maiti2024near,maiti2025open}. Specifically, we consider the following setting. We let $\Theta^K = \mathcal{M} = [-1,1]^{n \times n}$ where $n \in \mathbb{N}$ denotes the number of actions for each player and we suppose that each distribution is \eg Gaussian.\footnote{The following argument can easily be extended to the $n \times m$ setting and to other bounded domains outside the $[-1,1]$ interval provided that $\Theta$ is a closed interval strictly contained in an open one.} Each point $\bm\mu \in \Mcal$ denotes the mean of a stochastic payoff matrix whose distributions are in the underlying canonical exponential family. At each round, the learner controls both the min player and the max player by selecting an action $(I_t, J_t) \in n \times n$, and it receives as payoff $R_t \sim \nu_{i,j}$, where $\nu_{i,j}$ is the distribution with mean $\mu_{i,j}$. In the rest of this section, we slightly overload the notation and we denote by $x$ the min player and by $y$ the max player. Let $\mathcal{X} = \Delta_n \times \Delta_n$ be the answer space. Then, we define the agent's goal as follows:
\begin{align}\label{eq:nash-eq1}
    \mathcal{X}^\star(\bm\mu) = \left\{ (x,y) \in \Delta_n^2: \max_{j \in [n]} (x^\top \bm\mu)_j - \min_{i \in [n]} (\bm\mu y )_j \le \epsilon \right\},
\end{align}
where $\epsilon > 0$ is the desired accuracy parameter and $(\bm\mu y)_i = \sum_{j \in [n]} y_j \mu_{i,j}$ and $(x^\top \bm\mu)_j = \sum_{i \in [n]} x_i \mu_{i,j}$.

We now show that this learning task is a regular pure exploration problem.

\begin{lemma}
    \Cref{ass:compact} holds; namely, $\mathcal{X}$ is compact and $\mathcal{X}^\star$ is compact valued.
\end{lemma}
\begin{proof}
    Recall that $\mathcal{X} = \Delta_n \times \Delta_n$, which is clearly compact, and $\mathcal{X}^\star(\bm\mu)$ is compact as well.
\end{proof}

\begin{lemma}
    \Cref{ass:identify} holds; namely, for all $\bm\mu \in \mathcal{M}$, there exists $(\bar x, \bar y) \in \mathcal{X}^\star(\bm\mu)$ such that $\bm\mu \notin \cl(\neg (\bar x, \bar y))$.
\end{lemma}
\begin{proof}
    Fix $\bm\mu \in \mathcal{M}$ and consider $(x^\star, y^\star) \in \Delta_n^2$ such that $(x^\star, y^\star)$ is an exact Nash equilibrium for $\bm\mu$ (guaranteed to exists for each payoff matrix $\bm\mu$). At $(x^\star, y^\star)$ it holds that $(x^\star)^\top \bm\mu e_j \le (x^\star)^\top \bm\mu y^\star \le e_i^\top \bm\mu y^\star$. This, in turns, implies that, $(x^\star, y^\star) \in \Xcal^\star(\bm\mu)$. We now prove that $\bm\mu \notin \cl(\neg (x^\star, y^\star))$. Recall that:
    \begin{align*}
        \cl(\neg (x^\star, y^\star)) & = \cl\left( \left\{ \bm\lambda \in \mathcal{M}: (x^\star, y^\star) \notin \mathcal{X}^\star(\bm\lambda) \right\} \right) \\
        & = \cl\left( \left\{ \bm\lambda \in \mathcal{M}: \max_{j \in [n]} ( (x^\star)^\top \bm\lambda )_j - \min_{i \in [n]} (\bm\lambda y^\star)_i > \epsilon \right\} \right),
    \end{align*}
    and proceed by contradiction. Assume that $\bm\mu \in \cl(\neg (x^\star, y^\star))$. Then, there must exist a sequence $\{\bm\mu_n\}_n \to \bm\mu$ such that $\bm\mu_n \ne \bm\mu$ and $\bm\mu_n \in \cl(\neg (x^\star, y^\star)) ~\forall n$. However, if $\bm\mu_n \to \bm\mu$, the continuity of $\bm\lambda \to \max_{j \in [n]} ( (x^\star)^\top \bm\lambda )_j - \min_{i \in [n]} (\bm\lambda y^\star)_i$, would imply that 
    \begin{align*}
        \max_{j \in [n]} ( (x^\star)^\top \bm\mu_n )_j - \min_{i \in [n]} (\bm\mu_n y^\star)_i \to \max_{j \in [n]} ( (x^\star)^\top \bm\mu )_j - \min_{i \in [n]} (\bm\mu y^\star)_i = 0.
    \end{align*}
    However, for all $n$, $\max_{j \in [n]} ( (x^\star)^\top \bm\mu_n )_j - \min_{i \in [n]} (\bm\mu_n y^\star)_i > \epsilon > 0$, thus leading to a contradiction.
\end{proof}

\begin{lemma}\label{lem:nashLHC}
    \Cref{ass:conv} holds; namely, for all $\epsilon > 0$ sufficiently small, there exists $\rho > 0$ such that $\neg \mathcal{B}_\rho((x,y)) \ne \emptyset$ and for all $\bm\mu \in \Theta^K, \bm\omega \in \Delta_K, (x,y) \in \mathcal{X}$, it holds that $D(\bm\mu, \bm\omega, \neg \mathcal{B}_\rho((x,y))) - D(\bm\mu, \bm\omega, \neg (x,y)) \le \epsilon$.
\end{lemma}
\begin{proof}
    We this result using \Cref{theo:ass-conv-and-continuity}. Thus, we only need to prove the continuity of the correspondence $\bm\mu \mapstoto \Xcal^\star(\bm\mu)$.
    We start by proving that \Cref{eq:nash-eq1} is upper hemicontinuous.\footnote{Before reading the proof, the we reader might find useful to read \Cref{app:assumptions}.} To do that, we note that $\mathcal{X}$ is compact and we use \Cref{th:uhc-closed-graph}, which states that $\mathcal{X}^{\star}(\bm\mu)$ is upper hemicontinuous if its graph is closed. Thus, we need to prove that
    \begin{align*}
        \Gr( \Xcal^\star) = \left\{ (\bm\mu, (x,y)) \in [-1,1]^{n \times n} \times \Delta_n^2: \max_{j \in [n]} (x^\top \bm\mu)_j - \min_{i \in [n]} (\bm\mu y )_j \le \epsilon  \right\}
    \end{align*}
    is closed, \ie that it contains all its limit points. Let $(\bar{\bm\mu}, (\bar x, \bar y))$ be any limit point of $\Gr( \Xcal^\star)$. Then, there exists $\{(\bm\mu_n, (x_n, y_n)) \}_n$ such that (i) $(\bm\mu_n, (x_n, y_n)) \in \Gr( \Xcal^\star)$, (ii) $(\bm\mu_n, (x_n, y_n)) \ne (\bar{\bm\mu}, (\bar x, \bar y))$ and (iii) $(\bm\mu_n, (x_n, y_n)) \to (\bar{\bm\mu}, (\bar x, \bar y))$. Let us introduce some notation for brevity. Specifically, we define the following function $(\bm\mu, (x,y)) \to F(\bm\mu, (x,y))$ 
    \begin{align*}
        F(\bm\mu, (x,y)) = \max_{j \in [n]} (x^\top \bm\mu)_j - \min_{i \in [n]} (\bm\mu y )_j. 
    \end{align*}
    Then, let $\bar F = F(\bar{\bm\mu}, (\bar x, \bar y))$ and $F_n = (\bm\mu_n, (x_n, y_n))$.
    Now, suppose that $\bar F = \bar \epsilon > \epsilon$, \ie that $(\bar{\bm\mu}, (\bar x, \bar y)) \notin \Gr(\mathcal{X}^\star(\bm\mu))$). For any $n$ it holds that:
    \begin{align*}
        \bar F & = \bar F - F_n + F_n  
         \le |\bar F - F_n | + \epsilon,
    \end{align*}
    where we used that $F_n \le \epsilon$ since $(\bm\mu_n, (x_n, y_n)) \in \Gr( \Xcal^\star)$. Now, from the continuity of $(\bm\mu, x,y) \to F_n$, we have that $F_n \to \bar F$ as $n \to +\infty$. One can now pick $n$ such that $|\bar F - F_n| \le (\bar \epsilon - \epsilon)/2$, thus leading to $\bar F \le (\bar \epsilon - \epsilon)/2 + \epsilon < \bar \epsilon$, hence contradicting the fact that $\bar F = \bar \epsilon > \epsilon$. Thus, the graph is closed and \Cref{eq:nash-eq1} is upper hemicontinuous.
    
    We now continue by proving the lower hemicontinuity. We start with a preliminary consideration. Specifically, we observe that $\Xcal^{\star}(\bm\mu)$ is a polytope with non-empty interior for all $\bm\mu \in [-1,1]^{n \times n}$. Indeed, we have that:
    \begin{align*}
        \mathcal{X}^\star(\bm\mu) & = \left\{ (x,y) \in \Delta_n^2: \max_{j \in [n]} (x^\top \bm\mu)_j - \min_{i \in [n]} (\bm\mu y )_j \le \epsilon \right\} \\
        & = \left\{ (x,y) \in \Delta_n^2: (x^\top \bm\mu)_j - (\bm\mu y)_i \le \epsilon, \forall i,j \in [n]^2 \right\}.
    \end{align*}
    Now, for a fixed $\bm\mu$, each constraint is linear in $x,y$ and $\Delta_n^2$ is in itself a polytope. We now prove that the interior is non-empty. Here, it is sufficient to consider an exact Nash equilibrium point $(x^\star, y^\star)$; at $(x^\star, y^\star)$ it holds that $(x^\star)^\top \bm\mu e_j \le (x^\star)^\top \bm\mu y^\star \le e_i^\top \bm\mu y^\star$, and thus $F(\bm\mu, (x^\star, y^\star)) = 0 < \epsilon$ and the interior is non-empty. Furthermore, for a fixed $\bm\mu$, $F(\bm\mu, (x,y))$ is convex in $(x,y)$. Indeed, we can rewrite $F$ as follows:
    \begin{align*}
        F(\bm\mu, (x,y)) = \max_{i, j \in [n]^2} \left( (x^\top \bm\mu)_j - (\bm\mu y)_i\right),
    \end{align*}
    which is a maximum over a finite set of linear function, which is known to be convex.
    
    To prove the lower hemicontinuity, we need to show that, for any $\bm\mu \in [-1, 1]^{n \times n}$ and any open set $\mathcal{V} \in \Delta_n^2$ such that $\Xcal^{\star}(\bm\mu) \cap \mathcal{V} \ne \emptyset$, there exists a neighbourhood $U$ of $\bm\mu$ such that $\forall \bm\mu' \in U$, $\Xcal^{\star}(\bm\mu') \cap \mathcal{V} \ne \emptyset$. Fix a neighbourhood $U$ of $\bm\mu$ or radius $\kappa$ (we will pick an appropriate $\kappa$ later in the proof), \ie $\| \bm\mu - \bm\mu' \|_{\infty} \le \kappa$ for all $\bm\mu' \in U$. Then, for any $(x,y) \in \Delta_n^2$ it holds that:
    \begin{align*}
        F(\bm\mu', (x,y) ) & = \max_{j \in [n]} (x^\top \bm\mu')_j - \min_{i \in [n]} (\bm\mu' y )_j \\
        & = \max_{j \in [n]} (x^\top (\bm\mu' - \bm\mu + \bm\mu) )_j - \min_{i \in [n]} ((\bm\mu'- \bm\mu + \bm\mu) y )_j \\
        & \le  F(\bm\mu, (x,y)) + 2\kappa.
    \end{align*}
    Thus, a necessary condition for having $(x,y) \in \mathcal{X}^\star(\bm\mu')$ is that $F(\bm\mu, (x,y)) + 2\kappa \le \epsilon$.
    
    Now, fix any point $(x,y) \in \Xcal^{\star}(\bm\mu) \cap \mathcal{V}$ and let $(x^\star, y^\star)$ be an exact Nash equilibrium such that $F(\bm\mu, (x^\star, y^\star)) = 0$. Define $(x_\alpha, y_\alpha)$ as a convex combination between $(x,y)$ and $(x^\star, y^\star)$ with parameter $\alpha \in [0,1]$. Then, since $\mathcal{X}^\star(\bm\mu)$ is a polytope, $(x_\alpha, y_\alpha) \in \mathcal{X}^\star(\bm\mu)$ and, furthermore, since $(x,y) \in \mathcal{V}$ and $\mathcal{V}$ is open, there exists $\alpha \in (0,1)$ such that $(x_\alpha, y_\alpha) \in \mathcal{V}$ as well. Now, consider $F(\bm\mu, (x_\alpha, y_\alpha))$. By the convexity of $F$, it holds that:
    \begin{align*}
        F(\bm\mu, (x_\alpha, y_\alpha)) \le (1-\alpha) F(\bm\mu, (x,y)) + F(\bm\mu, (x^\star, y^\star)) = (1-\alpha) F(\bm\mu, (x,y)) \le (1-\alpha)\epsilon.
    \end{align*}
    Thus, we have that $(x_\alpha, y_\alpha) \in \mathcal{V} \cap \Xcal^{\star}(\bm\mu')$ for all $\bm\mu' \in U$ if $(1-\alpha)\epsilon + 2\kappa \le \epsilon$. Choosing $\kappa= \frac{\epsilon \alpha}{4}$ concludes the proof.
\end{proof}

\paragraph{Remark on the problem of learning $\epsilon$-Nash equilibrium}

In the proof of \Cref{lem:nashLHC}, for the lower hemicontinuity of $\Xcal^\star(\bm\mu)$, we exploited the fact that the correspondence defines a polytope for each payoff matrix $\bm\mu$. This is true only if we define $\epsilon$-Nash strategies in which the sum of the regrets is less than $\epsilon$, which has the advantage of canceling the value of the game $x^\top\bm\mu y$ (which is bilinear) from the definition. The same trick would not work for general-sum games, in which the sum of the regrets remains bilinear. We leave as future work to handle such cases.

\section{On Regular Pure Exploration Problems}\label{app:assumptions}

In this section, we show that \Cref{ass:identify} is necessary for finite sample complexity results (\Cref{subsec:ass-identify}), and we show that \Cref{ass:conv} holds for continuous correspondences (\Cref{subsec:ass-conv}).

\subsection{On \Cref{ass:identify}}\label{subsec:ass-identify}

In this section, we discuss \Cref{ass:identify}. As we anticipated in the main text, \Cref{ass:identify} is necessary for obtaining finite sample-complexity results. Indeed, even for problems with finitely many possible answers, the failure of \Cref{ass:identify} leads to infinite lower bounds on the sample complexity.
This is a direct consequence of Theorem 1 in \cite{degenne2019pure}. 

Indeed, suppose that $\mathcal{X}$ is finite and that \Cref{ass:identify} does not hold. Then, there exists $\bar{\bm\mu} \in \Mcal$ such that $\bar{\bm\mu} \in \cl(\neg x)$ for all $x \in \Xcal^{\star}(\bar{\bm\mu})$. However, from Theorem 1 in \cite{degenne2019pure}, this would imply that:
\begin{align*}
    & \liminf_{\delta \rightarrow 0} \frac{\E_{\bar{\bm\mu}}[\tau_\delta]}{\log(1/\delta)} \ge T^*(\bar{\bm\mu}) = D(\bar{\bm\mu})^{-1}, \\ 
    & D(\bar{\bm\mu}) = \max_{x \in \Xcal^{\star}(\bar{\bm\mu})} \max_{\bm\omega \in \Delta_K} \inf_{\bm\lambda \in \neg x} \sum_{k \in [K]} \omega_k d(\bar{\mu}_k, \lambda_k).
\end{align*}
However, since $\bar{\bm\mu} \in \cl(\neg x)$ for all $x \in \Xcal^{\star}(\bar{\bm\mu})$, this would lead to $D(\bar{\bm\mu}) = 0$, and hence $T^*(\bar{\bm\mu}) = +\infty$, thus leading to a lower bound with infinite sample complexity.

The above claims are related to the following result, showing that $D(\bm\mu, \neg \widetilde{\mathcal{X}}) > 0$ holds if and only if $\bm\mu \notin \cl(\neg \widetilde{\mathcal{X}})$.

\begin{lemma}[Strictly positive divergence]\label{lemma:iff-positive-div}
    Let $\bm{\mu} \in \mathcal{M}$ and $\widetilde{\Xcal} \subseteq \Xcal$. Then, $D(\bm{\mu}, \neg \widetilde{\Xcal}) > 0$ holds if and only if $\bm{\mu} \notin \cl(\neg \widetilde{\Xcal})$.
\end{lemma}
\begin{proof}

    We prove the first direction by contradiction.
    Suppose that $D(\bm{\mu}, \neg \widetilde{\Xcal}) > 0$, that is $\sup_{\bm{\omega} \in \Delta_K} \inf_{\bm{\lambda} \in \neg \widetilde{\Xcal}} \sum_{k \in [K]} \omega_k d(\mu_k, \lambda_k) > 0$ and suppose by contradiction that $\bm{\mu} \in \text{cl}(\neg \widetilde{\Xcal})$. Then we can take a sequence $\bm\lambda^{j}$ such that $\bm\lambda^j\to \bm\mu$ and, since $\sum_{k\in[K]}\omega_k d(\mu_k,\cdot)$ is continuous for all $\bm\omega\in\Delta_K$ and $\bm\mu\in\Mcal$, we obtain $\sum_{k\in[K]}\omega_k d(\mu_k,\lambda^j_k)\to\sum_{k\in[K]}\omega_k d(\mu_k,\mu_k)=0$, which shows that $D(\bm\mu,\neg\widetilde \Xcal)=0$.

    Now, we prove the second direction. Suppose that $\bm{\mu} \notin \cl(\neg \widetilde{\Xcal})$. Then, there exists $\epsilon > 0$, such that, for all $\bm{\lambda} \in \cl(\neg \widetilde{\Xcal})$, $\| \bm{\mu} - \bm{\lambda} \|_{\infty} \ge \epsilon$. Thus, for all $\bm{\lambda} \in \cl(\neg \widetilde{\Xcal})$, there exists $k \in [K]$ such that $|\mu_k - \lambda_k| > 0$, and consequently, $d(\mu_k, \lambda_k) > 0$. It follows that: 
    \begin{align*}
        D(\bm{\mu}, \neg \widetilde{\Xcal}) \ge \frac{1}{K} \inf_{\bm{\lambda} \in \neg \widetilde{\Xcal}} \sum_{k \in [K]} d(\mu_k, \lambda_k) > 0,
    \end{align*}
    which concludes the proof.
\end{proof}

\subsection{On \Cref{ass:conv}}\label{subsec:ass-conv}

In the following, we show that \Cref{ass:conv} holds under very mild conditions. Specifically, it is sufficient to assume that (i) $\bm\mu \mapstoto \mathcal{X}^{\star}(\bm\mu)$ is a continuous and compact-valued correspondence, and (ii) that $\mathcal{M}$ and $\mathcal{X}$ are compact sets.

The structure of this section is organized as follows. First, in \Cref{subsubsec:alt-models}, we prove that there exists $\rho > 0$ such that $\forall x \in \Xcal$, $\neg \mathcal{B}_{\rho}(x) \ne \emptyset$. As one can note, this is the first statement in \Cref{ass:conv}. Secondly, in \Cref{subsubsec:preliminaries-for-thm1}, we provide some intermediate and helper results that are used to prove the second part of \Cref{ass:conv}. Finally, in \Cref{subsubsec:cont-thm1}, we prove the second part of \Cref{ass:conv}, that is the fact that for all $\epsilon > 0$, there exists $\rho$ such that $D(\bm\mu, \bm\omega, \neg \mathcal{B}_{\rho}(x)) - D(\bm\mu, \bm\omega, \neg x) \le \epsilon$ holds uniformly across models, weights and answers.   

\subsubsection{(Extended) Alternative Models are non-empty, that is $\neg \mathcal{B}_{\rho}(x) \ne \emptyset$}\label{subsubsec:alt-models}

We start by showing that, under the assumptions above there exists ${\rho} > 0$ such that $\neg \mathcal{B}_{\rho}(x) \ne \emptyset$ for all $x \in \mathcal{X}$.

\begin{lemma}\label{lemma:empty-holds}
    There exists ${\rho} > 0$, such that, for all $x \in \Xcal$, $\lnot \Bcal_{\rho}(x) \neq \emptyset$.
\end{lemma}
\begin{proof}
    Let $\bar{x} \in \mathcal{X}$. By definition, $\lnot \Bcal_{\rho}(\bar x)$ is empty if, for all $\bm{\mu} \in \Mcal$, there exists $x \in \Bcal_{\rho}(\bar x)$ such that $x \in \mathcal{X}^{\star}(\bm\mu)$. In the following, we show that this cannot happen for an appropriate value of $\rho$. This value is derived in the following constructive way.
    
    Consider the following function $g: \mathcal{X} \rightarrow \mathbb{R}$ defined as follows:
    \begin{align*}
        g(x) = \sup_{\substack{\bm\lambda \in \mathcal{M}}} \inf_{\tilde{x} \in \mathcal{X}^{\star}(\bm\lambda)} \|x-\tilde{x} \|_{\infty}.
    \end{align*}
        Then, by \Cref{thm:berge}, since $\lambda \mapstoto \mathcal{X}^{\star}(\bm\lambda)$ is continuous and compact-valued, and since $\|x-\tilde{x} \|_{\infty}$ is continuous, we have that $\inf_{\tilde{x} \in \mathcal{X}^{\star}(\bm\lambda)} \| x - \tilde{x}\|_{\infty}$ is continuous as well. Furthermore, the inf is attained since it is an infimum of a continuous function over a compact domain. Furthermore, the supremum is also attained since the inf is continuous and $\mathcal{M}$ is compact. Finally, $g(x)$ is continuous. This follows from the fact $\mathcal{M}$ is constant (and hence a continuous correspondence of $\bm\lambda$) and compact-valued. Then, by further applying \Cref{thm:berge}, we have proved that $x \to g(x)$ is continuous. 

    Now, consider $\bm\lambda_x \in \argmax_{\bm\lambda \in \Mcal} \inf_{\tilde{x} \in \Xcal^{\star}(\bm\lambda)} \| x - \tilde{x} \|_{\infty}$. We note that $\min_{\tilde{x} \in \mathcal{X}^{\star}(\bm\lambda_x)} \| x - \tilde{x} \|_{\infty} > 0$. Indeed, that minimum is $0$ if and only if $x \in \mathcal{X}^{\star}(\bm\lambda_x)$. Nevertheless, for all $x \in \mathcal{X}$, there exists $\bar{\bm\lambda} \in \neg x$, and, therefore, $\min_{\tilde{x} \in \Xcal^{\star}({\bar{\bm\lambda}})} \| x - \tilde{x}  \|_{\infty} > 0$. Thus, we have that:
    \begin{align*}
        \min_{\tilde{x} \in \Xcal^{\star}(\bm\lambda_x)} \| x - \tilde{x} \|_{\infty} \ge \min_{\tilde{x} \in \Xcal^{\star}(\bar{\bm\lambda})} \| x - \tilde{x} \|_{\infty} > 0.
    \end{align*}

    Now, consider $\eta = \inf_{x \in \mathcal{X}} g(x)$. Since $g(x)$ is continuous, and since $\mathcal{M}$ is compact, then the inf is attained, and, as a consequence $\eta > 0$.

    The proof then follows by picking $\rho = \frac{\eta}{2}$. Indeed, suppose that there exists $\bar{x} \in \mathcal{X}$ such that $\neg \mathcal{B}_{\rho}(\bar{x}) = \emptyset$. Then, for all $\bm\mu \in \mathcal{M}$, there exists $x_{\bm\mu} \in \mathcal{X}: \|\bar{x} - x_{\bm\mu}  \|_{\infty} \le \rho$ and $x_{\bm\mu} \in \mathcal{X}^{\star}(\bm\mu)$. However, this would imply that:
    \begin{align*}
        \eta = \min_{x \in \mathcal{X}} g(x) \le \min_{\tilde{x} \in \Xcal^{\star}(\bm\lambda_{\bar{x}})} \|\bar{x} - \tilde{x} \|_{\infty} \le \| \bar{x} - x_{\bm\lambda_{\bar{x}}} \|_{\infty} \le \frac{\eta}{2},
    \end{align*}

     since $\eta > 0$ this leads to a contradiction, thus concluding the proof.
\end{proof}

\subsubsection{Preliminary Results}\label{subsubsec:preliminaries-for-thm1}
In order to continue, we need some intermediate results.
Before proceeding, we recall that, given a set $\mathcal{S} \subseteq \widetilde{\mathcal{{S}}}$, $s \in \mathcal{S}$ is either a limit point of $\mathcal{S}$ or an isolated point. A point is isolated if $s \in \mathcal{S}$ and there exists a neighborhood $\mathcal{U}$ of $s$ such that $\mathcal{U} \cap \mathcal{S} = \{s\}$. On the other hand, a point $s \in \widetilde{S}$ is a limit point of $\mathcal{S}$ if every neighbourhood $\mathcal{U}$ of $s$ contains at least one point of $\mathcal{S}$ different from $s$ itself. When dealing with metric spaces, this is equivalent to saying that there exists a sequence of points in $\mathcal{S} \setminus \{s\}$ whose limit is $s$. 

That being said, we now consider the correspondence $x \mapstoto \cl(\neg x)$. %

\begin{lemma}\label{lemma:cl-cont}
    The correspondence $x \mapstoto \cl(\lnot x)$ is lower hemicontinuous and compact-valued over $\Xcal$. 
\end{lemma}
\begin{proof}
    First, we note that $\cl(\lnot x)$ is compact-valued. The set is trivially bounded and closed. Thus, it is compact by the Heine-Borel theorem.

    We continue by proving that it is lower hemicontinuous. Consider an open set $\mathcal{V}$ and $x\in \Xcal$ such that $\mathcal{V} \cap \cl(\lnot x) \ne \emptyset$. Then, since $\mathcal{V}$ is open and $\cl(\neg x)$ is compact, we have that there exists $\bm\lambda \in \neg x$ such that $\bm\lambda \in \mathcal{V} \cap \cl(\neg x)$. To prove this, we proceed by contradiction. Assume that $\neg x \cap \mathcal{V} = \emptyset$. Then, since there exists $\bm\lambda \in \mathcal{V} \cap \cl(\neg x)$, then it must hold that $\bm\lambda$ is a limit point of $\cl(\neg x)$, i.e., there exists $\{ \bm\lambda_n \}_{n \ge 1}$ such that $\bm\lambda_n \ne \bm\lambda$, $\bm\lambda_n \in \neg x$ and $\bm\lambda_n \to \bm\lambda$. Therefore, for all $\alpha > 0$, there exists $\bar{\bm\lambda} \in \neg x$ and $\| \bm\lambda - \bar{\bm\lambda} \|_{\infty} \le \alpha$. But, then, since $\mathcal{V}$ is open and $\bm\lambda \in \mathcal{V}$, we can take $\alpha$ sufficiently small so that $\bar{\bm\lambda} \in \mathcal{V}$ as well.

    Now, consider $\bm\lambda \in \neg x$ and $\bm\lambda \in \mathcal{V} \cap \neg x$. In the following, we will prove that there exists a neighborhood $\mathcal{B}_{\kappa}(x)$ of $x$ such that $\bm\lambda \in \neg x'$ for all $x' \in \mathcal{B}_{\kappa}(x)$. Define $\kappa \in \mathbb{R}$ as follows:
    \begin{align*}
        \kappa = \frac{1}{2} \min_{\tilde{x} \in \Xcal^{\star}(\bm\lambda)} \|x - \tilde{x} \|_\infty > 0,
    \end{align*}
    where in the inequality step we used $x \notin \Xcal^{\star}(\bm\lambda)$. Now, consider $x' \in \mathcal{X}$ such that $\| x-x'\| \le \kappa$. Then, it holds that $\bm\lambda \in \neg x'$. Indeed, suppose it is false. Then:
    \begin{align*}
        \|x - x'\|_{\infty} \le \kappa = \frac{1}{2} \min_{\tilde{x} \in \mathcal{X}^{\star}(\bm\lambda)} \| x - \tilde{x} \|_{\infty} \le \frac{1}{2} \| x-x' \|_{\infty} <  \| x-x' \|_{\infty},
    \end{align*}
    thus leading to a contradiction. Thus, $\bm\lambda \in \neg x'$ and, by definition $\bm\lambda \in \mathcal{V}$. Thus, $\cl(\neg x') \cap \mathcal{V} \ne \emptyset$ and the correspondence is lower hemicontinuous.
\end{proof}

In the following lemma, we begin by characterizing isolated points of the graph of the correspondence $\cl \neg $. Before that, we introduce some notation.
Consider a correspondence $\phi: \mathbb{X} \rightrightarrows \mathbb{Y}$. Then, for any $Z \subseteq \mathbb{X}$, let $$\Gr_{Z}(\phi)=\{ (x,y) \in Z \times \mathbb{Y}: y \in \phi(x) \}.$$
In the following, we will provide some characterization of $\Gr_{\Xcal}(\cl\lnot)$

\begin{lemma}\label{lemma:isolated-points}
    Let $\Gr_{\Xcal}(\cl\lnot) = \left\{ (x, \bm\lambda) \in \Xcal \times \Mcal: \bm\lambda \in \cl(\lnot x) \right\}$ and $(x, \bm\lambda) \in \Gr_{\Xcal}(\cl \lnot)$ be an isolated point of $\Gr_{\Xcal}(\cl \lnot)$. Then $\bm\lambda$ is an isolated point of $\cl(\lnot x)$.
\end{lemma}
\begin{proof}       
    Since $(x, \bm\lambda) \in \Gr_{\Xcal}(\cl\lnot)$ is an isolated point of $\Gr_{\Xcal}(\cl\lnot)$, we have that, $\exists \bar{\kappa} > 0$ such that, for all $\kappa \in (0, \bar{\kappa}]$, $\Bcal_{\kappa}((x, \bm\lambda)) \cap \Gr_{\Xcal}(\cl\lnot) = (x, \bm\lambda)$.

    Thus, for all $\kappa \in (0, \bar{\kappa}]$ and all $\bar{\bm\lambda}$ such that $\|\bar{\bm\lambda} - \bm\lambda\|_{\infty} \le \kappa$, $\bar{\bm\lambda} \ne \bm\lambda$, we have that $(x, \bar{\bm\lambda}) \notin \Gr_{\Xcal}(\cl\lnot)$. It follows that $\bar{\bm\lambda} \notin \cl(\lnot x)$ as well. Thus, $\bm\lambda$ is an isolated point of $\cl(\lnot x)$, thus concluding the proof.
\end{proof}

Next, we characterize isolated points of $\cl(\lnot x)$.

\begin{lemma}\label{lemma:isolated-points-cli}
    Let $x \in \Xcal$ and let $\bm\lambda \in \cl(\lnot x)$ be an isolated point. Then $\bm\lambda \in \lnot x$.
\end{lemma}
\begin{proof}
    This follows trivially from the definition of isolated points. Indeed, by definition, $\cl(\lnot x)$ is defined as the union of $\lnot x$ together with all the limit points of $\lnot x$. Since $\bm\lambda$ is an isolated point, it is not a limit point. Thus, $\bm\lambda \in \lnot x$.
\end{proof}

\subsubsection{Continuous Correspondences Imply \Cref{ass:conv}}\label{subsubsec:cont-thm1}

Given the preliminary results discussed so far, we now dive into proving the second part of \Cref{ass:conv}, \ie the fact that, for all $\epsilon > 0$ there exists $\rho > 0$ such that $D(\bm\mu, \bm\omega, \neg \mathcal{B}_{\rho}(x)) - D(\bm\mu, \bm\omega, \neg x) \le \epsilon$ holds uniformly across $\Theta$, $\Delta_K$ and $\mathcal{X}$.

\paragraph{Proof Sketch} The main idea behind the proof is showing that, for all $\eta > 0$, there exists $\bar{\rho}_{\eta} > 0$ such that, for all $x \in \Xcal$, $\bm\lambda \in \cl(\lnot x)$, $\rho\in(0,\bar{\rho}_{\eta}]$, there exists $\tilde{\bm\lambda} \in \lnot \Bcal_{\rho}(x)$ such that $\|\tilde{\bm\lambda} - \bm\lambda\|_{\infty} \le \eta$ (\Cref{lemma:close-point}). Indeed, whenever this condition holds, \Cref{ass:conv} holds as well (\Cref{lemma:suff-cond-ass-conv}). It is important to note that, for all $\eta > 0$, $\bar{\rho}_\eta$ needs to uniformly exists for all $x\in \Xcal$ and $\bm\lambda \in \cl (\lnot x)$. That being said, the proof of the existence of $\bar{\rho}_{\eta}$ is constructive, and requires the following intermediate steps (\Cref{lemma:g-eta-cont}, \Cref{lemma:r-eta-cont} and \Cref{lemma:kappa-eta-pos}). Specifically, we will show that it is sufficient to set $\bar{\rho}_{\eta}$ as follows:
\begin{align*}
    \bar{\rho}_{\eta} = \frac{1}{2} \min_{x \in \mathcal{X}} \min_{\bm\lambda \in \cl(\neg x)} \max_{\tilde{\bm\lambda} \in \mathcal{M}: \|\tilde{\bm\lambda} - \bm\lambda \|_{\infty} \le \eta} \min_{\tilde{x} \in \Xcal^{\star}(\tilde{\bm\lambda})} \| x - \tilde{x} \|_{\infty}.
\end{align*}

In this sense, \Cref{lemma:g-eta-cont}, \Cref{lemma:r-eta-cont} and \Cref{lemma:kappa-eta-pos} analyze this expression with a bottom-up approach  whose main goal is showing that the value of the optimization problem is strictly greater than $0$. Once this is done,  \Cref{lemma:close-point} proves that for any $\bm\lambda \in \cl(\neg x)$, there exists $\tilde{\bm\lambda} \in \neg \mathcal{B}_{\rho}(x)$ such that $\| \tilde{\bm\lambda} - \bm\lambda \|_{\infty} \le \eta$ for all $\rho \le \bar{\rho}_{\eta}$, which, in turn, will imply \Cref{ass:conv} via \Cref{lemma:suff-cond-ass-conv}.

Now that the main steps of the proof have been highlighted, we present the following lemma that study the optimization problem:
\begin{align}\label{g eta x lambda}
     g_{\eta}(x, \bm\lambda) = \sup_{\tilde{\bm\lambda} \in \mathcal{M}: \|\tilde{\bm\lambda} - \bm\lambda \|_{\infty} \le \eta} \inf_{\tilde{x} \in \Xcal^{\star}(\tilde{\bm\lambda})} \| x - \tilde{x} \|_{\infty}.
\end{align}
Specifically, the following result shows that both the supremum and the infimum are attained, and that $g_{\eta}$ is continuous over $\mathcal{X} \times \mathcal{M}$.

\begin{lemma}\label{lemma:g-eta-cont}
    Let $\eta > 0$. Consider $g_{\eta}: \mathcal{X} \times \mathcal{M} \rightarrow \mathbb{R}$ defined as per \Cref{g eta x lambda}.
    Then, $g_{\eta}$ is continuous over $\mathcal{X} \times \mathcal{M}$. Furthermore, both the $\sup$ and the $\inf$ are attained.
\end{lemma}
\begin{proof}
    First, consider $\inf_{\tilde{x} \in \mathcal{X}^{\star}(\tilde{\bm\lambda})} \|x - \tilde{x} \|_{\infty}$. This infimum is attained for all $\bm\lambda \in \Theta$ since it is an infimum over a compact set of a continuous function. Furthermore, by \Cref{thm:berge}, the function is continuous. Therefore, the supremum is attained as well.  

    Finally, consider $g_\eta(x, \bm\lambda)$. Using the fact that the infimum is continuous and that $\{ \tilde{\bm\lambda} \in \mathcal{M}: \| \tilde{\bm\lambda} - \bm\lambda \|_{\infty} \le \eta \}$ is a continuous correspondence of $\bm\lambda$ (\Cref{lemma:i-star-cont}), we can apply \Cref{thm:berge}, and we have that $g_{\eta}$ is continuous over its domain, thus concluding the proof.
\end{proof}

We continue by optimizing $g_{\eta}(x, \bm\lambda)$ in its second argument, \ie:
\begin{align*}
    r_{\eta}(x) = \inf_{\bm\lambda \in \cl(\neg x)} g_{\eta}(x, \bm\lambda). 
\end{align*}
As above, we show that the infimum is attained, and that $r_{\eta}$ is continuous over $\mathcal{X}$. This proof requires an extended version of the Berge's maximum theorem that can handle correspondences which are only lower hemicontinuous (see \Cref{thm:feinberg}). Indeed, as we proved in \Cref{lemma:cl-cont}, we only have a lower hemicontinuity result on the correspondence $x \mapstoto \cl(\neg x)$.

\begin{lemma}\label{lemma:r-eta-cont}
    Let $\eta > 0$. Let $r_{\eta}: \Xcal \rightarrow \mathbb{R}$ be defined as follows:
    \begin{align*}
        r_{\eta}(x) = \inf_{\bm\lambda \in \cl(\lnot x)} g_{\eta}(x, \bm\lambda).
    \end{align*}
    Then, $r_{\eta}(x)$ is continuous over $\Xcal$. Furthermore, the infimum is attained.
\end{lemma}
\begin{proof}
    Let $\mathbb{X} = \Xcal$, $\mathbb{Y}=\Mcal$, $\phi(x)=\cl(\lnot x)$ for all $x\in \mathbb{X}$, and
    \[
    u(x, \bm\lambda) = \max_{\tilde{\bm\lambda} \in \Mcal: \|\tilde{\bm\lambda}-\bm\lambda\|_{\infty} \le \eta} \min_{\tilde{x} \in \mathcal{X}^{\star}(\tilde{\bm\lambda})} \| x -\tilde{x} \|.
    \]
    We want to apply \Cref{thm:feinberg}, which extends Berge's maximum theorem for lower hemicontinous correspondences, to prove the continuity of $r_\eta$. Before proceeding, we invite the reader to consult \Cref{app:corresp} for useful definitions that will be used within this proof.

    Now, we first note that, due to \Cref{lemma:cl-cont}, $\phi(x)$ is lower hemicontinuous over $\mathbb{X}$.

    Thus, to apply \Cref{thm:feinberg}, it remains to check that $u$ is $\mathbb{K}$-inf-compact and upper semicontinuous on $\Gr_{\mathbb{X}}(\phi)$. First, we note that $u$ is continuous on $\mathbb{X} \times \mathbb{Y}$ (\Cref{lemma:g-eta-cont}), and hence upper semicontinuous. Furthermore, let $K \subset \mathbb{X}$ be compact and consider $\Gr_K(\phi)$. Then, let us analyze the level sets $\mathcal{D}_u(\alpha, \Gr_K(\phi))$ for $\alpha \in \mathbb{R}$. Then, $\mathcal{D}_u(\alpha, \Gr_K(\phi)) \subseteq \mathbb{X} \times \mathbb{Y}$, which is compact, hence bounded. Thus, $\mathcal{D}_u(\alpha, \Gr_K(\phi))$ is bounded as well. Moreover, $\mathcal{D}_u(\alpha, \Gr_K(\phi))$ is closed since it is the preimage of a closed set of a continuous function. We have thus proved that $u$ is $\mathbb{K}$-inf-compact and upper semicontinuous on $\Gr_{\mathbb{X}}(\phi)$. From \Cref{thm:feinberg}, $r_\eta$ is continuous over $\mathbb{X}$

    Finally, the supremum can be replaced by the maximum since the objective function is continuous and the optimization domain is compact. 
\end{proof}

Finally, we optimize $r_{\eta}(x)$ over the possible answers $x \in \Xcal$. The resulting value, \ie $\kappa_{\eta}$, will be than used to define $\bar{\rho}_\eta$. Specifically, in \Cref{lemma:close-point}, we will set $\kappa_{\eta} = \bar{\rho}_\eta / 2$. That being said, in the following lemma we show that the infimum of $r_{\eta}(x)$ is attained at some point $x \in \mathcal{X}$. Furthermore, it also shows that $\kappa_{\eta} > 0$. 

\begin{lemma}\label{lemma:kappa-eta-pos}
    Let $\eta > 0$, and let $\kappa_\eta = \inf_{x \in \Xcal} r_{\eta}(x)$. The infimum is attained, \ie $\kappa_\eta = \min_{x \in \Xcal} r_{\eta}(x)$. Furthermore, $\kappa_{\eta} > 0$.
\end{lemma}
\begin{proof}
    First, we note that the infimum is attained. Indeed, from \Cref{lemma:r-eta-cont}, $r_{\eta}$ is continuous. The optimization domain is compact, and, therefore, the infimum is attained. 

    Secondly, we want to prove that $\kappa_{\eta} > 0$. Given the definition of $\kappa_{\eta}$, and since the infimum is attained, this is equivalent to proving that, for all $x \in \Xcal$, $r_{\eta}(x) > 0$. From \Cref{lemma:r-eta-cont}, we know that the min over the $\cl(\lnot x)$ is attained, therefore, we want to prove that for all $x \in \Xcal, \bm\lambda \in \cl(\lnot x)$, $g_{\eta}(x, \bm\lambda) > 0$.

    In other words, consider $\Gr_{\Xcal}(\cl\lnot)$ and let $(x, \bm\lambda) \in \Gr_{\Xcal}(\cl\lnot)$. We need to prove that $g_{\eta}(x,\bm\lambda) > 0$.
    We proceed by cases.
    \begin{itemize}
    \item If $(x, \bm\lambda)$ is an isolated point in $\Gr_{\Xcal}(\cl\lnot)$, then $\bm\lambda$ is an isolated point in $\cl(\lnot x)$ (by virtue of \Cref{lemma:isolated-points}), and, hence, $\bm\lambda \in \lnot x$ (as guaranteed by \Cref{lemma:isolated-points-cli}). Thus, we have that:
    \begin{align*}
        g_{\eta}(x,\bm\lambda) \ge \min_{\tilde{x} \in \Xcal^{\star}(\bm\lambda)} \| x-\tilde{x} \|_{\infty} > 0. 
    \end{align*}
     
    \item If, on the other hand, $(x, \bm\lambda)$ is a limit point, there are two sub-cases. 
    Either $\bm\lambda$ is an isolated point of $\cl(\lnot x)$ or it is a limit point of $\cl(\lnot x)$.
    If $\bm\lambda$ is an isolated point of $\cl(\lnot x)$ we have that $\bm\lambda \in \neg x$ (thanks to \Cref{lemma:isolated-points-cli}), and we can proceed as for isolated points of $\Gr_{\Xcal}(\cl\lnot)$. %
    \item If $\bm\lambda$ is a limit point of both the graph and of $\cl(\lnot x)$, there are two further cases. Either $\bm\lambda \in \lnot x$, and we can proceed as for isolated points of $\Gr_{\Xcal}(\cl\lnot)$, or $\bm\lambda \in \cl(\lnot x)$ and $\bm\lambda \notin \lnot x$. In this last case, however, since $\bm\lambda$ is a limit point of $\cl(\lnot x)$, there exists a sequence $\{ \bm\lambda_n \}_{n \in\mathbb{N}}$ such that (i) $\bm\lambda_n \ne \bm\lambda,~\forall n \in \mathbb{N}$, (ii) $\bm\lambda_n \in \lnot x,~\forall n \in \mathbb{N}$, and (iii) $\{\bm\lambda_n \} \to \bm\lambda$. Therefore, from (ii)+(iii), we have that for $\eta > 0$, there exists $n_\eta$ such that, for all $n \ge n_{\eta}$, $\|\bm\lambda_n - \bm\lambda\|_{\infty} \le \eta$ and $\bm\lambda_n \in \neg x$. Consider, \eg $\bm\lambda_{n_{\eta}}$; then, we have that: 
    \begin{align*}
        g_{\eta}(x, \bm\lambda) = \sup_{\tilde{\bm\lambda} \in \Mcal: \| \tilde{\bm\lambda} - \bm\lambda \|_{\infty} \le \eta} \inf_{\tilde{x} \in \Xcal^{\star}(\tilde{\bm\lambda})} \|x-\tilde{x} \|_{\infty} \ge \inf_{\tilde{x} \in \Xcal^{\star}(\bm\lambda_{n_\eta})} \|x-\tilde{x} \|_{\infty} > 0,
    \end{align*}
    thus leading to the desired result.
    \end{itemize}
    Therefore, all $(x, \bm\lambda) \in \Gr_{\Xcal}(\cl\lnot)$, we have that $g_{\eta}(x, \bm\lambda) > 0$ and therefore $\kappa_{\eta} > 0$.
\end{proof}

Now, we show that for all $\rho \in (0, \kappa_{\eta}]$, $\forall \bar x \in \Xcal, \bm\lambda \in \cl(\lnot \bar x)$ there exists $\tilde{\bm\lambda} \in \lnot \Bcal_{\rho}(\bar x)$ such that $\|\tilde{\bm{\lambda}} - \bm\lambda\|_{\infty} \le \eta$. As we shall see, combining this result with \Cref{lemma:suff-cond-ass-conv}, leads to the desired result.

\begin{lemma}\label{lemma:close-point}
     For all $\eta > 0$ sufficiently small, there exists $\bar{\rho}_\eta$ such that, for all $\rho \in (0, \bar{\rho}_{\eta}]$, it holds that:
    \begin{align*}
        \forall \bar x \in \Xcal, \bm\lambda \in \cl(\lnot \bar x),~\exists \tilde{\bm\lambda} \in \lnot \Bcal_{\rho}(\bar x): \|\tilde{\bm{\lambda}} - \bm\lambda\|_{\infty} \le \eta. 
    \end{align*}
\end{lemma}
\begin{proof}
    Consider any $x \in \Xcal$ and $\bm\lambda \in \cl(\lnot x)$. Let $\bar{\bm\lambda}_{x,\bm\lambda}$ be defined as:
    \begin{align*}
    \bar{\bm\lambda}_{x,\bm\lambda} \in \argmax_{\tilde{\bm\lambda} \in \Mcal: \|\tilde{\bm\lambda} - \bm\lambda\|_{\infty} \le \eta} \min_{\tilde{x} \in \Xcal^{\star}(\tilde{\bm\lambda})} \| x - \tilde{x} \|.
    \end{align*}
    From \Cref{lemma:g-eta-cont}, $\bar{\bm\lambda}_{x,\bm\lambda}$ is well-defined. Furthermore, by definition $\| \bar{\bm{\lambda}}_{x, \bm\lambda} - \bm\lambda \|_{\infty} \le \eta$.

    Take $\bar{\rho}_\eta = \kappa_\eta / 2$, and let $\rho \le \bar{\rho}_\eta$. In the following, we prove that $\bar{\bm\lambda}_{x, \bm\lambda} \in \neg \mathcal{B}_{\rho}(x)$.
     
    \begin{align*}
        \lnot \Bcal_{\rho}(x) & \supseteq \lnot \Bcal_{\kappa_{\eta} / 2}(x) \\
        & = \{ \bm\theta \in \Mcal: \forall \tilde x \in \Bcal_{\kappa_\eta/2}(x), \tilde{x} \notin \Xcal^{\star}(\bm\theta) \} \\
        & \supseteq \{ \bar{\bm\lambda}_{x, \bm\lambda} \}
    \end{align*}
    where the last step can be proved by contradiction. Suppose it is false, \ie there exists $\bar{x} \in \mathcal{B}_{\kappa_{\eta}/2}(x)$ such that $\bar{x} \in \Xcal^{\star}(\bar{\bm\lambda}_{x, \bm\lambda})$. Then, it holds that:
    \begin{align*}
        \|x-\bar{x} \|_{\infty} \le \frac{\kappa_{\eta}}{2} \le  \frac{1}{2} \min_{\tilde{x} \in \Xcal^{\star}(\bar{\bm\lambda}_{x, \bm\lambda})} \|x - \tilde{x}  \|_{\infty} \le \frac{1}{2} \|x-\bar{x} \|_{\infty} < \|x - \bar{x} \|_{\infty}, 
    \end{align*}
    where in the last step we used $\kappa_{\eta} > 0$. This leads to a contradiction and concludes the proof.
\end{proof}

Now, we show that, when \Cref{lemma:close-point} holds, then \Cref{ass:conv} holds as well. 

\begin{lemma}[Sufficient condition for Assumption~\ref{ass:conv}]\label{lemma:suff-cond-ass-conv}
    If, for all $\tilde{\epsilon} > 0$, there exists ${\rho} > 0$, such that:
    \begin{align*}
        \forall \bar x \in \Xcal, \bm\lambda \in \cl(\lnot \bar x),~\exists \tilde{\bm\lambda} \in \lnot \Bcal_{\rho}(\bar x): \|\tilde{\bm{\lambda}} - \bm\lambda\|_{\infty} \le \tilde{\epsilon}. 
    \end{align*}
    then Assumption~\ref{ass:conv} holds.
\end{lemma}
\begin{proof}
    Let $(\bm\mu, \bm\omega) \in \Theta^K \times \Delta_K$.
    Let $\bm\lambda^{*} \in \argmin_{\bm\lambda \in \cl(\neg x)} \sum_{k \in [K]} \omega_k d(\mu_k, \lambda_k)$ and denote by $\bm\theta^{\bm\lambda^*} \in \neg \mathcal{B}_{\rho}(x)$ be such that $\| \bm\lambda^* - \bm\theta^{\bm\lambda^*} \|_\infty \le \tilde{\epsilon}$.\footnote{Note that $\bm\lambda^*$ is well defined. Indeed, $\cl(\neg x)$ is compact, since $\mathcal{M}$ is compact. Furthermore, the function that is being optimized is continuous in $\bm\lambda$, and thus the $\min$ is attained.} Let $\tilde{\bm{\lambda}}^*$ and $\bm{\tilde{\theta}}^{\bm\lambda^*}$ be the natural parameters of the distributions related to means $\bm{\lambda}^*$ and $\bm{\theta}^{\bm\lambda^*}$, respectively. 
    Then, it holds that:
    \begin{align*}
        D(\bm{\mu}, \bm{\omega}, &\neg \Bcal_{\rho}(x)) - D(\bm{\mu}, \bm{\omega}, \neg x) 
        = \inf_{\bm{\lambda} \in \neg \Bcal_{\rho}(x)} \sum_{k \in [K]} \omega_k d(\mu_k, \lambda_k) - \inf_{\bm{\lambda} \in \neg x} \sum_{k \in [K]} \omega_k d(\mu_k, \lambda_k) \\
        & \le \inf_{\bm{\lambda} \in \neg \Bcal_{\rho}(x)} \sum_{k \in [K]} \omega_k d(\mu_k, \lambda_k) -  \inf_{\bm{\lambda} \in \cl(\neg x)} \sum_{k \in [K]} \omega_k d(\mu_k, \lambda_k)  \tag{Since $\neg x \subseteq \cl(\neg x)$}\\
        & \le  \sum_{k \in [K]} {\omega}_k \left( d(\mu_k, {\theta}_k^{\bm\lambda^*}) - d(\mu_k, {\lambda}^{*}_k) \right) \tag{Since $\bm\theta^{\bm\lambda^*}\in \lnot\Bcal_{\rho}(x)$ and by definition of $\bm\lambda^*$}\\ 
        & \le \max_{k \in [K]} \Big| d(\mu_k, \theta_k^{\bm\lambda^*}) - d(\mu_k, \lambda_k^{*}) \Big| \\ 
        & =  \max_{k \in [K]} \Big| d(\lambda_k^{*}, \theta_k^{\bm\lambda^*}) + (\tilde{\theta}_k^{\bm\lambda^*} - \tilde{\lambda}^{*}_k) ( \mu_k - \lambda_k^*) \Big| \tag{\Cref{lemma:kl-diff}}\\ 
        & \le \max_{k \in [K]} \Big| C_2 (\lambda_k^{*} - \theta_k^{\bm\lambda^*})^2 + M (\tilde{\theta}_k^{\bm\lambda^*} - \tilde{\lambda}^{*}_k) \Big| \tag{\Cref{coroll:kl-lip} and $\Theta$ bounded} \\
        & \le \max_{k \in [K]} \left( C_2 M \Big|\lambda_k^{*} - \theta_k^{\bm\lambda^*} \Big| + C_1 M \Big| \lambda_k^* - \theta_k^{\bm\lambda^*} \Big| \right)  \tag{\Cref{coroll:kl-lip} and $\Theta$ bounded} \\
        & = (C_1 + C_2) M ||\bm\lambda^* - \bm\theta^{\bm\lambda^*}||_{\infty}.
    \end{align*}
    Here, since $\Theta$ is bounded and contained in an open interval, we have used $(\mu - \mu') \le M$ for any $\mu, \mu' \in \Theta$, where $M \coloneqq \max_{\mu \in \Theta} \mu - \min_{\mu \in \Theta} \mu$. Now, taking $\tilde{\epsilon} < \frac{\epsilon}{M(C_1 + C_2)}$, concludes the proof. 
\end{proof}

\asscontinuity*
\begin{proof}
    Combine \Cref{lemma:suff-cond-ass-conv} together with \Cref{lemma:close-point}.
\end{proof}

\section{Proof of the Lower Bound}\label{app:lb}

In this section, we first provide a sketch of the proof that outlines all the underlying ideas and the differences with respect to previous works, and then we present the formal arguments. 

\subsection{Proof Sketch of \Cref{theo:lb}}

The general idea behind the proof is inspired by the lower bound for multiple answers problems presented in \cite{degenne2019pure}. Specifically, in \cite{degenne2019pure}, the authors start by noticing that for any $T \in \mathbb{N}$, using Markov's inequality, one has that:
\begin{align*}
    \E_{\bm\mu}[\tau_\delta] = T ( 1 - \mathbb{P}_{\bm\mu}(\tau_\delta \le T) \ge T\left( 1-  \left( \delta + \sum_{x \in \Xcal^{\star}(\bm\mu)} \mathbb{P}_{\bm\mu}\left( \{ \tau_\delta \le T \} \textup{ and } \{ \hat{x}_{\tau_\delta} = x \} \right) \right)\right).
\end{align*}
Then, the proof follows by upper bounding $\mathbb{P}_{\bm\mu}\left( \{ \tau_\delta \le T \} \textup{ and } \{ \hat{x}_{\tau_\delta} = x \} \right)$ for each $x \in \Xcal^{\star}(\bm\mu)$ using change-of-measure arguments. As one can see, however, such an argument can actually be applied only when $\Xcal^{\star}(\bm\mu)$ is finite and, thus, complications arise in our infinite answer setting. 

To solve this issue and prove \Cref{theo:lb}, we combine three distinct elements, that are:
\begin{itemize}
    \item[(i)] An exact covering of the set $\Xcal^{\star}(\bm\mu)$ using balls of radius $\rho$ 
    \item[(ii)] A change-of measure arguments that are directly related to this cover and to our new extended notion of alternative models over sets
    \item[(iii)] A limit argument for $\rho \to 0$  
\end{itemize}
As we now discuss, these three ingredients allows us to \quotes{reduce} the infinite answer problem to a finite answer one.

First, we note that, since $\Xcal^{\star}(\bm\mu)$ is compact, it admits a finite cover $\{\widetilde{\Xcal}_i\}_{i=1}^{n_\rho}$ of $n_{\rho} \in \mathbb{N}$ elements using sets $\widetilde{\Xcal}_i$ which are inscribed in balls of radius $\rho$ (\Cref{lemma:covering}). Now, the idea is to fix a cover and try to follow the arguments of \cite{degenne2019pure}. Specifically, we are going  apply change of measure arguments by directly exploiting the sets $\{ \widetilde{\Xcal}_i \}_i$. Thus, for any $T \in \mathbb{N}$, using Markov's inequality, one has that: 
\begin{align*}
    \E_{\bm\mu}[\tau_\delta] \ge T\left(1 - \mathbb{P}_{\bm\mu}(\tau_\delta \le T)\right) \ge T\left(1 - \left(\delta + \sum_{i=1}^{n_{\rho}} \mathbb{P}_{\bm\mu}\left( \{ \tau_\delta \le T \} \text{ and } \{ \hat{x}_{\tau_\delta} \in \widetilde{\Xcal}_i \} \right) \right) \right),
\end{align*}
where in the second step we have used the $\delta$-correctness of the algorithm on regions that are complementary to $\Xcal^{\star}(\bm\mu)$. 

Let $\mathcal{E}_i = \{ \{ \tau_\delta \le T \} \text{ and } \{ \hat{x}_{\tau_\delta} \in \widetilde{\Xcal}_i \} \}$.
From here, the idea is to upper bound $\mathbb{P}_{\bm\mu}\left( \mathcal{E}_i \right)$. Using change-of-measure arguments (\Cref{lemma:change-of-measure}), we can relate this probability to the one of the same event but under models in $\lnot \widetilde{\Xcal}_i$, \ie models for which \emph{all} answers within $\widetilde{\Xcal}_i$ are not correct. In this sense, the definition of alternative sets over sets of answers plays a crucial role, as it allows to obtain that, for all $\beta > 0$, some problem dependent constant $\alpha$, and some $\bm\lambda \in \lnot \widetilde{\Xcal}_i$:
\begin{align*}
\mathbb{P}_{\bm\mu}\left( \mathcal{E}_i \right) & \le \exp\left( T D(\bm\mu, \lnot \widetilde{\Xcal}_i) + \beta \right) \mathbb{P}_{\bm\lambda}\left( \mathcal{E}_i \right) + \exp\left( \frac{-\beta^2}{2T\alpha} \right) \\ 
& \le \exp\left( T D(\bm\mu, \lnot \widetilde{\Xcal}_i) + \beta \right) \delta  + \exp\left( \frac{-\beta^2}{2T\alpha} \right),
\end{align*}
where in the second step, we explicitly use $\bm\lambda \in \lnot \widetilde{\Xcal}_i$. For that step, indeed, it is required that every answer within $\widetilde{\Xcal}_i$ is not a correct one for $\bm\lambda$. From here, one can use standard arguments from \citet{degenne2019pure}, and obtain (\Cref{lemma:helper-lemma-lb}) the following asymptotic result:
\begin{align}\label{eq:approx-lb}
    \liminf_{\delta \to 0} \frac{\E_{\bm\mu}[\tau_\delta]}{\log(1/\delta)} \ge \frac{1}{\max_{i \in [n_\rho]} D(\bm\mu, \lnot \widetilde{\Xcal}_i)},
\end{align}
The proof of \Cref{theo:lb} then follows by analyzing \Cref{eq:approx-lb} as $\rho \to 0$.

\subsection{Proof of \Cref{theo:lb}}
We start the proof by introducing some preliminary results.

\begin{lemma}[Minimax Results]\label{lemma:max-min-to-min-max}
Let $\bm\mu\in\Theta^K$ and $\Lambda\subset \Mcal$ then
\[
D(\bm\mu,\Lambda)=\inf_{\Prob}\max_{k\in[K]}\mathbb{E}_{\bm\lambda\sim \Prob}[d(\mu_k,\lambda_k)],
\]
where the infimum ranges over probability distributions on $\Lambda$ supported on (at most) $K$ points.
\end{lemma}
\begin{proof}
    This result is a direct consequence of Lemma 2 in \cite{degenne2019non}. In \cite{degenne2019non} the authors state the result for a set $D(\bm\mu, \neg x)$, but it actually holds for any set $\Lambda$.
\end{proof}

We say that a distribution $q\in\Delta_K$, supported on $\bm\lambda^1,\ldots,\bm\lambda^K$, is optimal for $D(\bm\mu,\Lambda)$ for a given $\Lambda\subseteq\Mcal$, if it attains the infimum of \Cref{lemma:max-min-to-min-max}, \emph{i.e.}, if
\[
D(\bm\mu,\Lambda)=\max_{k\in [K]}\sum_{j\in[K]}q_jd(\mu_k,\lambda^j_k).
\]
In particular, we are interested in the case in which $\Lambda=\lnot \widetilde{\Xcal}$, for some $\widetilde{\Xcal}$ such that $\lnot \widetilde{\Xcal} \ne \emptyset$. In the following, we will assume, as in \cite{degenne2019pure}, that the infimum in $D(\bm{\mu}, \neg \widetilde{\Xcal})$ is attained in $\neg \widetilde{\Xcal}$. If not, one can apply the following arguments to a sequence of $\epsilon$-optimal distributions and let $\epsilon \rightarrow 0$.

Given this consideration, we recall a relevant change of measure arguments that have been previously used in BAI, \ie Lemma 19 of \cite{degenne2019pure}. 
Before doing that, we introduce some necessary notation. For any model $\bm\mu\in\Mcal$ and any $k$, we denote with $\tilde \mu_k$ the natural parameter of a distribution of the exponential family with mean $\mu_k$. Further details on canonical exponential families are deferred to  \Cref{sec:exponential_fam}.

\begin{lemma}[Change of Measure]\label{lemma:change-of-measure}
    Fix $\bm{\mu} \in \mathcal{M}$ and let $\widetilde{\Xcal} \subseteq \Xcal$ be a subset of answers such that $\neg \widetilde{\Xcal} \ne \emptyset$. Let $\bm{\lambda}^{1}, \dots, \bm{\lambda}^{K}$ and $\bm q \in \Delta_{K}$ be an optimal distribution for $D(\bm{\mu}, \neg \widetilde{\Xcal})$. Let $\alpha_{k} \coloneqq \tilde{\mu}_k - \sum_{j \in [K]} q_{j} \tilde{\lambda}^{j}_{k}$ and $\bar \alpha=\max_{k\in[K]}\alpha_k$. Fix a sample size $t$ and any event $\mathcal{E} \in \mathcal{F}_t$. Then, for any $\beta > 0$ it holds that:
    \begin{align*}
        \max_{k \in [K]} \Prob_{\bm{\lambda}^{k}}(\mathcal{E}) \ge \exp\left\{ {-tD(\bm{\mu}, \neg \widetilde{\Xcal})} - \beta \right\} \left( \Prob_{\bm{\mu}}(\mathcal{E}) - \exp\left\{ \frac{-\beta^2}{2t \bar\alpha^2 } \right\} \right).
    \end{align*}
\end{lemma}
\begin{proof}
    The proof is as in \citet[Lemma~19]{degenne2019pure}.
\end{proof}

\Cref{lemma:change-of-measure} can be interpreted as follow. Whenever $t \ll D(\mu, \neg \widetilde{\Xcal})^{-1}$, then if $\mathcal{E}$ is likely under $\bm\mu$ than it must also be likely under at least one $\bm\lambda^k$.

At this point, we derive the following intermediate results. The proof scheme is inspired by Theorem 1 in \cite{degenne2019pure}. The key difference is that now we are applying those arguments to infinite answer identification problems, where each $\Xcal^\star(\bm{\mu})$ is an arbitrary compact set.

\begin{lemma}[Intermediate Result]\label{lemma:helper-lemma-lb}
    For every $\bm{\mu} \in \mathcal{M}$ and for any $\rho > 0$ sufficiently small, there exists a finite set of answers $\{ x_j \}_{j=1}^{n_\rho}$ and $x_j\in\mathcal{X}^\star(\bm\mu)$ such that, if we define $\widetilde{\Xcal}_j=\Bcal_{\rho}(x_j) \cap \Xcal^\star(\bm{\mu})$, then $\Xcal^\star(\bm{\mu}) = \bigcup_{j=1}^{n_{\rho}} \widetilde{\mathcal{X}}_j$ . Moreover, there exists $\widetilde{\Xcal}_j$ in the cover such that $D(\bm{\mu}, \neg \widetilde{\Xcal}_j) > 0$. Furthermore, for any $\delta$-correct algorithm it holds that:
    \begin{align*}
        \liminf_{\delta \rightarrow 0}  \frac{\E_{\bm{\mu}}[\tau_\delta]}{\log(1 / \delta)} \ge \min_{j \in [n_{\rho}]:D(\bm{\mu}, \neg \widetilde{\Xcal}_{j})>0} D(\bm{\mu}, \neg \widetilde{\Xcal}_{j})^{-1}.
    \end{align*}
\end{lemma}
\begin{proof}
    Fix $T > 0$ (to be defined later), due to the Markov's inequality, we have that:
    \begin{align}\label{eq:helper-lemma-eq1}
        \E_{\bm{\mu}}[\tau_\delta] \ge T (1 - \Prob_{\bm{\mu}}(\tau_\delta \le T)).
    \end{align}
    
    Before analyzing $\Prob_{\bm{\mu}}(\tau_\delta \le T)$, we recall that $\Xcal^\star(\bm{\mu})$ is compact. Therefore, by \Cref{lemma:covering}, we know that, for any $\rho > 0$, there exists a finite collection of compact subsets $\{\widetilde{\Xcal}_j\}_{j\in [N_\rho]}$ such that $\Xcal^\star(\bm{\mu}) = \bigcup_{j=1}^{N_\rho} \widetilde{\Xcal}_j = \bigcup_{j=1}^{N_\rho} \left(\Xcal^\star(\bm{\mu}) \cap \Bcal_{\rho}(x_j) \right)$ for some $x_j \in \Xcal^\star(\bm{\mu})$. Furthermore, from \Cref{ass:conv}, we also know that there exists $\bar{\rho}$ sufficiently small such that, for all $\rho \le \bar{\rho}$, $\neg \mathcal{B}_{\rho}(x) \ne \emptyset$ for all $x \in \Xcal$. In the following, we thus consider any $\rho$ that satisfies this property.

    Now, from \Cref{ass:identify} we know that there exist $\bar{x}$ such that $\bm\mu \notin \cl(\neg \bar{x})$. We add the set $\Xcal^\star(\bm\mu)\cap \Bcal_\rho(\bar x)$ to the aforementioned cover, thus obtaining a  cover of size $n_\rho=N_\rho+1$, such that each element of the cover is of the kind $\Xcal^\star(\bm\mu)\cap\Bcal_\rho(x_j)$ for some $x_j \in \Xcal^\star(\bm\mu)$ and $\lnot\Bcal_\rho(x_j)\neq\emptyset$ for all $j\in[n_\rho]$. Furthermore, by applying \Cref{lemma:iff-positive-div}, we know there exists $x_j$ (\ie $\bar{x}$) for some $j\in[n_\rho]$ such that $D(\bm{\mu}, \neg \widetilde{\Xcal}_j) > 0$.\footnote{Indeed, since $\bm\mu \notin \cl(\neg \bar{x})$, it also holds that $\bm\mu \notin \widetilde{\Xcal}_j$.} As we shall see in a few steps, this property will be used to invert $\max_{j \in [n_\rho]} D(\bm\mu, \neg \widetilde{\Xcal}_j)$.

    Now, we analyze $\Prob_{\bm{\mu}}(\tau_\delta \le T)$. Let us introduce the event $\mathcal{E} = \{ \hat{x}_{\tau_\delta} \in \Xcal^\star(\bm{\mu}) \}$ and the events $\mathcal{E}_j=\{\tau_\delta \le T \}\cap\{ \hat{x}_{\tau_\delta} \in \widetilde{\Xcal}_j\}$ for all $j\in[n_\rho]$.\footnote{We note that $\mathcal{E}_j$ is measurable since $\hat{x}_{\tau_\delta}$ is measurable with respect to $\Fcal_{\tau_\delta}$ and $\widetilde{\Xcal}_j$ is a Borel set.} Consider the following:
    \begin{align}
        \Prob_{\bm{\mu}}(\tau_\delta \le T) & = \Prob_{\bm{\mu}}(\tau_\delta \le T | \mathcal{E}) \Prob_{\bm{\mu}}(\mathcal{E}) + \Prob_{\bm{\mu}}(\tau_\delta \le T | \mathcal{E}^c) \Prob_{\bm{\mu}}(\mathcal{E}^c) \notag\\ 
        & \le \Prob_{\bm{\mu}}( \{\tau_\delta \le T\} \cap \mathcal{E}) + \delta\notag \\ 
        & \le \sum_{j=1}^{n_\rho} \Prob_{\bm{\mu}}(\mathcal{E}_j) + \delta\label{eq:helper-lemma-eq2},
    \end{align}
    where in the first step, we used the law of total probability; in the second one, we used the fact that the algorithm is $\delta$-correct; and in the third one we applied a union bound. 

    At this point, consider all $j \in [n_{\rho}]$, and analyze $\Prob_{\bm{\mu}}(\mathcal{E}_j)$. Let $\bm{\lambda}^1, \dots, \bm{\lambda}^K$ and $\bm{q}$ be an optimal distribution according to $D(\bm{\mu}, \neg \widetilde{\Xcal}_j)$. Then, from Lemma~\ref{lemma:change-of-measure}, we have, for all $\beta > 0$:
    \begin{align*}
        \Prob_{\bm{\mu}}(\mathcal{E}_j) & \le \exp\left( {TD(\bm{\mu}, \neg \widetilde{\Xcal}_j)} + \beta \right) \max_{k \in [K]} \Prob_{\bm{\lambda}^{k}} \left( \mathcal{E}_j \right) + \exp\left( \frac{-\beta^2}{2T \bar{\alpha}^2} \right) \\ 
        & \le \delta \exp\left( {TD(\bm{\mu}, \neg \widetilde{\Xcal}_j)} + \beta \right) + \exp\left( \frac{-\beta^2}{2T \bar{\alpha}^2} \right),
    \end{align*}
    where, in the second step, we have used that $\Prob_{\bm{\lambda}^{k}} \left( \mathcal{E}_j \right) \le \delta$, since for all $\bm{\lambda}^k$ we have that $\bm\lambda^k\in\lnot\widetilde{\Xcal_j}$ and thus $\Prob_{\bm{\lambda}^{k}} \left( \mathcal{E}_j \right) \le \delta$ due to the $\delta$-correctness of the algorithm. Note that, for this step, the construction of the sets $\widetilde{\Xcal}_j$ and the notion of extended alternative model plays a crucial role. Indeed, since $\widetilde{\Xcal}_j = \{\bm\lambda \in \Mcal: \forall x \in \widetilde{\Xcal}_j, x \notin \Xcal^{\star}(\bm\mu) \}$ and since algorithms are $\delta$-correct, then $\Prob_{\bm{\lambda}} \left( \mathcal{E}_j \right) \le \Prob_{\bm\lambda}(\hat{x}_{\tau_\delta} \in \widetilde{\Xcal}_j) \le \delta$ for all $\bm\lambda \in \neg \widetilde{\Xcal}_j$.

    We now follow the argument of \cite{degenne2019pure}. For a fixed $\eta\in(0,1)$, we define\footnote{Note that the minimum here is finite since for at least one element of the cover has positive value.} 
    \[
    T=(1-\eta)\log(1/\delta)\min_{k\in[n_\rho]: D(\bm\mu,\lnot\widetilde\Xcal_j)>0}D(\bm\mu,\lnot\widetilde\Xcal_j)^{-1},
    \]
    and take $\beta=\frac{\eta}{2 \sqrt{1-\eta}} \sqrt{TD(\bm{\mu}, \neg \widetilde{\Xcal}_j) \log(1 / \delta)}$.
    Then
    \[
    \Prob_{\bm{\mu}}(\mathcal{E}_j) \le \delta^{\frac{\eta}{2}} + \delta^{\frac{\eta^2 D(\bm{\mu}, \neg \widetilde{\Xcal})}{8(1-\eta)  \bar{\alpha}}},
    \]
    which goes to zero for all $\eta>0$ when $\delta\to0$.

    If we plug this into \Cref{eq:helper-lemma-eq1} and using \Cref{eq:helper-lemma-eq2} we obtain that
    \begin{align*}
    \frac{\mathbb{E}_{\bm\mu}[\tau_\delta]}{\log(1/\delta)}&\ge (1-\eta)\min_{k\in[n_\rho]: D(\bm\mu,\lnot\widetilde\Xcal_j)>0}D(\bm\mu,\lnot\widetilde\Xcal_j)^{-1}\left(1-\delta-\sum_{j=1}^{n_\rho} \Prob_{\bm{\mu}}(\mathcal{E}_j)\right)\\
    &\ge (1-\eta)\min_{k\in[n_\rho]: D(\bm\mu,\lnot\widetilde\Xcal_j)>0}D(\bm\mu,\lnot\widetilde\Xcal_j)^{-1}\left(1-\delta-n_\rho\left(\delta^{\frac{\eta}{2}} + \delta^{\frac{\eta^2 D(\bm{\mu}, \neg \widetilde{\Xcal})}{8(1-\eta)  \bar{\alpha}}}\right)\right)
    \end{align*}
    and thus, by taking the limit $\delta\to0$, we obtain:
    \[
    \lim_{\delta\to0}\frac{\mathbb{E}_{\bm\mu}[\tau_\delta]}{\log(1/\delta)}\ge (1-\eta)\min_{k\in[n_\rho]: D(\bm\mu,\lnot\widetilde\Xcal_j)>0}D(\bm\mu,\lnot\widetilde\Xcal_j)^{-1},
    \]
    which holds for all $\eta>0$ and thus proves the result.
\end{proof}

We are now ready to prove Theorem~\ref{theo:lb}.

\lowerbound*
\begin{proof}[Proof of Theorem~\ref{theo:lb}]
    Let $\rho > 0$ be sufficiently small. Then, by \Cref{lemma:helper-lemma-lb}, we know that there exists a finite set of points $\{x_j\}_{j\in[n_\rho]}\subset\Xcal^\star(\bm\mu)$ such that there exists an exact cover of $\Xcal^\star(\bm\mu)$ of the form $\widetilde{\Xcal}_j=\Xcal^\star(\bm\mu)\cap \Bcal_\rho(x_j)$ and such that for at least one $j$ we have $D(\bm\mu,\lnot\widetilde{\Xcal}_j)>0$ and
    \begin{align*}
        \liminf_{\delta \rightarrow 0}  \frac{\E_{\bm{\mu}}[\tau_\delta]}{\log(1 / \delta)} \ge \min_{j \in [n_{\rho}]} D(\bm{\mu}, \neg \widetilde{\mathcal{X}}_{j})^{-1}.
    \end{align*}
    First of all, we notice that $T^*(\bm{\mu}) \ge \min_{j \in [n_{\rho}]} D(\bm{\mu}, \neg \widetilde{\Xcal}_{j})^{-1}$ always holds. Indeed, consider $x \in \argmax_{\tilde x \in \Xcal^\star(\bm{\mu})} D(\bm{\mu}, \neg \tilde x)$ and consider any $j$ such that $x \in \widetilde{\Xcal}_j$. Then, we have that:
    \begin{align*}
        D(\bm{\mu}) & = \sup_{\bm{\omega} \in \Delta_K} \inf_{\bm{\lambda} \in \neg x} \sum_{k \in [K]} \omega_k d(\mu_k, \lambda_k) \\
        & \le \sup_{\bm{\omega} \in \Delta_K} \inf_{\bm{\lambda} \in \neg \widetilde{\Xcal}_j} \sum_{k \in [K]} \omega_k d(\mu_k, \lambda_k) \\ 
        & =
        D(\bm{\mu}, \neg \widetilde{\Xcal}_j), 
    \end{align*}
    where in the inequality step we have used $\neg \widetilde{\Xcal}_j \subseteq \neg x$ since $x \in \widetilde{\Xcal}_j$.
    
    Then, let $\bar{j} \in \argmax_{j \in [n_\rho]} D(\bm{\mu}, \neg \widetilde{\Xcal}_j)$, and let $\bar{x}$ be such that $\widetilde{\Xcal}_{\bar{j}} = \Xcal^\star(\bm{\mu}) \cap \Bcal_{\rho}(\bar{x})$ for some $\bar{x} \in \Xcal^\star(\bm{\mu})$. Then, we have that:
    \begin{align*}
         \frac{1}{D(\bm{\mu})} - \min_{j \in [n_{\rho}]} \frac{1}{D(\bm{\mu}, \neg \widetilde{\Xcal}_{j})}  %
         & = \frac{D(\bm{\mu}, \neg \widetilde{\Xcal}_{\bar{j}}) - D(\bm{\mu})}{D(\bm{\mu}) D(\bm{\mu}, \neg \widetilde{\Xcal}_{\bar{j}})} \\ 
         & \le \frac{D(\bm{\mu}, \neg \Bcal_{\rho}(\bar{x})) - D(\bm{\mu})}{D(\bm{\mu}) D(\bm{\mu}, \neg \widetilde{\Xcal}_{\bar{j}})} \tag{since $\neg \Bcal_{\rho}(\bar{x}) \subseteq \neg \widetilde{\Xcal}_{\bar{j}}$}\\ 
         & \le \frac{D(\bm{\mu}, \neg \Bcal_{\rho}(\bar{x})) - D(\bm{\mu}, \neg \bar{x})}{D(\bm{\mu}) D(\bm{\mu}, \neg \widetilde{\Xcal}_{\bar{j}})} \tag{since $D(\bm{\mu}) \ge D(\bm{\mu}, \neg x)$}\\ 
         & \le \epsilon,
    \end{align*}
    where 
    the last inequality comes from \Cref{ass:conv}.\footnote{Notice that, since \Cref{ass:conv} holds for any $\bm{\omega} \in \Delta_K$, it holds also for the supremum over $\bm\omega$. Indeed, $D(\bm{\mu}, \neg \Bcal_{\rho}(\bar{x})) - D(\bm{\mu}, \neg \bar{x}) \le D(\bm{\mu}, \bar{\bm{\omega}}, \neg \Bcal_{\rho}(\bar{x})) - D(\bm{\mu}, \bar{\bm{\omega}}, \neg \bar{x})$ for any $\bar{\bm{\omega}}$ that attains the supremum in $D(\bm{\mu}, \neg \Bcal_{\rho}(\bar{x}))$.} Letting $\epsilon \rightarrow 0$ concludes the proof. 
\end{proof}

\section{Continuity Results}
In this section, we provide results on the continuity of the different divergences involved in an infinite answer problem. 

We begin with the following preliminary result.

\begin{lemma}\label{lemma:consequence-ass-compact}
    For all $\bm\mu \in {\Theta}^K$, it holds that:
    \begin{align*}
        D(\bm\mu) = \max_{x \in \mathcal{X}^{\star}(\bm\mu)}D(\bm\mu, \neg x) = \max_{x \in \mathcal{X}} D(\bm\mu, \neg x).
    \end{align*}
\end{lemma}
\begin{proof}
    Let $x \in \mathcal{X}$ and $x \notin \mathcal{X}^{\star}(\bm\mu)$. Then, consider any $\bm\omega \in \Delta_K$, we have that:
    \begin{align}\label{eq:cons-ass-compact}
        D(\bm\mu, \bm\omega, \neg x) = \inf_{\bm\lambda \in \neg x} \sum_{k \in [K]} \omega_k d(\mu_k, \lambda_k) = 0.
    \end{align}    
    Indeed, $\neg x = \{ \bm\lambda \in \mathcal{M}: x \notin \mathcal{X}^{\star}(\bm\lambda) \}$, and thus $\bm\mu \in \neg x$. Them, \Cref{eq:cons-ass-compact} follows by $d(\mu, \mu) = 0$.

    Observing that $D(\bm\mu, \neg x) \ge 0$ for all $x \in \mathcal{X}$ (i.e, \Cref{ass:compact}) concludes the proof.
\end{proof}

\continuity*
\begin{proof}
    \textbf{i)} First, note that, for all $x\in\Xcal$, the function $(\bm\mu, \bm\omega) \rightarrow D(\bm\mu, \bm\omega, \neg x)$ is jointly continuous over $\Theta^K \times \Delta_K$. This is due to \citet[Lemma 27]{degenne2019pure}. Then, it remains to show that $(\bm\mu, \bm\omega, x) \rightarrow D(\bm\mu, \bm\omega, \neg x)$ is jointly continuous over $\Theta^K \times \Delta_K \times \mathcal{X}$. 
    Thus, for all $\bm\mu, \bm\omega,x \in \Theta^K \times \Delta_K \times \mathcal{X}$ and $\forall \epsilon > 0$ there exists $\kappa^\star=\kappa^\star_{\mu,\omega,x,\epsilon} > 0$ such that, for all $\bm\mu', \bm\omega', x': \|\bm\mu - \bm\mu'\|_{\infty} \le \kappa^\star, \|\bm\omega - \bm\omega'\|_{\infty} \le \kappa^\star, \|x-x'\|_{\infty} \le \kappa^\star$, it holds that $|D(\bm\mu', \bm\omega', \neg x') - D(\bm\mu, \bm\omega, \neg x)| \le \epsilon$.
    Define $\kappa^\star = \min\{\tilde \kappa, \bar{\kappa}\}$, where $\bar{\kappa}$ and $\tilde \kappa$ are as follows. First, $\bar{\kappa}$ is such that:
    \begin{align}\label{eq:rmp1}
        D(\bar{\bm\mu}, \bar{\bm\omega}, \neg \Bcal_{\bar{\kappa}}(\bar{x})) - D(\bar{\bm\mu}, \bar{\bm\omega}, \neg \bar{x}) \le \frac{\epsilon}{2} \quad \forall \bar{x} \in \mathcal{X}, \bar{\bm\mu} \in \Theta^K, \bar{\bm\omega} \in \Delta_K.
    \end{align}
    Due to Assumption~\ref{ass:conv} such $\bar{\kappa}$ is guaranteed to exists. Secondly, $\tilde\kappa=\tilde\kappa_{\mu,\omega,x,\epsilon}$ is such that:
    \begin{align}\label{eq:rmp2}
        |D(\bm\mu', \bm\omega',\neg x) - D(\bm\mu, \bm\omega, \neg x)| \le \frac{\epsilon}{2} \quad \forall \bm\mu' \in \Bcal_{\tilde\kappa}(\bm\mu), \bm\omega' \in \Bcal_{\tilde\kappa}(\bm\omega).
    \end{align}
    Such $\tilde\kappa$ is guaranteed to exists due to the continuity of $(\bm\mu, \bm\omega) \rightarrow D(\bm\mu, \bm\omega, \neg x)$ for any fixed $x \in \mathcal{X}$ due to \citet[Lemma 27]{degenne2019pure}.
    
    Then, we analyze $|D(\bm\mu', \bm\omega', \neg x') - D(\bm\mu, \bm\omega, \neg x)|$ by studying an upper bound on the sum of the following terms: $|D(\bm\mu', \bm\omega', \neg x') - D(\bm\mu', \bm\omega', \neg x)|$ and $|D(\bm\mu', \bm\omega', \neg x) - D(\bm\mu, \bm\omega, \neg x)|$. We start with the former. Suppose, $D(\bm\mu', \bm\omega', \neg x') > D(\bm\mu', \bm\omega', \neg x)$, then
    \begin{align*}
        |D(\bm\mu', \bm\omega', \neg x') - D(\bm\mu', \bm\omega', \neg x)| & =D(\bm\mu', \bm\omega', \neg x') - D(\bm\mu', \bm\omega', \neg x)\\
        &\le D(\bm\mu', \bm\omega', \neg \Bcal_{\bar{\kappa}}(x)) - D(\bm\mu', \bm\omega', \neg x) \\ & \le \frac{\epsilon}{2}.
    \end{align*}
    where the first inequality holds since $x'\subseteq \Bcal_{\kappa^\star}(x)\subseteq \Bcal_{\bar \kappa}(x)$ and thus $\lnot x'\supseteq \lnot\Bcal_{\bar \kappa}(x)$, and the second inequality follows from \Cref{eq:rmp1}.
    Equivalently, if $D(\bm\mu', \bm\omega', \neg x') \le D(\bm\mu', \bm\omega', \neg x)$ then
    \begin{align*}
        |D(\bm\mu', \bm\omega', \neg x') - D(\bm\mu', \bm\omega', \neg x)| & = D(\bm\mu', \bm\omega', \neg x)-D(\bm\mu', \bm\omega', \neg x')\\
        &\le D(\bm\mu', \bm\omega', \neg \Bcal_{\bar{\kappa}}(x')) - D(\bm\mu', \bm\omega', \neg x') \\ 
        & \le \frac{\epsilon}{2},
    \end{align*}
    where the first inequality holds since $x\subseteq \Bcal_{\kappa^\star}(x')\subseteq \Bcal_{\bar \kappa}(x')$ and thus $\lnot x\supseteq \lnot\Bcal_{\bar \kappa}(x')$, and the second inequality follows from \Cref{eq:rmp1}.
    
    Similarly, for the second term, we have that:
    \begin{align*}
        |D(\bm\mu', \bm\omega', \neg x) - D(\bm\mu, \bm\omega, \neg x)| \le \frac{\epsilon}{2},
    \end{align*}
    by the fact that $\bm\mu' \in \Bcal_{\kappa^\star}(\bm\mu), \bm\omega' \in \Bcal_{\kappa^\star}(\bm\omega)$ and \Cref{eq:rmp2}. This concludes the proof of the first statement.

    The other three statements follow from various applications of Berge's maximum theorem (\Cref{thm:berge}).
    
    \textbf{ii)} 
    $D(\bm{\mu}, \bm{\omega}, \neg x)$ is continuous over $\Theta^K \times \Delta_K \times \Xcal$ and the maximization is over the simplex.

    \textbf{iii)} The third claim is due to Berge's maximum theorem. Indeed, $\Xcal^\star(\bm{\mu})$ is continuous and compact-valued and $ D(\bm{\mu}, \bm{\omega}, \neg x)$ is continuous. Hence, $(\bm{\mu}, \bm{\omega}) \rightarrow \max_{x \in \Xcal^\star(\bm{\mu})} D(\bm{\mu}, \bm{\omega}, \neg x)$ is continuous and $(\bm{\mu}, \bm{\omega}) \rightrightarrows \argmax_{x \in \Xcal^\star(\bm{\mu})} D(\bm{\mu}, \bm{\omega}, \neg x)$ is upper hemicontinuous and compact-valued.

    \textbf{iv)} Finally, the last claim follows by applying the Berge's maximum theorem. We recall from \Cref{lemma:consequence-ass-compact} that $D(\bm\mu) = \max_{x \in \mathcal{X}} D(\bm\mu, \neg x)$. Since $D(\bm{\mu}, \neg x)$ is continuous and $\bm{\mu} \rightrightarrows \mathcal{X}$ is constant and compact-valued (\Cref{ass:compact}), we obtain that $x_F(\bm{\mu})$ is upper hemicontinuous and compact-valued.  Furthermore, since $\max_{x \in \mathcal{X}^{\star}(\bm{\mu})} D(\bm{\mu}, \bm{\omega}, \neg x)$ is continuous and the simplex is a constant and compact set, we have that $\bm{\mu} \rightrightarrows \bm{\omega}^*(\bm{\mu})$ is upper hemicontinous and compact-valued.
\end{proof}

\begin{corollary}[Uniform Continuity]\label{corollary:unif-cont}
    Let $\mathcal{C} \subseteq \Theta^K$ and $\mathcal{H} \subseteq \mathcal{X}$ be compact sets. Then, we have that:
    \begin{itemize}
        \item $(\bm{\omega}, \bm{\mu}, x) \rightarrow D(\bm{\mu}, \bm{\omega}, \neg x)$ is uniformly continuous on $\mathcal{C} \times \Delta_K \times \mathcal{H}$.
        \item $(\bm{\mu},x) \rightarrow D(\bm{\mu}, \neg x)$ is uniformly continuous on $\mathcal{C} \times \mathcal{H}$.
        \item $(\bm{\mu}, \bm{\omega}) \rightarrow \max_{x \in \mathcal{X}^{\star}(\bm{\mu})} D(\bm{\mu}, \bm{\omega}, \neg x)$ is uniformly continuous on $\mathcal{C} \times \Delta_K$.
        \item $\bm{\mu} \rightarrow D(\bm{\mu})$ is uniformly continuous on $\mathcal{C}$
    \end{itemize}
\end{corollary}
\begin{proof}
    By Heine–Cantor, every continuous function is uniformly continuous on a compact domain. Then, apply \Cref{lemma:continuity}.
\end{proof}

\section{Algorithm Analysis}

We structure this section as follows. First, in \Cref{app:correct-aux}, we prove the correctness of the stopping and recommendation rules that we introduced in \Cref{sec:algo}. Then, we present the analysis of Sticky Sequence TaS for any answer selection rule that ensures convergence (\Cref{app:ss-tas}). Finally, in \Cref{app:alg-conv}, we show how to build such converging sequences according to the properties of $\Xcal$ and $\bm\mu \mapstoto \Xcal^{\star}(\bm\mu)$. 

\subsection{Correctness and Expected Stopping Time}\label{app:correct-aux}

We first recall the stopping and recommendation rules. Specifically, the stopping rule is as follows:
\begin{align}\label{eq:stopping-rule}
    \tau_\delta = \inf\left\{t \in \mathbb{N}:  \max_{x \in \Xcal^\star(\bm{\hat{\mu}}(t))} D\left( \bm{\hat{\mu}}(t), \bm{N}(t), \lnot x \right) > \beta_{t,\delta} \right\}.
\end{align}
Furthermore, it recommends:
\begin{align}\label{eq:recc-rule}
    \hat{x}_{\tau_\delta} \in \argmax_{x \in \Xcal^\star(\bm{\hat{\mu}}(t))} D\left( \bm{\hat{\mu}}(t), \bm{N}(t), \lnot x \right).
\end{align}
The following correctness analysis is based on concentration arguments that have been initially described by \citet{menard2019gradient}, where the author is working under the assumption of Gaussian distributions with unitary variance. Specifically, the stopping threshold $\beta_{t,\delta}$ is set to:
\begin{align}\label{eq:beta-thr}
    \beta_{t,\delta} = \log\left( \frac{1}{\delta} \right) + K \log \left( 4 \log\left( \frac{1}{\delta} \right) + 1 \right) + 6K \log(\log(t)+3) + K\tilde{C},
\end{align}
where $\tilde{C}$ is a universal constant. 
All the results of this section can be extended (with a more complex notation) to distributions in any canonical exponential family by modifying appropriately $\beta_{t,\delta}$ (see e.g., \cite{kaufmann2021mixture}).

\begin{lemma}[Correctness]\label{lemma:correctness}
    For any sampling rule $(A_t)_{t  \ge 1}$, the stopping and recommendation rules in Equations~\eqref{eq:stopping-rule}-\eqref{eq:recc-rule} guarantee that:
    \begin{align*}
        \mathbb{P}_{\bm{\mu}} \left(  \hat{x}_{\tau_\delta} \notin \Xcal^\star(\bm{\mu}) \right) \le \delta.
    \end{align*}
\end{lemma}
\begin{proof}
    It holds that:
    \begin{align*}
        \mathbb{P}_{\bm{\mu}} \left(  \hat{x}_{\tau_\delta} \notin \Xcal^\star(\bm{\mu}) \right) & = \mathbb{P}_{\bm{\mu}}\left( \exists t \in \mathbb{N}: \max_{x \in \Xcal^\star(\bm{\hat{\mu}}(t))} D\left( \bm{\hat{\mu}}(t) ,\bm{N}(t), \neg x \right) > \beta_{t,\delta} \text{ and } \hat{x}_t \notin \Xcal^\star(\bm{\mu})  \right) \\ 
        & \le \mathbb{P}_{\bm{\mu}}\left( \exists t \in \mathbb{N} \text{ and } x \notin \Xcal^\star(\bm{\mu}) : D\left(\bm{\hat{\mu}}(t), \bm{N}(t), \neg x \right) > \beta_{t,\delta} \right) \\ 
        & \le \mathbb{P}_{\bm{\mu}}\left( \exists t \in \mathbb{N}: \sum_{k \in [K]} N_k(t) d(\hat{\mu}_k(t), \mu_k) > \beta_{t,\delta} \right)  \\ 
        & \le \delta,
    \end{align*}
    where (i) in the first step we have used Equations~\eqref{eq:stopping-rule}-\eqref{eq:recc-rule}, (ii) in the third one the fact that, since $x \notin \Xcal^\star(\bm{\mu})$, then $\bm{\mu} \in \neg x$, and (iii) in the last one Proposition 1 in \cite{menard2019gradient}.    
\end{proof}

The following results will be used to control the expected stoppping time of the proposed algorithms. \Cref{lemma:control-expect} is a standard result; we include a proof for completeness.

\begin{lemma}\label{lemma:control-expect}
    Let $\{\mathcal{E}(t)\}_t$ be a sequence of $\mathcal{F}_t$-measurable events. Suppose there exists $T_0(\delta) \in \mathbb{N}$ such that, for all $t \ge T_0(\delta)$, $\mathcal{E}(t) \subseteq \{ \tau_\delta < t \}$. Then, it holds that:
    \begin{align*}
        \mathbb{E}_{\bm{\mu}}[ \tau_\delta] \le T_0(\delta) +\sum_{t=0}^{+\infty} \mathbb{P}_{\bm{\mu}} (\mathcal{E}(t)^c). 
    \end{align*}
\end{lemma}
\begin{proof}
    By standard probabilistic arguments, we have that:
    \begin{align*}
        \mathbb{E}_{\bm{\mu}}[ \tau_\delta] = \sum_{t=0}^{+\infty} \mathbb{P}_{\bm{\mu}}(\tau_\delta > t) \le T_0(\delta) + \sum_{t=T_0(\delta)}^{+\infty} \mathbb{P}_{\bm{\mu}}(\tau_\delta > t)  \le T_0(\delta) + \sum_{t=0}^{+\infty} \mathbb{P}_{\bm{\mu}}(\mathcal{E}(t)^c),
    \end{align*}
    which concludes the proof.
\end{proof}

\subsection{Sticky-Sequence Track-and-Stop}\label{app:ss-tas}

In this section, we prove \Cref{theo:sticky-seq}, \ie we analyze Sticky-Sequence Track-and-Stop for any general answer selection rule that satisfies the fact that $\{ x_t \}_t$ is converging to some $\bar{x} \in \Xcal_F(\bm\mu)$. 

The analysis will be carried out under the good event $\mathcal{E}_t = \bigcap_{s=h(t)}^t \{ \bm\mu \in C_s \}$ where $h(t) = \lceil \sqrt{t} \rceil$ and $C_s = \{ \bm\mu' \in \mathcal{M}: D(\hat{\bm\mu}(t), \bm{N}(t), \bm\mu') \le \log(g(t)) \}$ where $g(t) = C t^{10}$ for some constant $C \in \mathbb{R}$. It is possible to show that, for an appropriate value of $C$, it holds that $\sum_{t=0}^{\infty} \mathbb{P}_{\bm\mu} (\mathcal{E}(t)^c)$ is finite (see Lemma~14 in \cite{degenne2019pure}). Next, we first recall the definition of converging sequence $\{ x_t \}_t$. Specifically, for all $\epsilon > 0$, there exists $T_{\epsilon} \in \mathbb{N}$ such that, for all $t \ge T_{\epsilon}$, on $\mathcal{E}(t)$, there exists $\bar{x} \in \Xcal_F(\bm\mu): \| x_s - \bar{x} \| \le \epsilon$ for all $s \ge h(t)$.

Now that we clarified the setup, we first prove that, on $\mathcal{E}(t)$, $C_t$ is shrinking toward $\bm\mu$. Note that this result is independent from \Cref{def:conv-answer} and only depends on the definition of $\mathcal{C}_t$, $\mathcal{E}_t$ and the forced exploration of the algorithm (\ie the cumulative tracking procedure, \Cref{lemma:tracking}).
 
\begin{lemma}\label{lemma:ct-shrink}
    For all $\epsilon > 0$, there exists ${T}_{\epsilon}$ such that, for all $t \ge T_{\epsilon}$, on $\mathcal{E}(t)$, $\|\bm{\mu}-\bm{\mu}'\|_{\infty} \le \epsilon$ holds $\forall\bm{\mu}' \in C_s$ and all $s \ge h(t)$.
\end{lemma}
\begin{proof}
    For any $\bm{\mu}, \bm{\mu}'$, let $ch(\bm{\mu, \bm{\mu}'})=\inf_{\bm{\lambda} \in \mathbb{R}^K} \sum_{k \in [K]} \left( d(\lambda_k, \mu_k) + d(\lambda_k,\mu'_k) \right)$. 

    Consider $T_{\epsilon} \ge \bar{T}$, where $\bar{T}$ is such that $\sqrt{h(\bar{T})+K^2} - 2K > 0$ and let $t \ge T_{\epsilon}$. We recall that, by definition, on $\mathcal{E}(t)$, it holds that $\bm{\mu} \in C_s$ for all $s \ge h(t)$. Thus, $D(\hat{\bm\mu}(s), \bm{N}(s), \bm\mu) \le \log(g(s))$ for all $s \ge h(t)$. Furthermore, for all $\bm\mu' \in C_s$, $D(\hat{\bm\mu}(s), \bm{N}(s), \bm\mu') \le \log(g(s))$. Thus, on $\mathcal{E}(t)$, for all $s \ge h(t)$ and $\bm\mu' \in C_s$, we have that:
    \begin{align*}
        \sum_{k \in [K]} N_k(s) \left( d(\hat{\mu}_k(s), \mu_k) + d(\hat{\mu}_k(s), \mu_k') \right) \le 2 \log(g(s)),
    \end{align*}
    Using Lemma~\ref{lemma:tracking} together with the fact that $t$ is such that $\sqrt{h(t)+K^2} - 2K > 0$ (since $t \ge \bar{T}$), we obtain that, on $\mathcal{E}(t)$, $ch(\bm{\mu}, \bm{\mu}') \le \frac{2\log(g(s))}{\sqrt{s+K^2}-2K},~\forall \bm\mu' \in \mathcal{C}_s$. 
    
    Now, observe that $\frac{2\log(g(s))}{\sqrt{s+K^2}-2K}$ is decreasing in $s$. It follows that, for all $\epsilon> 0$, there exists $s_{\epsilon}$ such that $\forall s \ge s_{\epsilon}$, $\|\bm{\mu}-\bm{\mu}'\|_{\infty} \le \epsilon$.  Indeed, let us analyze $ch(\bm{\mu}, \bm{\mu}')$. Using the sub-gaussianity of the arms 
    together with the fact that $ch(\bm{\mu}, \bm{\mu}') \le \frac{2\log(g(s))}{\sqrt{s+K^2}-2K}$ we obtain:\footnote{We recall that sub-gaussianity implies $d(p,q) \ge \frac{(p-q)^2}{2\sigma^2}$.}
    \begin{align*}
        \frac{2\log(g(s))}{\sqrt{s+K^2}-2K} & \ge ch(\bm\mu, \bm\mu') \\ & \ge \inf_{\bm\lambda \in \mathbb{R}^K}\sum_{k \in [K]} \frac{(\lambda_k - \mu_k)^2 + (\lambda_k - \mu'_k)^2}{2\sigma^2} \\ 
        & = \frac{1}{2\sigma^2} \sum_{k \in [K]} \left( \frac{\mu_k+\mu_k'}{2} - \mu_k \right)^2 + \left( \frac{\mu_k+\mu_k'}{2} - \mu'_k \right)^2  \\ 
        & = \frac{1}{4\sigma^2} \sum_{k \in [K]} (\mu_k - \mu'_k)^2 \\ 
        & = \frac{1}{4\sigma^2} \| \bm\mu - \bm\mu' \|^2_2 \\ 
        & \ge \frac{1}{4\sigma^2} \| \bm\mu - \bm\mu'\|_{\infty}^2.
    \end{align*} 
    Thus, we have that, on $\mathcal{E}(t)$, for $s \ge h(t)$ and $\bm\mu' \in C_s$:
    \begin{align*}
        \| \bm\mu - \bm\mu' \|_{\infty} \le \sqrt{\frac{8\sigma^2 \log(g(s))}{\sqrt{s+K^2}-2K}}.
    \end{align*}

    Let $T_{\epsilon} \coloneqq \max\left\{\bar{T}, \inf \left\{n \in \mathbb{N}: \sqrt{\frac{8\log(g(h(n)))}{\sqrt{h(n)+K^2}-2K}} \le \epsilon \right\} \right\}$. Then, for all $t \ge T_{\epsilon}$, $\|\bm{\mu}-\bm{\mu}'\|_{\infty} \le \epsilon$ holds on $\mathcal{E}(t)$ for all $\bm{\mu}' \in C_s$ and $s \ge h(t)$. 
\end{proof}

At this point, we combine \Cref{lemma:ct-shrink} with \Cref{lemma:continuity} and \Cref{corollary:unif-cont} to prove the following result. \Cref{lemma:smoothness-conv-answer} will be used to analyze the behavior of the algorithm in \Cref{lemma:good-event-analysis-conv-answer}, which is at the core of our asymptotic optimality result.

\begin{lemma}\label{lemma:smoothness-conv-answer}
    If \Cref{def:conv-answer} is satisfied, then, for all $\kappa > 0$ there exists $T_{\kappa} \in \mathbb{N}$ such that for all $t \ge T_{\kappa}$, on $\mathcal{E}(t)$, it holds that, for all $s \ge h(t)$ and $\bm\mu' \in C_s$:
    \begin{align*}
        & \Big| D(\bm\mu) - D(\bm\mu') \Big| \le \kappa, \\ 
        & \Big| \max_{x \in \Xcal_F(\bm\mu)} D(\bm\mu, \bm\omega, \neg x) - \max_{x \in \Xcal_F(\bm\mu')} D(\bm\mu', \bm\omega, \neg x) \Big| \le \kappa \quad \forall \bm\omega \in \Delta_K. 
    \end{align*}
    Furthermore, there exists $\bar{x} \in \mathcal{X}_F(\bm\mu)$ such that, on $\mathcal{E}(t)$:
    \begin{align*}
        \Big| D(\bm\mu, \bm\omega, \neg \bar{x}) - D(\bm\mu', \bm\omega, \neg x_s) \Big| \le \kappa, \quad \forall \bm\omega \in \Delta_K, \bm\mu' \in C_s, s \ge h(t). 
    \end{align*}
\end{lemma}
\begin{proof}
    First, from \Cref{corollary:unif-cont}, we have that for all $\kappa > 0$, there exists $\beta_{\kappa_1}$ such that, for all $\bm\mu': \|\bm\mu - \bm\mu'\|_{\infty} \le \beta_{\kappa_1}$:
    \begin{align*}
        & \Big| D(\bm\mu) - D(\bm\mu') \Big| \le \kappa, \\ 
        & \Big| \max_{x \in \Xcal_F(\bm\mu)} D(\bm\mu, \bm\omega, \neg x) - \max_{x \in \Xcal_F(\bm\mu')} D(\bm\mu', \bm\omega, \neg x) \Big| \le \kappa \quad \forall \bm\omega \in \Delta_K. 
    \end{align*}
    Furthermore, from \Cref{lemma:ct-shrink}, we know that, for all $\epsilon$, there exists $T_{\epsilon}$ such that, for all $t \ge T_{\epsilon}$, on $\mathcal{E}(t)$, $\|\bm\mu - \bm\mu' \| \le \epsilon$ holds for all $s \ge h(t)$ and $\bm\mu' \in C_s$. Thus, picking $\epsilon = \beta_{\kappa_1}$, we obtain that there exists $\bar{T}_1$ such that, for all $t \ge \bar{T}_1$, the first claim holds. 

    Second, from \Cref{corollary:unif-cont} and the compactness of $\Xcal$ (\ie \Cref{ass:compact}), we know that, for all $\kappa > 0$, there exists $\beta_{\kappa_2}$ such that, for all $\bm\mu'$ such that $\|\bm\mu - \bm\mu' \|_{\infty} \le \beta_{\kappa_2}$ and for all $x, x' \in \Xcal$ such that $\|x-x' \|_{\infty} \le \beta_{\kappa_2}$, we have that:
    \begin{align*}
        \Big| D(\bm\mu, \bm\omega, \neg x) - D(\bm\mu', \bm\omega, \neg x') \Big| \le \kappa, \quad \forall \bm\omega \in \Delta_K.
    \end{align*}
    From \Cref{def:conv-answer}, for all $\tilde{\epsilon} > 0$, there exists $T_{\tilde{\epsilon}}$ such that, for all $t \ge T_{\tilde{\epsilon}}$, on $\mathcal{E}(t)$, there exists $\bar{x} \in \Xcal_F(\bm\mu)$ such that $\|\bar{x}  - x_s\| \le \epsilon$ holds for all $s \ge h(t)$. 
    Furthermore, from \Cref{lemma:ct-shrink}, we know that, for all $\epsilon$, there exists $T_{\epsilon}$ such that, for all $t \ge T_{\epsilon}$, on $\mathcal{E}(t)$, $\|\bm\mu - \bm\mu' \| \le \epsilon$ holds for all $s \ge h(t)$ and $\bm\mu' \in C_s$. Picking $\tilde{\epsilon} \le \beta_{\kappa_2}$ and $\epsilon \le \beta_{\kappa_2}$, we have that there exists $\bar{T}_2$ such that for all $t \ge \bar{T}_2$, the second claim holds. Taking $T_{\kappa} = \max\{ \bar{T}_1, \bar{T}_2 \}$ concludes the proof.
\end{proof}

Next, the following Lemma analyzes the behavior of the algorithm in terms of the l.h.s. of the stopping rule in \Cref{eq:stopping-rule}. Specifically, it relates the stopping rule to the characteristic time $T^*(\bm\mu)$.

\begin{lemma}\label{lemma:good-event-analysis-conv-answer}
    Suppose that \Cref{def:conv-answer} holds. Then, for all $\kappa > 0$, there exists $T_{\kappa} \in \mathbb{N}$ such that, for all $t \ge T_{\kappa}$, on $\mathcal{E}(t)$, it holds that:
    \begin{align*}
        \frac{1}{t} \max_{x \in \Xcal^{\star}(\bm{\hat{\mu}}(t))} D(\bm{\hat{\mu}}(t), \bm{N}(t), \neg x)  \ge & \frac{t - \lceil h(t) \rceil }{t} T^*(\bm{\mu})^{-1} -  3\kappa - \frac{K(1+\sqrt{t})}{t} C_{\bm{\mu}}
    \end{align*}
    for some problem dependent constant $C_{\bm{\mu}} > 0$.
\end{lemma}
\begin{proof}
    Let $\kappa > 0$ and consider $t \ge T_{\kappa}$ such that the results of \Cref{lemma:smoothness-conv-answer} holds. 
    
    Then, we have that:
    \begin{align*}
        \frac{1}{t} \max_{x \in \Xcal^{\star}(\bm{\hat{\mu}}(t))} D(\bm{\hat{\mu}}(t), \bm{N}(t), \neg x) = \frac{1}{t} \max_{i \in \Xcal^{\star}(\bm{{\mu}})} D(\bm{{\mu}}, \bm{N}(t), \neg x)  - h_1(t),
    \end{align*}
    where $h_1(t)$ is given by:
    \begin{align*}
        h_1(t) & \coloneqq \frac{1}{t} \max_{x \in \Xcal^{\star}(\bm{{\mu}})}  D(\bm{{\mu}}, \bm{N}(t), \neg x) - \frac{1}{t} \max_{x \in \Xcal^{\star}(\bm{\hat{\mu}}(t))}  D(\bm{\hat{\mu}}(t), \bm{N}(t), \neg x).
    \end{align*}
    From \Cref{lemma:smoothness-conv-answer}, we obtain $h_1(t) \le \kappa$, thus leading to:
    \begin{align*}
        \frac{1}{t} \max_{x \in \Xcal^{\star}(\bm{\hat{\mu}}(t))} D(\bm{\hat{\mu}}(t), \bm{N}(t), \neg x) \ge \frac{1}{t} \max_{x \in \Xcal^{\star}(\bm{{\mu}})} D(\bm{{\mu}}, \bm{N}(t), \neg x) - \kappa.
    \end{align*}

    Furthermore, for all $\bar{x} \in \Xcal^{\star}(\bm\mu)$, we have that:
    \begin{align*}
        \frac{1}{t} \max_{x \in \Xcal^{\star}(\bm{{\mu}})} D(\bm{{\mu}}, \bm{N}(t), \neg x) & \ge \frac{1}{t} D(\bm{{\mu}}, \bm{N}(t), \neg \bar{x})  \\ & \ge \frac{1}{t} D\left(\bm{\mu}, \sum_{s = 1}^t \bm{\omega}(s), \neg \bar{x}\right) - h_2(t),
    \end{align*}
    where $h_2(t)$ is given by:
    \begin{align*}
        h_2(t) & \coloneqq \frac{1}{t} \sup_{x \in \Xcal^{\star}(\bm\mu)} \inf_{\bm{\lambda} \in \neg x} \sum_{k \in [K]} \left( \sum_{s=1}^t \omega_k(s) - N_k(t) \right) d(\mu_k, \lambda_k) \\ & \le \frac{K(1+\sqrt{t})}{t} \sup_{x \in \Xcal^{\star}(\bm\mu)} \inf_{\bm{\lambda} \in \neg x} \sum_{k \in [K]} d(\mu_k, \lambda_k) \\
        & \coloneqq \frac{K(1+\sqrt{t})}{t} C_{\bm{\mu}},
    \end{align*}
    where, in the first inequality, we used Lemma~\ref{lemma:tracking} and in the last one the fact that the exponential family is regular and bounded.
    Thus, we obtained that, for all $\bar{x} \in \Xcal^{\star}(\bm\mu)$:
    \begin{align*}
        \frac{1}{t} \max_{x \in \Xcal^{\star}(\bm{\hat{\mu}}(t))} D(\bm{\hat{\mu}}(t), \bm{N}(t), \neg x) \ge \frac{1}{t} D\left(\bm{\mu}, \sum_{s = 1}^t \bm{\omega}(s), \neg \bar{x} \right) - \kappa - \frac{K(1+\sqrt{t})}{t} C_{\bm{\mu}}.
    \end{align*}
    Therefore, it also holds for $\bar{x} \in \Xcal_F(\bm\mu)$ such that, on $\mathcal{E}(t)$,
    \begin{align*}
        \Big| D(\bm\mu, \bm\omega, \neg \bar{x}) - D(\bm\mu', \bm\omega, \neg x_s) \Big |\le \kappa, \quad \forall \bm\omega \in \Delta_K, \bm\mu' \in C_s, s \ge h(t), 
    \end{align*}
    where the existence of such an $\bar{x}$ is given in \Cref{lemma:smoothness-conv-answer}. Thus, focus on $\frac{1}{t} D\left(\bm{\mu}, \sum_{s = 1}^t \bm{\omega}(s), \neg \bar{x} \right)$:
    \begin{align*}
        \frac{1}{t} D\left(\bm{\mu}, \sum_{s = 1}^t \bm{\omega}(s), \neg \bar{x} \right) & \ge \frac{1}{t} \sum_{s = h(t)}^t D(\bm\mu, \bm\omega(s), \neg \bar{x})  \\ 
        & = \frac{1}{t} \sum_{s=h(t)}^{t} D(\bm\mu'(s), \bm\omega(s), \neg x_s) - h_3(t) \\ 
        & = \frac{1}{t} \sum_{s=h(t)}^t D(\bm\mu'(s)),
    \end{align*}
    where $\bm\mu'(s) \in C_s$ is such that $\bm\omega(s) \in \bm\omega^*(\bm\mu'(s), \neg x_s)$ and $h_3(t)$ is given by:
    \begin{align*}
        h_3(t) = \frac{1}{t} \sum_{s=h(t)}^t \left( D(\bm\mu'(s), \bm\omega(s), \neg x_s)  - D(\bm\mu, \bm\omega(s), \neg \bar{x})\right) \le \kappa,
    \end{align*}
    where in the inequality step we used \Cref{lemma:smoothness-conv-answer}. Finally, we have that:
    \begin{align*}
        \frac{1}{t} \sum_{s=h(t)}^t D(\bm\mu'(s)) = \frac{(t - h(t))}{t} T^*(\bm\mu)^{-1} - h_4(t),
    \end{align*}
    where $h_4(t)$ is given by:
    \begin{align*}
        h_4(t) = \frac{1}{t} \sum_{s=h(t)}^t D(\bm\mu) - D(\bm\mu'(s)) \le \kappa,
    \end{align*}
    where the inequality step follows from \Cref{lemma:smoothness-conv-answer}, thus concluding the proof.
    \end{proof}

Finally, \Cref{lemma:good-event-analysis-conv-answer} allows to prove the optimality of Sticky Sequence Track-and-Stop. 

\convanswer*
\begin{proof}
    Let $\kappa > 0$. From \Cref{lemma:good-event-analysis-conv-answer}, for $t \ge T_{\kappa}$, on $\mathcal{E}(t)$, it holds that:
    \begin{align}\label{eq:conv}
        \frac{1}{t} \max_{x \in \Xcal^{\star}(\bm{\hat{\mu}}(t))} D(\bm{\hat{\mu}}(t), \bm{N}(t), \neg x)  \ge & \frac{t - \lceil h(t) \rceil }{t} T^*(\bm{\mu})^{-1} - 3 \kappa - \frac{K(1+\sqrt{t})}{t} C_{\bm{\mu}}, 
    \end{align}
    where $C_{\bm{\mu}}$ is a problem-dependent constant.
    
    Let $\gamma \in \left(0, \frac{T^*(\bm{\mu})^{-1}}{4} \right]$, take $\kappa = \kappa(\gamma) \le \frac{\gamma}{12}$. Furthermore, consider $T_{\gamma}$ as the smallest integer $n \in \mathbb{N}$ such that $\frac{h(n)}{n} T^*(\bm\mu)^{-1} \le \frac{\gamma}{4}$ and $\frac{K(1+\sqrt{n})}{n} C_{\bm\mu} \le \frac{\gamma}{4}$. Then, for $t \ge T_{\gamma} + T_{\kappa(\gamma)}$, \Cref{eq:conv} implies:
    \begin{align*}
        \frac{1}{t} \max_{x \in \Xcal^{\star}(\bm\mu)} D(\hat{\bm\mu}(t), \bm{N}(t), \neg x) \ge T^*(\bm\mu)^{-1} - \gamma.
    \end{align*}
    Applying \Cref{lemma:tech-lemma} with $\alpha = L = \gamma$ and $D = T^*(\bm\mu)^{-1}$, we have that for $t \ge \lceil T_0(\gamma, \gamma, \delta, T^*(\bm\mu)^{-1}) + T_{\gamma} + T_{\kappa(\gamma)} \rceil$, we have that:
    \begin{align*}
        \frac{1}{t} \max_{x \in \Xcal^{\star}(\bm\mu)} D(\hat{\bm\mu}(t), \bm{N}(t), \neg x) \ge \frac{\beta_{t,\delta}}{t},
    \end{align*}
    which implies that, on the good event, the algorithm stops using at most $\lceil T_0(\gamma, \gamma, \delta, T^*(\bm\mu)^{-1}) + T_{\gamma} + T_{\kappa(\gamma)} \rceil$ samples. 
    Applying \Cref{lemma:control-expect} we, thus, obtain:
    \begin{align*}
        \E_{\bm\mu}[\tau_\delta] \le T_0(\gamma, \gamma, \delta, T^*(\bm\mu)^{-1}) + T_\gamma + T_{\kappa(\gamma)} + \sum_{s=0}^t \mathbb{P}_{\bm\mu}(\mathcal{E}(s)^c).
    \end{align*}
    From Lemma~14 in \cite{degenne2019pure}, $\sum_{s=0}^t \mathbb{P}_{\bm\mu}(\mathcal{E}(s)^c)$ is finite. Thus, using the expression of $T_0(\gamma, \gamma, \delta, T^*(\bm\mu)^{-1})$ given by \Cref{lemma:tech-lemma}, we have that:
    \begin{align*}
        \limsup_{\delta \to 0} \frac{\E_{\bm\mu}[\tau_\delta]}{\log(1/\delta)} \le \frac{1}{T^*(\bm\mu)^{-1}-2\gamma}.
    \end{align*}
    Letting $\gamma \to 0$ concludes the proof.
\end{proof}

\subsection{Algorithms for Converging Sequences}\label{app:alg-conv}

\subsubsection{Sticky Sequence Track-and-Stop for the case in which $|\Xcal_F(\mu)|$ is unique}
We now show that \Cref{def:conv-answer} holds whenever $\Xcal_F(\bm\mu)$ is finite by simply picking any $x_t \in \Xcal_t$.

\begin{lemma}\label{lemma:unique-conv-seq}
    Let $\{x_t\}_t$ be such that $x_t \in \Xcal_t$ for all $t \in \mathbb{N}$. Then, if $|\Xcal_F(\bm\mu)|=1$ for all $\bm\mu \in \Theta^K$, then \Cref{def:conv-answer} holds, i.e., for all $\epsilon > 0$, there exists ${T}_{\epsilon}$ such that, for all $t \ge {T}_{\epsilon}$, on $\mathcal{E}(t)$, it holds that $|x_s - \bar{x}| \le \epsilon$ for all $s \ge h(t)$ for $\bar{x} = \Xcal_F(\bm\mu)$.
\end{lemma}
\begin{proof}
    From \Cref{lemma:continuity}, $\bm\mu \mapstoto \Xcal_F(\bm\mu)$ is upper hemicontinuous. Furthermore, since $\Xcal_F(\bm\mu)$ is single-valued, then $\bm\mu \mapstoto \Xcal_F(\bm\mu)$ is a continuous function. Thus, for all $\epsilon > 0$, there exists $\eta_{\epsilon} > 0$ such that, for all $\bm\mu' \in \mathcal{B}_{\eta_\epsilon}(\bm\mu)$, we have that $\| \Xcal_F(\bm\mu) - \Xcal_F(\bm\mu') \|_{\infty} \le \epsilon$. 

    Furthermore, from \Cref{lemma:ct-shrink}, for all $\kappa > 0$, there exists $T_\kappa \in \mathbb{N}$ such that, for all $t \ge  T_{\kappa}$, on $\mathcal{E}_t$, $\| \bm\mu - \bm\mu' \|_{\infty} \le \kappa$ for all $\bm\mu' \in C_s$ and $s \ge h(t)$. It then follows that, for $\kappa = \eta_\epsilon$, we have that $\| \Xcal_F(\bm\mu) - \Xcal_F(\bm\mu') \|_{\infty} \le \epsilon$ for all $\bm\mu' \in C_s$ and $s \ge h(t)$. Then, since $x_s \in \Xcal_s$ and $\Xcal_s = \bigcup_{\bm\mu' \in C_s} \Xcal_F(\bm\mu)$, we also have that, on $\mathcal{E}(t)$, $\|x_s - \Xcal_F(\bm\mu)\|_{\infty} \le \epsilon,~\forall s \ge h(t)$, which concludes the proof.
\end{proof}

As for the case in which $|\Xcal_F(\bm\mu)|=1$, we note that \Cref{lemma:unique-conv-seq} is actually slightly stronger than \Cref{def:conv-answer} as we know exactly the answer toward which we are converging.

\subsubsection{Sticky Sequence Track-and-Stop for the case in which $\Xcal\subseteq \mathbb{R}$}

We show that \Cref{def:conv-answer} holds when $\Xcal \subseteq \mathbb{R}$ and for $x_t \in \argmin_{x \in \Xcal_t} x$. In the following, we assume that $x_t$ is attained within $\Xcal_t$. If this is not the case, the inf will be for sure attained on the closure of $\Xcal_t$ (as $\cl(\Xcal_t)$ is a compact set). In this case, we can simply pick any $x_t \in \Xcal_t$ arbitrary close to $\argmin_{x \in \cl(\Xcal_t)} x$ and the proof follows by identical arguments.

\begin{lemma}\label{lemma:conv-real}
    Let $\Xcal \subset \mathbb{R}$.
    and $\{ x_t \}_t$ be such that $x_t \in \argmin_{x \in \Xcal_t} x$ for all $t \in \mathbb{N}$. Then, \Cref{def:conv-answer} holds, \ie
    for all $\epsilon > 0$, there exists ${T}_{\epsilon}$ such that, for all $t \ge {T}_{\epsilon}$, on $\mathcal{E}(t)$, it holds that $|x_s - \bar x| \le \epsilon$ for all $s \ge h(t)$ for $\bar x \in \argmin_{x \in \Xcal_F(\bm{\mu})} x$.
\end{lemma}
\begin{proof}
    By upper hemicontinuity of $\Xcal_F(\bm{\mu})$, we have that, for all $\epsilon > 0$, there exists $\rho_{\epsilon} > 0$ such that, if $\|\bm{\mu}- \bm{\mu}'\|_{\infty} \le \rho_\epsilon$, then $\Xcal_F(\bm{\mu}') \subset \Bcal_{\epsilon}(\Xcal_F(\bm{\mu}))$.

    Let ${T}_{\epsilon}$ be such that, for all $t \ge {T}_{\epsilon}$, under $\mathcal{E}(t)$, it holds that $\|\bm{\mu} - \bm{\mu}'\|_{\infty} \le \rho_\epsilon$ for all $\bm{\mu}' \in C_s$ and $s \ge h(t)$. From Lemma~\ref{lemma:ct-shrink} we are guaranteed that such ${T}_{\epsilon}$ exists.

    Then, it follows that, on $|x_s - \bar x| \le \epsilon$ for all $s \ge h(t)$ on the good event $\mathcal{E}(t)$. Indeed, (i) $x_s \le \bar x$ since $\bm{\mu} \in C_s$ on the good event, (ii) $x_s \ge \bar x - \epsilon$ since $\Xcal_F(\bm{\mu}') \in \Bcal_{\epsilon}(\Xcal_F(\bm{\mu}))$ and $\bar x \in \argmin_{x \in \Xcal_F(\bm{\mu})} x$.
\end{proof}

As a minor remark, we observe that \Cref{lemma:conv-real} is actually slightly stronger than \Cref{def:conv-answer} as we know exactly the answer toward which we are converging.

\subsubsection{Sticky Sequence Track-and-Stop for the case in which $|\Xcal_F(\mu)|$ is finite}

We show that \Cref{def:conv-answer} holds when $|\Xcal_F(\bm\mu)|$ is finite and $x_t \in \argmin_{x \in \Xcal_t} \| x - x_{t-1} \|_{\infty}$. As above, we assume that $x_t$ is attained within $\Xcal_t$. Again, if this is not the case, the inf will be for sure attained on the closure of $\Xcal_t$ and we can simply pick any $x_t \in \Xcal_t$ arbitrary close to $\argmin_{x \in \cl(\Xcal_t)} x$. The proof follows by identical arguments.

\begin{lemma}\label{lemma:conv-finite}
    Let $\{ x_t \}_t$ be such that $x_t \in \argmin_{x \in \Xcal_t} \| x - x_{t-1} \|_{\infty}$ for all $t \in \mathbb{N}$. Suppose that $|\Xcal_F(\bm\mu)| \le M$ for all $\bm\mu \in \Theta^K$, then \Cref{def:conv-answer} holds, \ie
    for all $\epsilon > 0$, there exists $T_{\epsilon}$, such that, for all $t \ge T_{\epsilon}$, under $\mathcal{E}(t)$, there exists $\bar x \in \Xcal_F(\bm\mu)$, such  that $\|\bar x - x_s \|_{\infty} \le \epsilon$ for all $s \ge h(t)$. 
\end{lemma}
\begin{proof}
    Let $\rho$ be such that $\| x_1 - x_2 \|_{\infty} > \rho$ for all $x_1, x_2 \in \Xcal_{F}(\bm{\mu})$, $x_1 \ne x_2$. Then, since $\Xcal_F(\bm{\mu})$ is upper hemicontinuous (\Cref{lemma:continuity}), there exists $\eta_{\rho} > 0$ such that, for all $\bm{\mu}': \|\bm{\mu} - \bm{\mu}'\|_{\infty} \le \eta_{\rho} \implies \Xcal_F(\bm{\mu}') \subset \Bcal_{\rho}(\Xcal_F(\bm{\mu}))$. Since $|\Xcal_F(\bm{\mu})|<M$, it follows that:
    \begin{align*}
        \bigcup_{\bm{\mu}:\| \bm{\mu} - \bm{\mu}' \|_{\infty} \le \eta_{\rho}} \Xcal_F(\bm{\mu}') \subseteq \bigcup_{x \in \Xcal_F(\bm{\mu})} \Bcal_{\rho}(x),
    \end{align*}
    and, moreover, $\Bcal_{\rho}(x_1) \cap \Bcal_{\rho}(x_2) = \emptyset$ for all $x_1, x_2 \in \Xcal_F(\bm{\mu})$, $x_1 \ne x_2$.

    Let $t \ge \widetilde T_{\eta_\rho}$, where $\widetilde T_{\eta_\rho}$ is such that, for all $t \ge \widetilde T_{\eta_\rho}$, on $\mathcal{E}(t)$, $\|\bm{\mu} - \bm{\mu}' \|_{\infty} \le \eta_\rho$ for all $\bm{\mu}' \in C_s$ and all $s \ge h(t)$. Such $\widetilde T_{\eta_{\rho}}$ is guaranteed to exists due to \Cref{lemma:ct-shrink}. Then, on $\mathcal{E}(t)$, $\|\bm{\mu} - \bm{\mu}' \|_{\infty} \le \eta_\rho$ for all $\bm{\mu}' \in C_s$ and all $s \ge h(t)$.

    Now, we observe that, for $t \ge \widetilde T_{\eta_{\rho}}$ and $s \ge h(t)$, on $\mathcal{E}(t)$, $\|x_s - \bar x\|_{\infty} \le {\rho}$ for some $\bar x \in \Xcal_F(\bm{\mu})$, and, furthermore, $\|x_s - x\|_{\infty} > \rho$ for all $x \in \Xcal_{F}(\bm{\mu}) \setminus \bar{x}$. 
    This is due to the fact that $\bm{\mu} \in C_s$ for all $s \ge h(t)$ (and, consequently, $\Xcal_F(\bm{\mu}) \in \Xcal_s$),  $\| x_1 - x_2 \|_{\infty} > \rho$ for all $x_1, x_2 \in \Xcal_{F}(\bm{\mu})$, and thanks to the fact that the selection rule of $x_s$ is such that it always select the feasible solution which is closest to the previous point.

    Thus, \Cref{def:conv-answer} directly follows by considering any $t \ge T_{\epsilon} \coloneqq \max\{\widetilde T_{\epsilon}, \widetilde T_{\rho} \}$. Indeed, by the reasoning above, we have that there exists $\bar{x} \in \Xcal_F(\bm\mu)$ such that $\|x_s - \bar{x} \| \le \min\{ \rho, \epsilon \} \le \epsilon$ for all $s \ge h(t)$ under the good event $\mathcal{E}(t)$.
    
\end{proof}

\subsubsection{Sticky Sequence Track and Stop with Adaptive Discretization}\label{app:adaptive}

\begin{algorithm}[t]
	\caption{A general procedure for selecting answers.}
    \label{alg:general}
	\begin{algorithmic}[1]
		\REQUIRE{Radius function $\rho: \mathbb{N} \rightarrow \mathbb{R}$, answer space $\mathcal{X}$}
        \STATE{\textbf{Initialization:} Pick any $\bar{x}_1 \in \mathcal{P}_{1}: \mathcal{B}_{\rho_1}(\bar{x}_1) \cap \Xcal_1 \ne \emptyset$. Set $x_1 \in \mathcal{B}_{\rho_1}(\bar{x}_1) \cap \Xcal_1$. Set $\mathcal{H}_1 = \{ (\bar{x}_1, \rho_1)\}$}
        \STATE{\textbf{Selection Rule at step $t$}:}
        \STATE{Let $\mathcal{S}_t = \{ ({x}, {\rho}) \in \mathcal{H}_{t-1}: \forall (\bar{x}_s, \rho_s) \in \mathcal{H}_{t-1},\, \rho_s > {\rho}\, \textnormal{ and } \mathcal{B}_{\rho_s}(\bar{x}_s) \cap \Xcal_t \ne \emptyset \}$ }\label{alg:selection rule1}
        \STATE{%
        Let $\bar{s}$ be the time corresponding to the element $(\bar x_s,\rho_s)\in \mathcal{S}_t$ with the smallest radius $\rho_s$\label{alg:selection rule2}}
        \STATE{Let $\overline{\Hcal}_t = \emptyset$}
        \FOR{$s=\bar{s}+1,\dots, t$}
        \STATE{Pick any $\bar{x}_s \in \mathcal{P}_s$ such that $\mathcal{B}_{\rho_s}(\bar{x}_s) \cap \mathcal{B}_{{\rho}_{s-1}}(\bar{x}_{s-1}) \cap \Xcal_t \ne \emptyset$ \label{alg:selection rule3}}
        \STATE{$\overline{\mathcal{H}}_t = \overline{\mathcal{H}}_t \cup (\bar{x}_s, \rho_s)$\label{alg:selection rule4}}
        \ENDFOR
        \STATE{Pick any $x_t \in \mathcal{B}_{\rho_t}(\bar{x}_t) \cap \Xcal_t$\label{alg:selection rule5}}
        \STATE{Update History $\mathcal{H}_t = \{ (x,\rho) \in \mathcal{H}_{t-1}: \rho \ge \rho_{\bar{s}}\} \cup \overline{\mathcal{H}}_t$\label{alg:selection rule6}}
	\end{algorithmic}
\end{algorithm}

We now show that there exists a general algorithm for ensuring \Cref{def:conv-answer} in arbitrary compact spaces. The pseudocode of the selection rule for the answer $x_t$ can be found in Algorithm~\ref{alg:general}. Before detailing the algorithm we introduce some notation.

\paragraph{Preliminary definitions}
We consider a radius function $\rho: \mathbb{N} \to \mathbb{R}_{>0}$ such that $\rho(t+1) < \rho(t)$ and $\lim_{t \to \infty} \rho(t) = 0$. Specifically, we choose $\rho(t) = 2^{-t}$. In the following, with some abuse of notation, we write $\rho(t) = \rho_t$. Note that, for any $\rho_t$, the corresponding time $t$ is uniquely identified by $\rho_t$ due to the fact that $\rho(t)$ is invertible.  

For the answer set $\mathcal{X}$ and any $t \in \mathbb{N}$, we write write $\mathcal{P}_t = \{ x_i \}_{i=1}^{n_{\rho_t}}$ (or, with some abuse of notation, $\mathcal{P}_{\rho_t}$) to denote the centers of a cover of $\Xcal$ that uses balls of radius $\rho_t$ centered in points $x_i \in \mathcal{P}_t$.

\paragraph{Explanation of the algorithm}
\Cref{alg:general} works by combining the progressive discretization of $\Xcal$ together with a ``history mechanism''.

In the first turn, the algorithm selects any point $\bar{x}_1 \in \mathcal{P}_1$ such that $\mathcal{B}_{\rho_1}(\bar{x}_1) \cap \Xcal_1\neq\emptyset$ and it picks any answer $x_1 \in \mathcal{B}_{\rho_1}(\bar{x}_1) \cap \Xcal_1$. Then, it initializes a ``history'' $\Hcal_1=\{(\bar x_1,\rho_1)\}$.

Assume now that a history $\Hcal_{t-1}$ of $t-1$ tuples $\{(\bar x_s,\rho_s)\}_{s=1}^{t-1}$ is given to the algorithm (we will explain shortly how $\Hcal_t$ is defined and updated by the algorithm). The algorithm decides the answer $x_t$ to play as follows.
Among all elements of $\Hcal_{t-1}$, the algorithm first selects an index $\bar s$, which is the one with the smallest radius (\ie the greatest index $\bar s\le t-1$) that guarantees that the intersection of $\Bcal_{\rho_{s}}(\bar x_{s})$ and $\Xcal_t$ is non-empty for all $s \le \bar s$  (\Cref{alg:selection rule1} and \Cref{alg:selection rule2}). 
Intuitively, the region $\Bcal_{\rho_{\bar s}}(\bar x_{\bar s})$ can be seen as an anchoring mechanism that constrains the search for a good answer in $\Xcal_F(\bm\mu)$ towards previously selected points (for further explanation, see the remark of $\bar{s}$ below).
When such $(\bar x_{\bar s},\rho_{\bar s})$ is selected, then the tuples $(x_s,\rho_s)$ with $s>\bar s$ are discarded from the history set, and $\mathcal{H}_t$ is ``repopulated'' by picking some $(\bar x_s, \rho_s)$ such that $\bar x_s$ is in the centers of the cover at time $s$ and such that $\mathcal{B}_{\rho_s}(\bar{x}_s) \cap \mathcal{B}_{{\rho}_{s-1}}(\bar{x}_{s-1}) \cap \Xcal_t \ne \emptyset$ (\Cref{alg:selection rule3}). 
In \Cref{alg:general} the ``repopulation'' of $\Hcal_t$ is formalized trough $\overline{\Hcal}_t$ (\Cref{alg:selection rule4} and \Cref{alg:selection rule6}).

Finally, \Cref{alg:general} simply selects any $x_t\in\mathcal{B}_{\rho_t}(\bar{x}_t) \cap \Xcal_t$ (\Cref{alg:selection rule5}).

\paragraph{Remark on $\bar s$}
We observe that, for all $s > \bar{s}$, there exists an element in position $s' \in \{s+1, s\}$ such that $\mathcal{B}_{\rho_{s'}}(\bar{x}_{s'}) \cap \Xcal_t = \emptyset$. We recall that:
\begin{itemize}
    \item $\Xcal_t$ represents the set of candidate answers that correspond to models within a confidence region around $\hat{\bm\mu}_t$;
    \item \Cref{alg:general} keeps in $\Hcal_t$ only elements of $\Hcal_{t-1}$ whose index is less than $\bar{s}$, \ie for all $s \le \bar{s}$, it holds that $\Bcal_{\rho_s}(\bar{x}_s) \cap \Xcal_t \ne \emptyset$.
\end{itemize}
These two facts can be interpreted as a \emph{backtracking} operation that is needed to guide the search of the algorithm toward some $\bar{x} \in \Xcal_F(\bm\mu)$. Indeed, for all $s > \bar{s}$, there is a point within the sequence $\{\bar s, \dots, s \}$ whose ball does not intersect the set of candidate answers $\Xcal_t$ (which, on the good event, contains $\Xcal_F(\bm\mu)$).

\paragraph{Remark on notation}
In the rest of this section, for the sake of clarity, we will use the following convention:
\begin{itemize}
    \item $\bar{x} \in \Xcal$ denotes the answer in $\Xcal_F(\bm\mu)$ towards which \Cref{alg:general} is converging (or, whenever needed, answers in $\Xcal_F(\bm\mu)$);
    \item $x_t$ denotes the answer selected by \Cref{alg:general} at step $t$;
    \item $\bar{x}_{s}^{(t)}$ denotes the center of the cover of radius $\rho_s$ (\ie $\bar{x}_s \in \mathcal{P}_s$) within the history set at step $t$, \ie $(\bar{x}_s^{(t)}, \rho_s) \in \Hcal_t$. Note that here, we are expanding the notation introduced above since the elements within $\Hcal_t$ can change throughout the execution of the algorithm. This is to avoid potential ambiguities.
\end{itemize}

Given this convention, we observe that for all $t \in \Naturals$, the selection rule of \Cref{alg:general} can be rewritten as $x_t \in \Bcal_{\rho_t}(\bar{x}_t^{(t)})$ for some $\bar{x}_t^{(t)} \in \mathcal{P}_{\rho_t}$. Similarly, the history set $\Hcal_t$ can be rewritten as $\Hcal_t = \{ (x_s^{(t)}, \rho_s) \}_{s=1}^t$. Observe that, due to the backtracking operation highlighted above, in principle it can happen that $x_s^{(t)} \ne x_s^{(t+1)}$.

\paragraph{Proof of convergence}

In the following, we prove that \Cref{alg:general} generates a converging sequence.
We first give a proof outline and then we dive into the formal arguments.

The main idea is showing that for any $n \in \mathbb{N}$, and for sufficiently large $t$, under the good event $\mathcal{E}(t)$, there exists an element $(\bar x_n, \rho_n)$ that remains in $\Hcal_s$ for all $s \ge h(t)$, \ie $\bar{x}_n = \bar{x}_n^{(s)}$ for all $s \ge h(t)$. This first result (which is formally stated in \Cref{lemma:new}) intuitively says that, when enough information has been collected, then the algorithm is able to fix a region of arbitrary radius within its history set $\Hcal_s$. As \Cref{lemma:conv-general} then shows, this is enough to ensure that the algorithm is converging to some $\bar{x} \in \Xcal_F(\bm\mu)$. Indeed, we will show that the aforementioned $\bar{x}_n$ is close to some $\bar{x} \in \Xcal_F(\bm\mu)$, and, by the design of the algorithm (\ie \Cref{alg:selection rule4}), the answer selected at step $s$ (\ie $x_s$) will be close to $\bar{x}_n$ for all $s \ge h(t)$. By a triangular inequality argument, this implies that $x_s$ remains close to some $\bar{x} \in \Xcal_F(\bm\mu)$, thus leading to the desired convergence property. 

Now, before moving to \Cref{lemma:new}, we first prove a technical result which will be used within the proof of \Cref{lemma:new}.

\begin{lemma}\label{lemma:conv-general-value}
    Consider $\bm\mu \in \mathcal{M}$, $n \in \mathbb{N}$ and $\rho = \rho_n$. Suppose that there exists ${x} \in \mathcal{P}_{\rho}$ such that $ \Bcal_{\rho}({x}) \cap \Xcal_F(\bm\mu) = \emptyset$. Then, it holds that:
    \begin{align*}
        \eta_{\rho} = \min_{\bar{x} \in \Xcal_F(\bm\mu)} \min_{\substack{{x} \in \mathcal{P}_\rho: \\ \mathcal{B}_{\rho}({x}) \cap \Xcal_F(\bm\mu) = \emptyset}} \min_{x' \in \mathcal{B}_{\rho}({x})} \| x' - \bar{x} \|_{\infty} > 0.
    \end{align*}
    where $\eta_\rho$ is well defined and $\eta_\rho>0$.
\end{lemma}
\begin{proof}
    Consider
    \[
    \eta_\rho\coloneqq \inf_{\bar x \in \Xcal_F(\bm\mu)} \inf_{\substack{{x} \in \mathcal{P}_\rho: \\ \mathcal{B}_{\rho}({x}) \cap \Xcal_F(\bm\mu) = \emptyset}} \inf_{x' \in \mathcal{B}_{\rho}({x})} \| x' - \bar x \|_{\infty}.
    \]
    Then, fix any ${x} \in \mathcal{P}_\rho$ and $\bar x \in \Xcal_F(\bm\mu)$ and let
    \(
        \phi(\bar x,x)=\inf_{x' \in \mathcal{B}_{\rho}({x})} \| x'-\bar x\|_{\infty},
    \)
    where the $\inf$ can be replaced by $\min$ (indeed, $\Bcal_{\rho}(x)$ is compact and the inner infinite-norm function is continuous).

    Now, let $\bar x \in \Xcal_F(\bm\mu)$ and consider:
    \begin{align*}
        \inf_{\substack{{x} \in \mathcal{P}_\rho: \\ \mathcal{B}_{\rho}({x}) \cap \Xcal_F(\bm\mu) = \emptyset}} \phi(\bar x, x).
    \end{align*}
    We observe that, by assumption, the  set over which we are optimizing is non-empty. Furthermore, it holds that:
    \begin{align*}
        \inf_{\substack{{x} \in \mathcal{P}_\rho: \\ \mathcal{B}_{\rho}({x}) \cap \Xcal_F(\bm\mu) = \emptyset}} \phi(\bar x, x) = \min_{\substack{{x} \in \mathcal{P}_\rho: \\ \mathcal{B}_{\rho}({x}) \cap \Xcal_F(\bm\mu) = \emptyset}} \phi(\bar x, x) \coloneqq g(\bar x).
    \end{align*}
    Indeed, the cardinality of the set ${x} \in \mathcal{P}_\rho: \mathcal{B}_{\rho}({x}) \cap \Xcal_F(\bm\mu) = \emptyset$  is finite since $\mathcal{P}_\rho$ is finite. %
    Furthermore, $g(\bar x)$ is continuous in $\bar x$ since it is a minimum over a finite set of continuous functions.

    At this point, we note that the outer $\inf$ can be replaced by a $\min$ since that objective function is continuous in $\bar x$ and $\Xcal_F(\bm\mu)$ is compact. This shows that $\eta_\rho$ is well defined.

    It remains to prove that the value of the optimization problem is strictly greater than $0$. Consider a triplet $(\bar x, {x}, x')$ that attains all the minimums of the optimization problem. Then, we need to prove that $\|\bar x-x'\|_{\infty} > 0$, i.e., that $\bar x \ne x'$. However, $\bar x \in \Xcal_F(\bm\mu)$ and $x' \notin\Xcal_F(\bm\mu)$ (since $x' \in \mathcal{B}_{\rho}({x})$ and $\mathcal{B}_{\rho}({x}) \cap \Xcal_F(\bm\mu) = \emptyset$), thus concluding the proof.
\end{proof}

We are now ready to show that, under the good event, the algorithm is able to identify good regions in which it will conduct the search.

\begin{lemma}\label{lemma:new}
    For all $n \in \Naturals$, there exists $T_n \in \Naturals$ such that, for all $t \ge T_n$, under $\mathcal{E}(t)$, there exists $\bar x_n \in \mathcal{P}_n$ and:
    \begin{align}\label{eq:lemma-new-1}
        (\bar x_n, \rho_n) \in \Hcal_s, \quad \forall s \ge h(t),
    \end{align}
    \ie for all $s \ge h(t)$, $\bar x_n = \bar x_n^{(s)}$.
    Furthermore, under $\mathcal{E}(t)$,
    there exists $\bar{x} \in \Xcal$ such that
    \begin{align}\label{eq:lemma-new-2}
    \bar{x} \in \Xcal_F(\bm\mu) \cap \Bcal_{\rho_n}(\bar x_n).
    \end{align}
\end{lemma}
\begin{proof}

    We start with some preliminary definitions. 
    
    Let $n \in \Naturals$, and let us define the following set:
    \begin{align}\label{eq:bad-set}
        \Naturals(n) = \{ j \in \Naturals: j \le n \textup{ and } \exists {x} \in \mathcal{P}_{\rho_j}: \Bcal_{\rho_j}({x}) \cap \Xcal_F(\bm\mu) = \emptyset \}.
    \end{align}
    The set $\Naturals(n)$ represents the subset of $[n]$ for which we can apply \Cref{lemma:conv-general-value}. Furthermore, let us define $r_{n}$ as follows:
    \begin{align*}
        r_n = \frac{1}{2}\min_{j \in \Naturals(n)} \eta_{\rho_j},
    \end{align*}
    where $\eta_{\rho_j}$ is as in \Cref{lemma:conv-general-value}. 
    
    Next, we recall that, by \Cref{lemma:ct-shrink} and the upper hemicontinuity of $\Xcal_F$ (\ie \Cref{lemma:continuity}), for all $\Delta > 0$, there exists $\widetilde T_{\Delta}$ such that, for all $t \ge \widetilde T_{\Delta}$, under $\mathcal{E}(t)$, $\Xcal_s \subseteq \Bcal_{\Delta}(\Xcal_F(\bm\mu))$ for all $s \ge h(t)$.

    In the following, we will show that picking $\Delta = r_{n}$ leads to the desired result.
    Specifically, we will consider $T_n = \max\{ \widetilde T_{r_n}, \bar{T}_n \}$), where $\bar{T}_n$ is such that $\rho_{h(\bar{T}_n)} \le \rho_n$. Here, the requirement that $t \ge \bar{T}_n$ is a technicality that ensures that elements of the form $(\cdot, \rho_n)$ are already within the history set $\Hcal_{h(t)}$. Indeed, at step $h(t)$, the history set $\Hcal_{h(t)}$ 
    only contains elements of the form $(\cdot, \rho)$ for $\rho \ge \rho_{h(t)}$. Intuitively, to ensure that $(\bar{x}_n^{(h(t)}, \rho_n) \in \Hcal_{h(t)}$ for some $\bar{x}_n^{(h(t)} \in \mathcal{P}_{h(t)}$ we need $t \ge \bar{T}_n$. 
    We now proceed with the crucial part of the proof, that is exploiting $t \ge \widetilde T_{r_n}$.

    Observe that, to prove \Cref{eq:lemma-new-1}, we need to show that there exists an element $(\bar x_{n}, \rho_n)$ that will not be removed from $\Hcal_s$ in any $s \ge h(t)$, \ie $\exists \bar{x}_n: \bar{x}_n = \bar{x}_n^{(s)}$ for all $s \ge h(t)$.
    As discussed above, for $t \ge \bar{T}_n$, there is an element of the form $(\bar{x}_n^{(h(t))}, \rho_n)$ within $\Hcal_{h(t)}$. In the following, we prove that such an element will not be eliminated from $\Hcal_s$ as the execution proceeds.
    
    By definition of \Cref{alg:general}, an element $(\bar x_n^{(s)}, \rho_n)$ can fail to belong to $\Hcal_{s+1}$ if and only if there exists $(\bar{x}^{(s)}_j, \rho_j) \in \Hcal_s$ such that $j \le n$ and $\Bcal_{\rho_j}(\bar x_j^{(s)}) \cap \Xcal_{s+1} = \emptyset$. 

    We proceed by contradiction. Suppose that there exists a $(\bar x^{(s)}_j, \rho_j) \in \Hcal_s$ for which $\Bcal_{\rho_j}(\bar x_j^{(s)}) \cap \Xcal_{s+1} = \emptyset$ for some $j \le n$.
    In the following, we will refer to the index $j$ as the index of the element that \emph{triggers} the elimination of $(\bar x^{(s)}_n, \rho_n)$ from $\Hcal_{s+1}$. Now, we proceed by cases.

    \paragraph{Case one: $j \notin \Naturals(n)$}
    If $j \notin \Naturals(n)$, then, we know by \Cref{eq:bad-set} that $\mathcal{B}_{\rho_j}(\bar x_j^{(s)}) \cap \Xcal_F(\bm\mu) \ne \emptyset$. Hence, since on the good event $\Xcal_F(\bm\mu) \subseteq \Xcal_s$ 
    for all $s \ge h(t)$,\footnote{Recall that $\bm\mu \in C_s$ by definition of $\mathcal{E}(t)$.} we also have that $\mathcal{B}_{\rho_j}(\bar x_j^{(s)}) \cap \Xcal_s \ne \emptyset$, and, as a consequence, the elimination condition of $(\bar x^{(s)}_n, \rho_n)$ from $\Hcal_{s+1}$ cannot be triggered by the element in position $j$, thus leading to the contradiction.

    \paragraph{Case two: $j \in \Naturals(n)$}
    If $j \in \Naturals(n)$, instead, we prove that for $t \ge T_{n}$, $\Bcal_{\rho_j}(\bar x_j^{(s)}) \cap \Xcal_F(\bm\mu) \ne \emptyset$ and (as we have shown in the previous case), this in turn implies that the elimination condition of $(\bar x^{(s)}_n, \rho_n)$ from $\Hcal_{s+1}$  cannot be triggered by the element in position $j$. To this end, the following argument shows that, for $t \ge \widetilde T_{r_n}$, all the elements $x \in \mathcal{P}_{\rho_j}$ such that $\mathcal{B}_{\rho_j}(x) \cap \Xcal_F(\bm\mu) = \emptyset$ cannot satisfy $\mathcal{B}_{\rho_j}(x) \cap \Xcal_s \ne \emptyset$ at any $s \ge h(t)$, and therefore, they cannot belong to $\Hcal_s$. Specifically, suppose that there exists $x \in \mathcal{P}_{\rho_j}$ such that the following holds for $s \ge h(t)$:
     \begin{align}
         & \mathcal{B}_{\rho_j}(x) \cap \Xcal_F(\bm\mu) = \emptyset \label{eq:cont-1} \\
         & \mathcal{B}_{\rho_j}(x) \cap \Xcal_s \ne \emptyset. \label{eq:cont-2}
     \end{align}
     Then, under $\mathcal{E}(t)$, for all $s \ge h(t)$, and any $y \in \mathcal{B}_{{\rho_j}}({x})\cap \mathcal{X}_s$ (such $y$ exists due to \Cref{eq:cont-2}), there exists $\bar{x} \in \Xcal_F(\bm\mu)$ such that $\|y - \bar{x} \|_{\infty}  \le r_n$ (by definition of $\widetilde T_{r_n}$ and $t \ge T_{r_n}$). However:
    \begin{align*}
        \|y - \bar{x} \|_{\infty} & \le r_n \\ 
        & < \eta_{\rho_j} \tag{$r_n = \frac{1}{2}\min_{i \in \Naturals(n)} \eta_{\rho_i}$ and $j \in \Naturals(n)$} \\
        & \le \min_{\substack{{x}' \in \mathcal{P}_{\rho_j}: \\ \mathcal{B}_{\rho_j}({x}') \cap \Xcal_F(\bm\mu) = \emptyset}} \min_{x'' \in \mathcal{B}_{\rho_j}({x}')} \| x'' - \bar{x} \|_{\infty} \tag{Def.~of $\eta_{\rho_j}$} \\
        & \le \min_{x'' \in \mathcal{B}_{\rho_j}({x})} \| x'' - \bar{x} \|_{\infty} \tag{$x \in \mathcal{P}_{\rho_j}: \Bcal_{\rho_j}(x) \cap \Xcal_F(\bm\mu) = \emptyset$, \ie \Cref{eq:cont-1}}  \\
        & \le \|y-\bar{x} \|_{\infty} \tag{$y \in \Bcal_{\rho_j}(x)$}.
    \end{align*}
    Thus, we have shown that $\|y - \bar{x} \|_{\infty} < \|y - \bar{x} \|_{\infty}$ which leads to a contradiction. This concludes the proof of \Cref{eq:lemma-new-1} since we have shown that for all $s \ge h(t)$ all the conditions that can trigger the elimination of $(\bar{x}^{(s)}_n, \rho_n)$ from $\Hcal_{s+1}$ cannot be triggered.   

    Now, concerning \Cref{eq:lemma-new-2}, we observe that the proof of \Cref{eq:lemma-new-1} already used in its argument the existence, for all $j \le n$, of $\Xcal_F(\bm\mu) \cap \Bcal_{\rho_{j}}(\bar{x}_j^{(s)}) \ne \emptyset$. Thus it holds also for $j = n$ and $\bar{x}_n^{(s)} = \bar{x}_n$ for all $s \ge h(t)$. 
\end{proof}

We are now ready to prove that \Cref{def:conv-answer} is satisfied for the sequence $\{x_t\}_t$ generated by \Cref{alg:general}.

\begin{lemma}\label{lemma:conv-general}
    Consider that $\{ x_t \}$ is given by \Cref{alg:general}. Then, \Cref{def:conv-answer} holds, i.e., for all $\epsilon > 0$, there exists $T_{\epsilon} \in \mathbb{N}$ such that, for all $t \ge T_{\epsilon}$, on $\mathcal{E}(t)$, there exists $\bar{x} \in \Xcal_F(\bm\mu)$ such that $\|\bar{x} - x_s\|_{\infty} \le \epsilon$ for all $s \ge h(t)$.
\end{lemma}
\begin{proof}
    From \Cref{lemma:new} we know that, for all $n \in \Naturals$, there exists $\widetilde T_n$ such that, for all $t \ge \widetilde T_n$, under $\mathcal{E}(t)$, there exists $\bar x_n \in \mathcal{P}_n$ and $(\bar x_n, \rho_n) \in \Hcal_s, \forall s \ge h(t)$.  Let $\epsilon > 0$ and consider $T_{\epsilon} = \widetilde T_{\bar{n}}$ where $\bar{n}$ is the smallest integer that verifies $\rho_{\bar n - 1} \le \epsilon / 4$.

    It follows that for all $t \ge T_{\epsilon}$ the following properties hold under $\mathcal{E}_t$:
    \begin{align}
        & \exists \bar x_{\bar n} \in \mathcal{P}_{\bar{n}}: (\bar x_{\bar n}, \rho_{\bar{n}}) \in \Hcal_s~ \forall s \ge h(t), \textup{ and } \rho_{\bar{n}} \le \epsilon/4.  \label{eq:prop-1}  
    \end{align}
    Observe that, as a direct consequence of \Cref{eq:prop-1}, we have that:
    \begin{align}\label{eq:prop-4}
        \rho_s \le \rho_{\bar n} \quad \forall s \ge h(t).
    \end{align}
    Indeed, from \Cref{lemma:new}, we have that $(\bar x_{\bar n}, \rho_{\bar{n}}) \in \Hcal_s$ for all $s \ge h(t)$. Then, it follows by definition of \Cref{alg:general} that $\rho_{\bar n}$ is related to some step $\bar n$ which is at most $h(t)$ (since that element need to be in the history set already at step $h(t)$). In other words, $h(t) \ge \bar n$.
    
    Furthermore, from \Cref{lemma:new} and $\bar x_{\bar n}$ as in \Cref{eq:prop-1}, we have that:
    \begin{align}
        \exists \bar{x} \in \Xcal_F(\bm\mu)\cap\bar{x} \in \mathcal{B}_{\rho_{\bar{n}}}(\bar x_{\bar n}).\label{eq:prop-3}
    \end{align}

    Now, by definition of \Cref{alg:general}, $x_s \in \Bcal_{\rho_s}(\bar{x}_s^{(s)})$ for all $s \ge 1$ for some $\bar{x}_s^{(s)} \in \mathcal{P}_{\rho_s}$. Let $\bar x_{\bar n}$ be as in \Cref{eq:prop-1} and $\bar{x}$ as in \Cref{eq:prop-3}. In the following, we will prove that under $\mathcal{E}(t)$, $\|x_s - \bar{x} \|_{\infty} \le \epsilon$ for all $s \ge h(t)$, which is exactly \Cref{def:conv-answer}. 
    Specifically, we have that:
    \begin{align*}
        \|x_s - \bar{x} \|_{\infty} & \le \|x_s - \bar x_{\bar n} \|_{\infty} + \|\bar x_{\bar n} - \bar{x} \|_{\infty} \\
        & \le \|x_s - \bar x_{\bar n} \|_{\infty} + \rho_{\bar{n}} \tag{\Cref{eq:prop-3}} \\
        & \le \|x_s - \bar x_{\bar n} \|_{\infty} + \frac{\epsilon}{4} \tag{Def. of $\bar{n}$} \\
        & \le \|x_s - \bar{x}_s^{(s)} \|_{\infty} + \| \bar{x}_s^{(s)} - \bar x_{\bar n } \|_{\infty} + \frac{\epsilon}{4} \\
        & \le \rho_{s} + \|\bar{x}_s^{(s)} - \bar x_{\bar n} \|_{\infty} + \frac{\epsilon}{4} \tag{Def. of $x_s$} \\
        & \le \rho_{\bar{n}} + \|\bar{x}_s^{(s)} - \bar x_{\bar n} \|_{\infty} + \frac{\epsilon}{4}\tag{\Cref{eq:prop-4}} \\
        & \le \|\bar{x}_s^{(s)} - \bar x_{\bar n} \|_{\infty} + \frac{\epsilon}{2} \tag{Def. of $\bar{n}$}
    \end{align*}
    At this point, it remains to upper bound $\|\bar{x}_s^{(s)} - \bar x_{\bar n} \|_{\infty}$, which we analyze with a telescoping argument. Recall that $\bar n \le s$ for all $s \ge h(t)$. Then,
    \begin{align*}
        \|\bar{x}_s^{(s)} - \bar x_{\bar n} \|_{\infty} \le \sum_{j = \bar{n}}^s \|\bar{x}_{j}^{(s)} - \bar{x}_{j+1}^{(s)} \|_{\infty},
    \end{align*}
    where we introduced all the elements in the history $\Hcal_s$ from step $\bar{n}$ to $s$ and we have used that, due to \Cref{lemma:new}, $\bar{x}_{\bar n} = \bar{x}_{\bar n}^{(s)}$ for all $s \ge h(t)$.
    Then, by construction (\ie \Cref{alg:selection rule3} in \Cref{alg:general}), we have that $\mathcal{B}_{\rho_j}(\bar{x}_j^{(s)}) \cap \mathcal{B}_{\rho_{j+1}}(\bar{x}_{j+1}^{(s)}) \ne \emptyset$. Thus, $ \|\bar{x}_{j}^{(s)} - \bar{x}_{j+1}^{(s)} \|_{\infty} \le \rho_{j} + \rho_{j+1} \le 2 \rho_j$. It then follows that:
    \begin{align*}
        \|\bar{x}_s^{(s)} - \bar x_{\bar n} \|_{\infty} \le 2 \sum_{j = \bar{n}}^s \rho_j \le 2 \sum_{j=\bar{n}}^{+\infty} \frac{1}{2^{j}} \le 2 \frac{1}{2^{\bar{n}-1}} = 2 \rho_{\bar{n}-1} \le \frac{\epsilon}{2}.
    \end{align*}
    Combining these results, we have obtained that $\| x_s - \bar{x} \|_{\infty} \le \epsilon$ under $\mathcal{E}(t)$ for sufficiently large $t$. This concludes the proof.
    \end{proof}

\subsection{What happens when $\{x_t\}$ is not converging}\label{app:not-conv}

In this section, we discuss what happens when the underlying sequence is not a converging one. Specifically, what kind of theoretical guarantees can we obtain?
Note that answering this question also provides the theoretical guarantees of TaS and Sticky-TaS whenever they fail to generate a converging sequence. Indeed, as we discussed in \Cref{sec:algo}, our framework actually generalizes both TaS and Sticky-TaS (according to the rule with which $x_t$ is selected). 

On a theoretical level, the main result of this section is \Cref{thm:not-conv}. Specifically, \Cref{thm:not-conv} shows that, when the sequence $x_t$ is not a converging one, we are only able to guarantee a results of the following form:
\begin{align}\label{eq:app-not-conv-temp}
\limsup_{\delta \to 0} \frac{\E_{\bm\mu}[\tau_\delta]}{\log(1/\delta)} \le \max_{\bm\omega \in \textup{conv}(\bm\omega^{\star}(\bm\mu))} \min_{x \in \Xcal^{\star}(\bm\mu)} \frac{1}{D(\bm\mu, \bm\omega, \neg x)}.
\end{align}

The rest of this section is structured as follows. First, in \Cref{app-subsec-not-conv}, we provide the formal statement and the proof of \Cref{eq:app-not-conv-temp}. The formal result is given in \Cref{thm:not-conv}. Then, in \Cref{app:not-conv-example}, we provide, as curiosity, a theoretical example of an application for which we can quantify the gap between the result of \Cref{theo:sticky-seq} and \Cref{thm:not-conv}.

\subsubsection{Theoretical Results}\label{app-subsec-not-conv}
To prove \Cref{thm:not-conv}, we first show a result that combines \Cref{lemma:continuity} and the good event.

\begin{lemma}\label{lemma:smoothness-on-good-event}
    For all $\kappa > 0$, there exists $T_{\kappa} \in \mathbb{N}$ such that, for all $t \ge T_{\kappa}$, on $\mathcal{E}(t)$, it holds that, for all $s \ge h(t)$, and all $\bm\mu' \in C_s$:
    \begin{align*}
    & \Big|\max_{x \in \Xcal_F(\bm\mu)} D(\bm\mu,\bm\omega, \neg x) - \max_{x \in \Xcal_F(\bm\mu')} D(\bm\mu',\bm\omega, \neg x) \Big| \le \kappa ,~\forall \bm\omega \in \Delta_K, \\ 
    &  \inf_{\bm\omega \in \bm\omega^*(\bm\mu)} \| \bm\omega - \bm\omega' \|_{\infty} \le \kappa,~\forall \bm\omega' \in \bm\omega^*(\bm\mu').
    \end{align*}
\end{lemma}
\begin{proof}
Due to \Cref{lemma:continuity} and \Cref{corollary:unif-cont}, we have that, for all $\kappa = \kappa(\bm\mu) > 0$, there exists $\beta_\kappa$ such that:
\begin{align*}
    & \|\bm\mu - \bm\mu' \|_{\infty} \le \beta_\kappa \implies \Big|\max_{x \in \Xcal_F(\bm\mu)} D(\bm\mu,\bm\omega, \neg x) - \max_{x \in \Xcal_F(\bm\mu')} D(\bm\mu',\bm\omega, \neg x) \Big| \le \kappa ,~\forall \bm\omega \in \Delta_K, \\ 
    & \|\bm\mu - \bm\mu' \|_{\infty} \le \beta_\kappa \implies \inf_{\bm\omega \in \bm\omega^*(\bm\mu)} \| \bm\omega - \bm\omega' \|_{\infty} \le \kappa,~\forall \bm\omega' \in \bm\omega^*(\bm\mu').
\end{align*}
Recall that, as a consequence of \Cref{lemma:ct-shrink}, it holds that, for all $\epsilon > 0$, there exists $T_{\epsilon}:\forall t \ge T_{\epsilon}$, then, on $\mathcal{E}(t)$, $\|\bm\mu - \bm\mu' \|_{\infty} \le \epsilon$ for all $s \ge h(t)$ and $\bm\mu' \in C_s$.

Picking $\epsilon = \beta_{\kappa}$ concluds the proof.
\end{proof}

The following lemma analyzes the behavior of the framework when the sequence of answers is not a converging one.

\begin{lemma}\label{lemma:stas-tas-subopt}
    Let $\kappa > 0$ and let $T_{\kappa} \in \mathbb{N}$ as in \Cref{lemma:smoothness-on-good-event}. Then, for $t \ge T_{\kappa}$, under $\mathcal{E}(t)$, it holds that:
    \begin{align*}
        \frac{1}{t} \max_{x \in \Xcal^{\star}(\hat{\bm\mu}(t))} & D(\hat{\bm\mu}(t), \bm{N}(t), \neg x)   \ge \frac{(t-\lceil h(t) \rceil)}{t} \min_{\bm\omega \in \textup{conv}(\bm\omega^*(\bm\mu))} \max_{x \in \Xcal^{\star}(\bm\mu)} D(\bm\mu, \bm\omega, \neg x) - \\ & \hspace{2.5cm}-\kappa(1+C_{\bm\mu})  - \frac{K(1+\sqrt{t})}{t}C_{\bm\mu},
    \end{align*}
    where $C_{\bm\mu} \ge 0$ is a problem dependent constant.
\end{lemma}
\begin{proof}
    Consider $\frac{1}{t} \max_{x \in \Xcal^\star(\bm{\hat{\mu}}(t))} D(\bm{\hat{\mu}}(t), \bm{N}(t), \neg x)$. By adding and subtracting $\frac{1}{t} \max_{x \in \Xcal^\star(\bm{\mu})} D(\bm{\mu}, \bm{N}(t), \neg x)$ we have that
    \begin{align*}
        \frac{1}{t} \max_{x \in \Xcal^\star(\bm{\hat{\mu}}(t))} D(\bm{\hat{\mu}}(t), \bm{N}(t), \neg x)  = \frac{1}{t} \max_{x \in \Xcal^\star(\bm{\mu})} D(\bm{\mu}, \bm{N}(t), \neg x) - h_1(t)
    \end{align*}
    where $h_1(t)$ is given by:
    \begin{align*}
        h_1(t) & \coloneqq \max_{x \in \Xcal^\star(\bm{\mu})} D\left(\bm{\mu}, \frac{\bm{N}(t)}{t}, \neg x\right) - \max_{x \in \Xcal^\star(\bm{\hat{\mu}}(t))} D\left(\bm{\hat{\mu}}(t), \frac{\bm{N}(t)}{t}, \neg x\right). 
    \end{align*}
    Now, using \Cref{lemma:smoothness-on-good-event} and noticing that $\bm\mu \in C_t$ on $\mathcal{E}(t)$, we have that $h_1(t)\le \kappa$, thus leading to:
    \begin{align*}
         \frac{1}{t} \max_{x \in \Xcal^\star(\bm{\hat{\mu}}(t))} D(\bm{\hat{\mu}}(t), \bm{N}(t), \neg x) \ge \frac{1}{t} \max_{x \in \Xcal^\star(\bm{\mu})} D(\bm{\mu}, \bm{N}(t), \neg x) - \kappa.
    \end{align*}
    
    We now continue by lower bounding $\frac{1}{t} \max_{x \in \Xcal^\star(\bm{\mu})} D(\bm{\mu}, \bm{N}(t), \neg x)$. Specifically, for all $x \in \Xcal^{\star}(\bm\mu)$, we have that:
    \begin{align*}
        \frac{1}{t} &\max_{x\in \Xcal^\star(\bm{\mu})} D(\bm{\mu}, \bm{N}(t), \neg x)  \ge \frac{1}{t} D(\bm{\mu}, \bm{N}(t), \neg x)) \\ 
        & = \frac{1}{t}\inf_{\bm\lambda\in\lnot x}\sum_{k\in K}N_k(t)d(\mu_k,\lambda_k)\\
        & = \frac{1}{t}\inf_{\bm\lambda\in\lnot x}\left\{\sum_{k\in K}\sum_{s\in[t]}\omega_k(s)d(\mu_k,\lambda_k)+\sum_{k\in K}\left(N_k(t)-\sum_{s\in[t]}\omega_k(s)\right)d(\mu_k,\lambda_k)\right\}\\
        &\ge \frac{1}{t}\inf_{\bm\lambda\in\lnot x}\sum_{k\in K}\sum_{s\in[t]}\omega_k(s)d(\mu_k,\lambda_k)+\frac{1}{t}\inf_{\bm\lambda\in\lnot x}\sum_{k\in K}\left(N_k(t)-\sum_{s\in[t]}\omega_k(s)\right)d(\mu_k,\lambda_k)\\
        & = \frac{1}{t} D \left( \bm{\mu}, \sum_{s=1}^t \bm{\omega}(s), \neg x \right) - h_2(t) 
    \end{align*}
    where $h_2(t)$ is given by:
    \begin{align*}
        h_2(t) & \coloneqq \frac{1}{t} \inf_{\bm{\lambda} \in \neg x} \sum_{k \in [K]} \left(\sum_{s=1}^t \omega_k(s) - N_k(t) \right) d(\mu_k, \lambda_k) \\ 
        & \le \frac{{K(1+\sqrt{t})}}{t} \sup_{x \in \Xcal^{\star}(\bm\mu)} \inf_{\bm{\lambda} \in \neg x} \sum_{k \in [K]} d(\mu_{k}, \lambda_k) \\ 
        &\le \frac{K(1+\sqrt{t})}{t} C_{\bm\mu},
    \end{align*}
    where in the last step we have used the fact that the exponential family is regular and bounded.

    At this point, focus on $\frac{1}{t} D \left( \bm{\mu}, \sum_{s=1}^t \bm{\omega}(s), \neg x \right)$ and let us analyze $\bm\omega(s)$. We know that, $\bm\omega(s) \in \bm\omega^*(\bm\mu'_s)$ for some $\bm\mu'_s \in C_s$ (indeed, $x_s$ is such that $x_s \in \Xcal_F(\bm\mu'_s)$). Now, from \Cref{lemma:smoothness-on-good-event}, we know that, on $\mathcal{E}(t)$, there exists $\{ \bar{\bm\omega}(s) \}_{s\ge h(t)}$ such that (i) $\|\bar{\bm\omega}(s) - \bm\omega(s) \|_{\infty} \le \kappa$ and (ii) $\bar{\bm\omega}(s) \in \bm\omega^*(\bm\mu)$. This is direct by taking $\bar{\bm\omega}(s) \in \argmin_{\bm\omega \in \bm\omega^*(\bm\mu)} \| \bm\omega(s) - \bm\omega \|_{\infty}$. Therefore, we obtain:
    \begin{align*}
        \frac{1}{t} D\left(\bm\mu, \sum_{s=1}^t \bm\omega(s), \neg x \right) & \ge \frac{1}{t} D\left(\bm\mu, \sum_{s=\lceil h(t) \rceil }^t \bm\omega(s), \neg x\right) \\ 
        & \ge \frac{t - \lceil h(t) \rceil}{t} D\left(\bm\mu, \frac{\sum_{s=\lceil h(t) \rceil}^t \bar{\bm\omega}(s)}{t-\lceil h(t) \rceil}, \neg x\right) + \frac{1}{t}  D\left(\bm\mu, \sum_{s=\lceil h(t) \rceil}^t \bm\omega(s) - \bar{\bm\omega}(s), \neg x\right) \\
        & \ge \frac{t - \lceil h(t) \rceil}{t} D\left(\bm\mu, \frac{\sum_{s=\lceil h(t) \rceil}^t \bar{\bm\omega}(s)}{t-\lceil h(t) \rceil}, \neg x\right) - h_3(t)
    \end{align*}
    where $h_3(t)$ is given by:
    \begin{align*}
        h_3(t) & \coloneqq \frac{1}{t} \min_{x \in \Xcal^{\star}(\bm\mu)} \inf_{\bm\lambda \in \neg x} \sum_{s=\lceil h(t) \rceil}^t \sum_{k \in [K]} (\bar{\omega}_k(s) - \omega_k(s)) d(\mu_k, \lambda_k) \\ & \le \kappa \min_{x \in \Xcal^{\star}(\bm\mu)} \inf_{\bm\lambda \in \neg x} d(\mu_k, \lambda_k) \\ & \le \kappa C_{\bm\mu},
    \end{align*}
    Therefore, for all $x \in \Xcal^{\star}(\bm\mu)$ we have that:
    \begin{align*}
        \frac{1}{t} \max_{x \in \Xcal^{\star}(\hat{\bm\mu}(t))} D(\hat{\bm\mu}(t), \bm{N}(t), \neg x) \ge \frac{t - \lceil h(t) \rceil}{t} D\left(\bm\mu, \frac{\sum_{s=\lceil h(t) \rceil}^t \bar{\bm\omega}(s)}{t-\lceil h(t) \rceil}, \neg x\right) - \kappa(1+C_{\bm\mu})  - \frac{K(1+\sqrt{t})}{t}C_{\bm\mu}.
    \end{align*}
    Hence:
    \begin{align*}
        \frac{1}{t} \max_{x \in \Xcal^{\star}(\hat{\bm\mu}(t))} D(\hat{\bm\mu}(t), \bm{N}(t), \neg x) \ge \max_{x \in \Xcal^{\star}(\bm\mu)}\frac{t - \lceil h(t) \rceil}{t} D\left(\bm\mu, \frac{\sum_{s=\lceil h(t) \rceil}^t \bar{\bm\omega}(s)}{t-\lceil h(t) \rceil}, \neg x\right) - \kappa(1+C_{\bm\mu})  - \frac{K(1+\sqrt{t})}{t}C_{\bm\mu}.
    \end{align*}
    At this point, we note that $\frac{1}{t - \lceil h(t) \rceil }\sum_{s=\lceil h(t) \rceil}^t \bar{\bm\omega}(s)$ is an average of points that belongs to $\bm\omega^{\star}(\bm\mu)$. Thus, this average belong to convex hull of $\bm\omega^{\star}(\bm\mu)$. Consequently, we obtain:
    \begin{align*}
        \frac{1}{t} \max_{x \in \Xcal^{\star}(\hat{\bm\mu}(t))} D(\hat{\bm\mu}(t), \bm{N}(t), \neg x) \ge \frac{t - \lceil h(t) \rceil}{t} \min_{\bm\omega \in \textup{conv}(\bm\omega^{\star}(\bm\mu))}\max_{x \in \Xcal^{\star}(\bm\mu)} D\left(\bm\mu, \bm\omega, \neg x\right) - \kappa(1+C_{\bm\mu})  - \frac{K(1+\sqrt{t})}{t}C_{\bm\mu}.
    \end{align*}
    thus concluding the proof.
\end{proof}

Finally, we are able to prove to prove the following result, \Cref{thm:not-conv}, on the performance of the presented framework whenever $\{x_t\}$ is not a converging sequence.

\begin{theorem}\label{thm:not-conv}
    Suppose that $\bm\mu \mapstoto \Xcal_F(\bm\mu)$ is not single-valued. Then, the presented framework $\delta$-correct, and it always holds that:
    \begin{align*}
        \limsup_{\delta \to 0} \frac{\E_{\bm\mu}[\tau_\delta]}{\log(1/\delta)} \le \max_{\bm\omega \in \textup{conv}(\bm\omega^{\star}(\bm\mu))} \min_{x \in \Xcal^{\star}(\bm\mu)} \frac{1}{D(\bm\mu, \bm\omega, \neg x)}.
    \end{align*}
\end{theorem}
\begin{proof}
    Let us define $\tilde{T}^*(\bm\mu)$ as follows:
    \begin{align*}
        \tilde{T}^*(\bm\mu) = \max_{\bm\omega \in \textup{conv}(\bm\omega^{\star}(\bm\mu))} \min_{x \in \Xcal^{\star}(\bm\mu)} \frac{1}{D(\bm\mu, \bm\omega, \neg x)}.
    \end{align*}

    Let $\kappa > 0$. From \Cref{lemma:stas-tas-subopt}, we have that, for $t \ge T_{\kappa}$, on $\mathcal{E}(t)$:
    \begin{align}\label{eq:tas-uniq-theo-1}
        \frac{1}{t} \max_{x \in \Xcal^\star(\bm{\hat{\mu}}(t))} D(\bm{\hat{\mu}}(t), \bm{N}(t), \neg x) \ge \frac{(t-\lceil h(t) \rceil)}{t} \tilde{T}^*(\bm\mu)^{-1} - \kappa(1+C_{\bm\mu})  - \frac{K(1+\sqrt{t})}{t}C_{\bm\mu},
    \end{align}
    where $C_{\bm{\mu}}$ is a problem-dependent constant.
    
    Let $\gamma \in \left(0, \frac{\tilde{T}^*(\bm{\mu})^{-1}}{4} \right]$, and take $\kappa = \kappa(\gamma) \le \min\left(\frac{\gamma}{4},\frac{\gamma}{4 C_{\bm\mu}}\right)$. Furthermore, consider ${T}_{\gamma}$ as the smallest $n \in \mathbb{N}$ such that $\frac{h(n)}{n} \tilde{T}^*(\bm{\mu})^{-1} \le \frac{\gamma}{4}$ and $\frac{K(1+\sqrt{n})}{n}  C_{\bm{\mu}} \le \frac{\gamma}{4}$. Then, for $t \ge T_{\gamma} + T_{\kappa(\gamma)}$, Equation~\eqref{eq:tas-uniq-theo-1} implies that
    \begin{align*}
        \frac{1}{t} \max_{x\in \Xcal^\star(\bm{\hat{\mu}}(t))} D(\bm{\hat{\mu}}(t), \bm{N}(t), \neg x) \ge \tilde{T}^*(\bm{\mu})^{-1} - \gamma.
    \end{align*}
    Applying \Cref{lemma:tech-lemma} with $\alpha = L =\gamma$ and $D = \tilde{T}^*(\bm\mu)$, we have that, for $t \ge T_0(\gamma, \gamma, \delta, \tilde{T}^*(\bm\mu)^{-1}) + T_{\gamma} + T_{\kappa(\gamma)}$:
    \begin{align*}
        \max_{x \in \Xcal^\star(\bm{\hat{\mu}}(t))} D(\bm{\hat{\mu}}(t), \bm{N}(t), \neg x) \ge \beta_{t,\delta},
    \end{align*}
    which implies that, on the good event, the algorithm stops at most at time $\lceil T_0(\gamma,\gamma,\delta, \tilde{T}^{*}(\bm\mu)^{-1})+ T_{\gamma} + T_{\kappa(\gamma)} \rceil$. 
    
    Moreover, from \Cref{lemma:control-expect}, we obtain that:
    \begin{align*}
        \mathbb{E}_{\bm{\mu}}[\tau_\delta] \le T_0(\gamma, \gamma, \delta, \tilde{T}^*(\bm\mu)^{-1}) + T_{\gamma} + T_{\kappa(\gamma)} + 1 + \sum_{s=0}^{+\infty} \mathbb{P}_{\bm{\mu}}(\mathcal{E}(s)^c).
    \end{align*}
    From Lemma 19 in \cite{garivier2016optimal}, $\sum_{s=0}^{+\infty} \mathbb{P}_{\bm{\mu}}(\mathcal{E}(s)^c)$ is finite. Thus, using the expression of $T_0(\gamma,\gamma,\delta,\tilde{T}^*(\bm\mu)^{-1})$ given by \Cref{lemma:tech-lemma}, we have that:
    \begin{align*}
        \limsup_{\delta \rightarrow0} \frac{\mathbb{E}_{\bm{\mu}}[\tau_\delta]}{\log(1 / \delta)} \le \frac{1}{\tilde{T}^*(\bm{\mu})^{-1}- 2\gamma}.
    \end{align*}
    Letting $\gamma \rightarrow 0$ concludes the proof.
\end{proof}

\subsubsection{An example of the gap between \Cref{theo:sticky-seq} and \Cref{thm:not-conv}}\label{app:not-conv-example}

In the following, we provide a simple example where we can precisely evaluate the difference between \Cref{theo:sticky-seq} and \Cref{thm:not-conv}.

Consider a problem with $K=2$, $\bm\mu=(\mu_1, \mu_2)$ with Bernoulli distributions over $\mathcal{M} = [\alpha, 1-\alpha]$, and the task is to regress any of the two arms.
Note that the task is non-trivial as it is intuitively more convenient to sample the arm with less variance.
In this case, the set of correct answers is $\mathcal{X}^{\star}(\bm\mu) = \{ (x,i): |\mu_i - x| \le \epsilon \}$. 
Consider $\alpha, \epsilon\in (0, 1/10)$ and $\bm\mu$ such that $\mu_1 = \alpha$ and $\mu_2 = 1-\alpha$. By symmetry of the KL divergence at $\alpha$ and $1-\alpha$, we have that the set of optimal weights are trivially $(1,0)$ for the answer $(\mu_1, 1)$ and $(0,1)$ for the answer $(\mu_2, 2)$. 

In this case, thus, we have that: $$T^{\star}(\bm\mu) = \frac{1}{d(\alpha, \alpha+\epsilon))}=\frac{1}{d(1-\alpha, 1-\alpha-\epsilon)}.$$ 

On the other hand, $\tilde{T}^{\star}(\bm\mu)$ actually evaluates to $2T^{\star}(\bm\mu)$ (\ie as the worst weights over the convex hull are clearly $(0.5, 0.5)$).

\section{On simple solutions}\label{app:experiments}

This section is dedicated to show why some simple solution to the infinite answer problems fails, \ie they incur in sub-optimal sample complexity.

More precisely:
\begin{itemize}
    \item In \Cref{app:alpha-neta}, we discuss why the intuitive approach of discretizing the answer space and to apply Sticky Track-and-Stop on the resulting finite answer problem is not a good option to achieve statistical efficiency.
    \item In \Cref{app:empirical-failure-stas}, instead, we provide empirical evidence for the problem that Sticky Track-and-Stop faces when applied to an infinite answer space. This complements the discussion that we presented in \Cref{sec:property}, where we argued that the current proof of optimality of Sticky Track-and-Stop cannot be applied to the infinite answer setting. Indeed, the presence of infinite answers undermine the core argument behind optimality of Sticky Track-and-Stop (\ie the fact that it can stick to a correct answer in $\Xcal_F(\bm\mu)$). 
\end{itemize}

\subsection{Why discretizing the answer space does not work}\label{app:alpha-neta}
In this section, we comment upon a simple approach that one might want to take when solving infinite answer problems on a compact, yet infinite, answer space. The simplest idea that one might have is to construct an $\alpha$-net over the answer space and to apply Sticky Track-and-Stop to the resulting finite-answer problem. This solves the issue of sticking to a single-answer that we identified and discussed in \Cref{sec:property}. Nonetheless, as we now show, this automatically implies a loss in the statistical efficiency.

\subsubsection{A simple example that explains the failure}
We highlight the limits of this approach with a simple example. Consider the case where $K=1$ and the target is the $\epsilon$-regression of $\mu_1$. In the following, we simply write $\mu$ since $K=1$. Assume the distribution to be Gaussian with variance $1/2$ and let $\mathcal{M}=\mathcal{X}=[-1,1]$ . Then, it is easy to verify that \Cref{theo:lb} leads to $T^{\star}(\mu) = \frac{1}{\epsilon^2}$, and that this lower bound is only attained at $\mathcal{X}_F(\mu)=\mu$. 

Now, denote by $\bar{\mathcal{X}}_{\alpha}$ the points in the $\alpha$-net and suppose that $\alpha < \epsilon$. Then, for any $\mu \notin \bar{\mathcal{X}}_{\alpha}$, Sticky Track-and-Stop applied on the discretized answer space cannot select the statistically convenient answer $\mu \in \mathcal{X}_F(\mu)$ but can only select answers within $\bar{\mathcal{X}}_{\alpha}$. This implies that the algorithm can only attain an asymptotic upper bound of the form of $\frac{1}{(\epsilon - \alpha)^2}$, therefore failing to achieve optimality. 

\subsubsection{Abstracting away from the example}
We observe that the previous scenario is only a very simple example and that discretizing "blindly" with an accuracy $\alpha$ can lead to more severe issues. 

More generally, to apply an $\alpha$-net and choose $\alpha$ so as to control the statistical error introduced by the discretization, one would need a Lipschitz continuity bound on the function $\frac{1}{D(\mu, \neg x)}$, and then take the minimum of such bound over all the possible $\bm\mu \in \mathcal{M}$. This is highly non-trivial: even in $D(\bm\mu, \neg x)$ (for which we can only prove a uniformly continuous result) the dependence in $\neg x$ enters through the (possibly non-convex) constraints of the underlying optimization problem. 

Finally, any discretization must ensure that, for every $\bm\mu \in \mathcal{M}$, there exists some $x \in \bar{\mathcal{X}}_{\alpha} \cap \mathcal{X}^{\star}(\bm\mu)$ such that $\bm\mu \notin \cl(\neg x)$. If this condition fails, \Cref{ass:identify} no longer holds in the resulting finite-answer problem, and the upper bounds for the discretized version would have infinite sample complexity.

\subsection{Empirical Failure of Sticky Track-and-Stop}\label{app:empirical-failure-stas}

\begin{figure*}[t]
\centering
\begin{tikzpicture}[x=0.65pt,y=0.75pt,xscale=1]

  \newcommand{\cross}[4]{%
    \draw[#4, line width=1.5pt] (#1-#3,#2-#3) -- (#1+#3,#2+#3);
    \draw[#4, line width=1.5pt] (#1-#3,#2+#3) -- (#1+#3,#2-#3);
  }

  \newcommand{\plus}[4]{%
  \draw[#4, line width=1.5pt] (#1,#2-#3) -- (#1,#2+#3);
  \draw[#4, line width=1.5pt] (#1-#3,#2) -- (#1+#3,#2);
  }

  \def\panelsep{240}

  \begin{scope}[shift={(0,0)}]
    \draw[step=10, very thin, gray!40] (-100,-100) grid (100,100);
    \draw[thick] (-100,-100) rectangle (100,100);
    \draw[->] (-100,0)--(100,0) node[right] {$x$};
    \draw[->] (0,-100)--(0,100) node[above] {$y$};
    \foreach \x in {-75,0,75} {
      \node[below] at (\x,-104) {\small \x};
    }
    \foreach \y in {-75,0,75} {
      \draw (-2,\y) -- (2,\y);
      \node[left]  at (-104,\y) {\small \y};
    }

    \def\cLx{3}\def\cLy{-72}\def\muLx{30}\def\muLy{18}
    \filldraw[fill=typ_blue, draw=black] (\cLx-\muLx,\cLy-\muLy) rectangle (\cLx+\muLx,\cLy+\muLy);
    \def\cUx{-4}\def\cUy{77}\def\muUx{26}\def\muUy{20}
    \filldraw[fill=typ_blue, draw=black] (\cUx-\muUx,\cUy-\muUy) rectangle (\cUx+\muUx,\cUy+\muUy);

    \plus{\cLx}{\cLy}{5}{typ_fuchsia,line cap=round}
    \plus{\cUx}{\cUy}{5}{typ_fuchsia,line cap=round}
    \cross{0}{-75}{5}{typ_yellow!85!black,line cap=round}
    \cross{0}{75}{5}{typ_yellow!85!black,line cap=round}

    \fill (33,-54) circle[radius=2.6]; %
  \end{scope}

  \begin{scope}[shift={(\panelsep,0)}]
    \draw[step=10, very thin, gray!40] (-100,-100) grid (100,100);
    \draw[thick] (-100,-100) rectangle (100,100);
    \draw[->] (-100,0)--(100,0) node[right] {$x$};
    \draw[->] (0,-100)--(0,100) node[above] {$y$};
    \foreach \x in {-75,0,75} {
      \node[below] at (\x,-104) {\small \x};
    }
    \foreach \y in {-75,0,75} {
      \draw (-2,\y) -- (2,\y);
      \node[left]  at (-104,\y) {\small \y};
    }

    \def\cLx{2}\def\cLy{-73.5}\def\muLx{22}\def\muLy{13}
    \filldraw[fill=typ_blue, draw=black] (\cLx-\muLx,\cLy-\muLy) rectangle (\cLx+\muLx,\cLy+\muLy);
    \def\cUx{3}\def\cUy{76}\def\muUx{24}\def\muUy{16}
    \filldraw[fill=typ_blue, draw=black] (\cUx-\muUx,\cUy-\muUy) rectangle (\cUx+\muUx,\cUy+\muUy);

    \plus{\cLx}{\cLy}{5}{typ_fuchsia,line cap=round}
    \plus{\cUx}{\cUy}{5}{typ_fuchsia,line cap=round}
    \cross{0}{-75}{5}{typ_yellow!85!black,line cap=round}
    \cross{0}{75}{5}{typ_yellow!85!black,line cap=round}

    \fill (27,92) circle[radius=2.6]; %
  \end{scope}

  \begin{scope}[shift={(2*\panelsep,0)}]
    \draw[step=10, very thin, gray!40] (-100,-100) grid (100,100);
    \draw[thick] (-100,-100) rectangle (100,100);
    \draw[->] (-100,0)--(100,0) node[right] {$x$};
    \draw[->] (0,-100)--(0,100) node[above] {$y$};
    \foreach \x in {-75,0,75} {
      \node[below] at (\x,-104) {\small \x};
    }
    \foreach \y in {-75,0,75} {
      \draw (-2,\y) -- (2,\y);
      \node[left]  at (-104,\y) {\small \y};
    }

    \def\cLx{1}\def\cLy{-74.5}\def\muLx{16}\def\muLy{10}
    \filldraw[fill=typ_blue, draw=black] (\cLx-\muLx,\cLy-\muLy) rectangle (\cLx+\muLx,\cLy+\muLy);
    \def\cUx{0.5}\def\cUy{75.5}\def\muUx{12}\def\muUy{9}
    \filldraw[fill=typ_blue, draw=black] (\cUx-\muUx,\cUy-\muUy) rectangle (\cUx+\muUx,\cUy+\muUy);

    \plus{\cLx}{\cLy}{5}{typ_fuchsia,line cap=round}
    \plus{\cUx}{\cUy}{5}{typ_fuchsia,line cap=round}
    \cross{0}{-75}{5}{typ_yellow!85!black,line cap=round}
    \cross{0}{75}{5}{typ_yellow!85!black,line cap=round}

    \fill (17,-64.5) circle[radius=2.6]; %

  \end{scope}

\def\yline{-140} %

\draw[thick,->] (-100,\yline) -- (2*\panelsep + 100,\yline);

\foreach \i/\lab in {0/$t_1$,1/$t_2$,2/$t_3$} {
  \draw (\i*\panelsep,\yline-6) -- (\i*\panelsep,\yline+6); %
  \node[below] at (\i*\panelsep,\yline-10) {\large \lab};   %
}

\end{tikzpicture}
\caption{In \textbf{\textcolor{typ_blue}{blue}} we are visualizing how the set $\mathcal{X}_t$ is (approximately) behaving. The plus signs in \textbf{\textcolor{typ_fuchsia}{fuchsia}} represents the empirical estimates of the first two arms and the last two arms that are generating $\mathcal{X}_t$. The \textbf{\textcolor{myellow!90!}{yellow}} crosses represents the two unique answers in $\mathcal{X}_F(\bm\mu)$. Finally, the \textbf{black} dot is the answer $x_t$ selected by Sticky Track-and-Stop as it first maximizes the x-axis and then maximizes the y-axis. As a result, the algorithm oscillates between picking answers whose oracle weights are different.} 
\label{fig:new}
\end{figure*}

In the following, we present a simple example which highlights the core issue of Sticky Track-and-Stop in an infinite answer setting.

\subsubsection{A simple example that explains the failure}
Consider a bandit problem with $K=4$  arms and Gaussian rewards with unitary variance. Let $\mathcal{M} = [-100, 100]^4$ and consider the problem of regressing up to an $\epsilon$-accuracy (e.g., $\epsilon=0.1$) either $(\mu_1, \mu_2)$ or $(\mu_3, \mu_4)$. Then, $\mathcal{X} = [-100, 100]^2$ and $$\mathcal{X}^{\star}(\mu) = \{ (x,y) \in \mathcal{X}: (|\mu_1 - x| \le \epsilon \land |\mu_2-y|\le \epsilon) \lor (|\mu_3 - x| \le \epsilon \land |\mu_4 - y| \le \epsilon ) \}.$$ 

Here, it is easy to verify that, for all $\bm\lambda$ such that $\mu_2 \ll \mu_4$, $\mathcal{X}_F(\bm\lambda)$ directly evaluates to $\{(\lambda_1, \lambda_2)\} \cup \{(\lambda_3, \lambda_4)\}$ with corresponding oracle weights $(0.5, 0.5, 0, 0)$ and $(0, 0, 0.5, 0.5)$. 

Consider the instance $\bm\mu = (0, -75, 0, 75)$ and the total order that first minimizes/maximizes the x-axis. Then, answers in $\mathcal{X}_t$ have approximately the following shape $[\hat{\mu}_1 \pm \textup{CI}_1, \hat{\mu}_2 \pm \textup{CI}_2]$ and $[\hat{\mu}_3 \pm \textup{CI}_3, \hat{\mu}_4 \pm \textup{CI}_4]$. Now, we note that:
\begin{itemize}
    \item These two sets are well-separated by assumption on the y-axis but almost overlapping on the x-axis. 
    \item Answers of the first set are related to weights that samples the first two arms, while answers in the second set leads to sampling arms 3 and 4.
\end{itemize}

As a consequence,  since Sticky-TaS is using the total order that first minimizes over the x-axis, it will switch between picking answers in $x_t$ from the two sets, and the corresponding expected pull proportions for small $\delta$ lies on the convex hull of $\bm\omega^{\star}(\bm\mu)$. We have provided a visualization of this behavior in \Cref{fig:new}.

\subsubsection{Numerical Simulations}
\begin{figure*}[t]
\centering
  \includegraphics[width=0.49\textwidth]{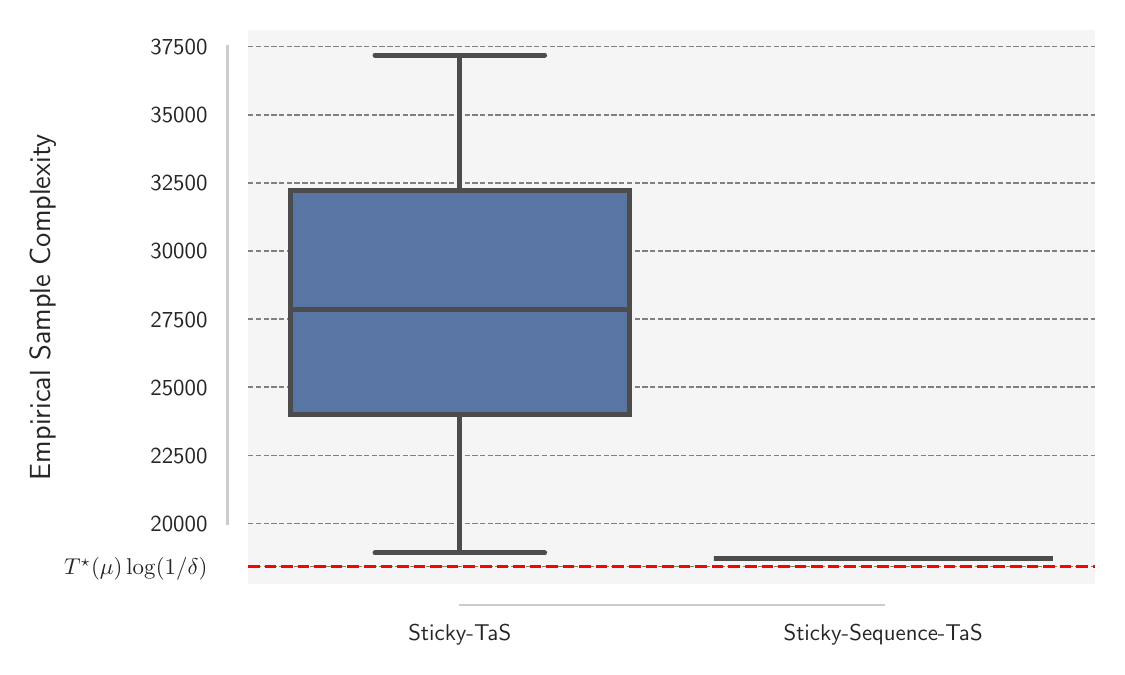}
  \includegraphics[width=0.49\textwidth]{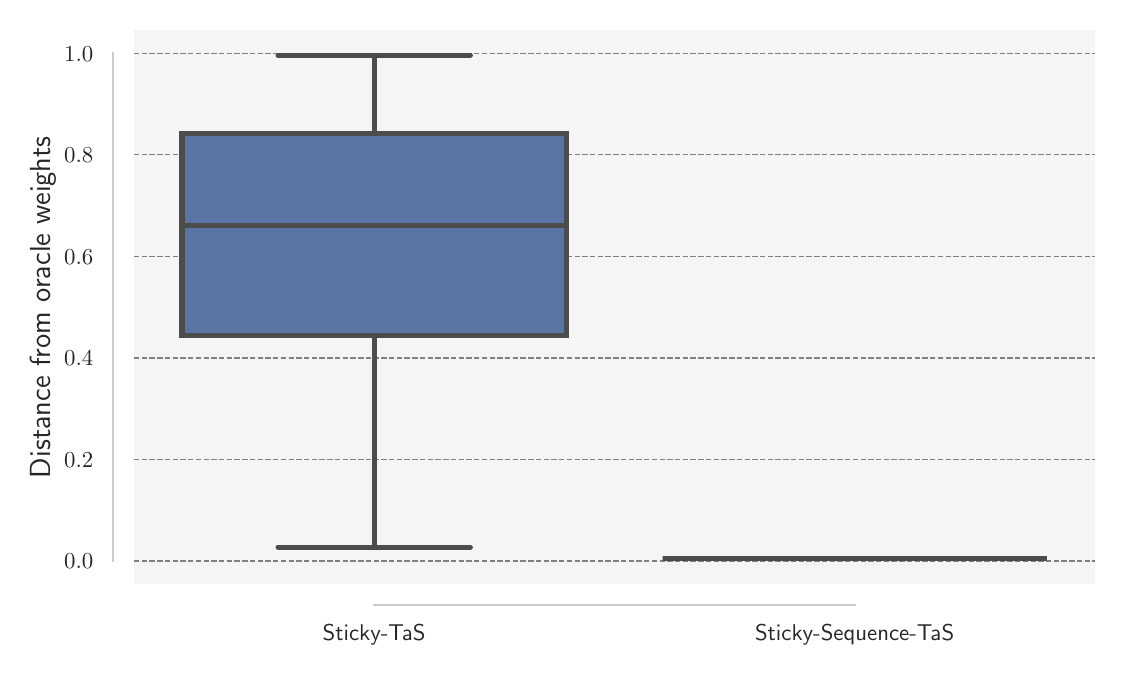}
  \caption{Empirical sample complexity of Sticky Track-and-Stop and Sticky Sequence Track-and-Stop on the presented regression task (\emph{left}) and empirical $l_1$ distance of the empirical pull proportions at stopping from the closest $\bm\omega \in \bm\omega^{\star}(\bm\mu)$; more formally $\inf_{\bm\omega \in \bm\omega^{\star}(\bm\mu)} \| \bm{N}(\tau_\delta)/\tau_\delta - \bm\omega \|_1$ (\emph{right}).}
  \label{fig:exp}
\end{figure*}

A part from the intuition that we provided above, we have also verified this issue of Sticky Track-and-Stop numerically. We have run the experiments on $8$ CPUs Intel(R) Core(TM) i5-8250U CPU @ 1.60GHz and $8$ GB of RAM. In the simulation, we used $\delta=10^{-80}$. We repeated the experiment $100$ times and we compared the results with the ones of Sticky Sequence Track-and-Stop. Since $\mathcal{X}_F(\bm\mu)$ is finite, we implemented the version that picks the next answer as $x_{t} \in \argmin_{x \in \mathcal{X}_t} \|x - x_{t-1} \|_{\infty}$. 

The results are presented in \Cref{fig:exp}. As one can see from \Cref{fig:exp} (\emph{left}), the empirical sample complexity of Sticky Track-and-Stop is way higher than the one of Sticky Sequence Track-and-Stop, which, instead, is approaching the lower bound at $T^{\star}(\bm\mu) \log(1/\delta)$ (the red line in the plot). The issue, as we anticipated above, is that Sticky Track-and-Stop ends up in playing in the convex hull of $\bm\omega^{\star}(\bm\mu)$. To verify this, in \Cref{fig:exp} (\emph{right}), we report the $l_1$ distance of the empirical pull proportions at the stopping time $\tau_\delta$ from the closest $\bm\omega \in \bm\omega^{\star}(\bm\mu)$; more formally $\inf_{\bm\omega \in \bm\omega^{\star}(\bm\mu)} \| \bm{N}(\tau_\delta)/\tau_\delta - \bm\omega \|_1$. As one can see, Sticky Sequence Track-and-Stop ends up playing in expectation with proportions that are very close to  $(0.5, 0.5, 0, 0)$ or $(0, 0, 0.5, 0.5)$. Sticky Track-and-Stop, on the other hand, will select answers in an oscillating fashion that leads to sub-optimal behavior.

\subsubsection{Concluding Remarks}
We conclude with two comments:
\begin{itemize}
    \item The issue on the instance $\bm\mu$ that we considered might be solved by considering the total order that first minimizes/maximizes over $y$. However, one could create an instance where $\mu_2 = \mu_4$ while $\mu_1$ and $\mu_3$ are well-separated. 
    \item We presented the problem in the simple setting of Gaussian distributions, for which an algorithm that samples arms 1 and 2 in equal proportion would suffice. Nonetheless, in the Bernoulli case no such fixed-proportion oracle exists and an adaptive algorithm becomes necessary (as sampling is more convenient where the variance is lower). Although the extension incurs additional mathematical complexity, our example can be generalized to cover this Bernoulli scenario. Specifically, we can pick $\mu_1$ and $\mu_3$ equal to $\alpha$, for some small $\alpha$, and $\mu_3 = 0.5 + \alpha$ and $\mu_4 = 0.5 - \alpha$. Let $\epsilon \ll \alpha$. Here, due to the symmetry of the KL divergence around $0.5$, we have that regressing arm 1 and 2 is equivalent to regressing 3 and 4. Moreover, answers in $\mathcal{X}_t$ will still have the property of splitting into two distinct regions over the y-axis and the total order used by Sticky Track-and-Stop will lead to the same issue.
\end{itemize}

\section{Helper Lemmas}

\subsection{Tracking}
The following lemma is a standard result related to the C-Tracking sampling rule. Tighter constant dependencies can be obtained using more fine-grained analysis \cite{degenne2020structure}. 

\begin{lemma}[Lemma 7 in \cite{garivier2016optimal}]\label{lemma:tracking}
    For all $k \in [K]$ and all $t \ge 1$ it holds that $N_k(t) \ge \sqrt{t+K^2} - 2K$ and $\max_{k \in [K]}|N_k(t) - \sum_{s=1}^t \omega_k(s) |\le K(1+\sqrt{t})$.
\end{lemma}

\subsection{Stopping time}
Similarly, the following lemma is useful in controlling the stopping time of the proposed algorithms. This result generalizes Lemma 1 in \cite{menard2019gradient}.

\begin{lemma}\label{lemma:tech-lemma}
    Consider $D \ge 0$.
    Consider $\alpha > 0$ and $L \in \mathbb{R}$ such that $D - \alpha - L > 0$. There exists $C_{\alpha} > 0$ such that, for:
    \begin{align*}
        T \ge \max \left\{  C_{\alpha}, \frac{\log\left( 1/\delta \right) + K \log\left( 4 \log(1 / \delta) + 1 \right) }{D - \alpha - L} \right\} \coloneqq T_0(\alpha,L,\delta, D)
    \end{align*}
    it holds that $D-L \ge \frac{\beta_{T,\delta}}{T}$.
\end{lemma}
\begin{proof}
    Let $C_\alpha$ be such that, for $T \ge C_{\alpha}$ it holds that $6K \log(\log(T)+3) + K\tilde{C} \le \alpha T$.
Then, for $T \ge T_0(\alpha,L,\delta,D)$, we have that:
\begin{align*}
    \frac{\beta_{T,\delta}}{T} & = \frac{\log\left( \frac{1}{\delta} \right) + K \log \left( 4 \log\left( \frac{1}{\delta} \right) + 1 \right) + 6K \log(\log(T)+3) + K\tilde{C}}{T} \\ & \le \frac{\log\left( \frac{1}{\delta} \right) + K \log \left( 4 \log\left( \frac{1}{\delta} \right) + 1 \right)}{T} + \alpha \\ & \le D - L.
\end{align*}
where (i) in the first step we have used the definition of $\beta_{t,\delta}$, i.e., Equation~\eqref{eq:beta-thr}, (ii) in the second one the definition of $C_\alpha$, and (iii) in the third one the definition of $T_0(\alpha,L,\delta,D)$.
\end{proof}

\section{Mathematical Background}

This section contains mathematical background that can be helpful throughout the document. Specifically, \Cref{app:set-theory} shows that is possible to have an exact covering of a compact set. \Cref{sec:exponential_fam} contains useful information on canonical exponential families. \Cref{app:inf} provides simple results on the infimum of optimization problems. Finally, \Cref{app:corresp} presentes auxiliary results on correspondences and set-valued analysis.

\subsection{Set Theory}\label{app:set-theory}

\begin{lemma}[Exact covering of a compact set]\label{lemma:covering}
    Let $\mathcal{X} \subset \mathbb{R}^d$ be a compact set. For all $\rho > 0$, there exists $n_{\rho} \in \mathbb{N}$ finite and $\{ \mathcal{X}_i \}_{i=1}^{n_{\rho}}$ such that $\mathcal{X}_i \subset \mathbb{R}^d$ is compact, $\mathcal{X} = \bigcup_{i=1}^{n_{\rho}} \mathcal{X}_i$, and for all $i\in [n_\rho]$, $\mathcal{X}_i \subseteq {\Bcal}_{\rho}(x_i)$ for some $x_i \in \mathcal{X}$. 
\end{lemma}
\begin{proof}
    Let $\bar{\Bcal}_{\rho}(x) = \{ y \in \mathbb{R}^d: || x - y ||_{\infty} < \rho \}$ be an open ball of radius $\rho$ centered in $\rho$. Then, for all $\rho > 0$, it holds that $\mathcal{X} \subset \bigcup_{x \in \mathcal{X}} \bar{\Bcal}_{\rho}(x_i)$.
    Furthermore, since $\mathcal{X}$ is compact, every open cover of a compact set admits a finite subcover, that is, there exists  $n_{\rho} \in \mathbb{N}$ finite and a collection of points $\{ x_i \}_{i=1}^{n_{\rho}}$ such that $
        \mathcal{X} \subset \bigcup_{i=1}^{n_{\rho}} \bar{\Bcal}_{\rho}(x_i)$, where each $x_i \in \mathcal{X}$ by construction. Thus, by taking the closure of each ball, we obtain, 
        $\mathcal{X} \subset \bigcup_{i=1}^{n_{\rho}} {\Bcal}_{\rho}(x_i)$. Finally, we have that:
    \begin{align*}
        \mathcal{X} = \bigcup_{i=1}^{n_{\rho}} \left( {\Bcal}_{\rho}(x_i)  \cap \mathcal{X} \right) \coloneqq \bigcup_{i=1}^{n_{\rho}} \mathcal{X}_i.
    \end{align*}
    To conclude the proof, we note that the intersections of compact euclidean subsets is compact. Hence, $\mathcal{X}_i =  {\Bcal}_{\rho}(x_i)  \cap \mathcal{X}$ is compact.
\end{proof}

\subsection{Canonical Exponential Family}\label{sec:exponential_fam}

In a canonical and one-parameter exponential family, distributions are indexed according to a parameter $\phi \in \Phi$, and each distribution $\nu_{\phi}$ is absolutely continuous with respect to a reference measure $\rho$ on $\mathbb{R}$ such that:
\begin{align*}
    \frac{\text{d} \nu_{\eta}}{\text{d}\rho}(x) = \exp(x\eta - b(\eta)),
\end{align*}
where $b: \Phi \rightarrow \mathbb{R}$ is a twice differentiable convex function. Each distribution $\nu_{\phi}$ can be uniquely identified with its mean $\mu$, which is given by $\dot{b}(\eta)$. Given an interval of open means (i.e., the family is regular), $b$ is strictly convex on that interval, and the distribution is non-degenerate, meaning that its variance is strictly positive. The KL distribution between two distributions $\nu_{\eta}, \nu_{\eta'}$ with means $\mu, \mu'$, is given by:
\begin{align*}
    \text{KL}(\nu_{\eta}, \nu_{\eta'}) = d(\mu, \mu') = b(\eta') - b(\eta) - \dot{b}(\eta)(\eta' - \eta).
\end{align*}

Now, consider two bandits $\bm{\mu}$ and $\bm{\lambda}$. After $t$ rounds it holds that 
\begin{align*}
    \ln \frac{\text{d}\mathbb{P}_{\bm{\mu}}}{\text{d}\mathbb{P}_{\bm{\lambda}}} & = \ln \frac{\text{d}\mathbb{P}_{\hat{\bm{\mu}}(t)}}{\text{d}\mathbb{P}_{\bm{\lambda}}} - \ln \frac{\text{d}\mathbb{P}_{\hat{\bm{\mu}}(t)}}{\text{d}\mathbb{P}_{\bm{\mu}}} \\ & = \sum_{k \in [K]} N_k(t) \left(d(\hat{\mu}_k(t),\lambda_k) - d(\hat{\mu}_k(t), \mu_k)\right) \\ & = \sum_{k \in [K]} N_k(t) \left(d(\mu_k, \lambda_k) + (\tilde{\lambda}_k - \tilde{\mu}_k)(\mu_k-\hat{\mu}_k(t)) \right),
\end{align*}
where $\tilde{\mu}_k=\dot b^{-1}(\mu_k)$ and $\tilde{\lambda}_k=\dot b^{-1}(\lambda_k)$ represents the natural parameter of the distributions with mean $\lambda_k$ and $\mu_k$ respectively. We also recall that $\sum_{k \in [K]} N_k(t)(\tilde{\lambda}_k - \tilde{\mu}_k)(\mu_k-\hat{\mu}_k(t)))$ is a martingale. We also recall that in the third equality above, we used the following property of canonical exponential family.

\begin{lemma}[KL difference in canonical exponential families]\label{lemma:kl-diff}
    Consider three distributions in a one-dimensional exponential family with means $a,b,c$. Then, it holds that:
    \begin{align*}
        d(a,b) = d(a,c) + d(c,b) + (\tilde{b} - \tilde{c})(c-a).
    \end{align*}
\end{lemma}
\begin{proof}
    For a proof, see e.g., Lemma E.6 in \cite{poiani2024best}.
\end{proof}
  
Finally, we show lipschitzianity properties of the KL divergence and natural parameters when dealing with canonical exponential family. This result is well-known and we report a proof for completeness.

\begin{lemma}[Local lipschitzianity of KL divergence and natural parameters]\label{lemma:local-lip}
    Consider $\mu, \lambda \in \Theta$, and denote by $\tilde{\mu}, \tilde{\lambda}$ their parameter in the canonical exponential family. Then, it holds that:
    \begin{align*}
        & |\tilde{\lambda} - \tilde{\mu}| \le C_{1, \mu, \lambda} |\lambda-\mu|  \\ 
        & d(\mu, \lambda) \le C_{2, \mu, \lambda} (\lambda-\mu)^2,
    \end{align*}
    where,
    \begin{align*}
        & C_{1, \mu, \lambda} \coloneqq \frac{1}{\min_{\xi \in [\min\{\tilde{\mu}, \tilde{\lambda} \}, \max \{\tilde{\mu}, \tilde{\lambda} \}  ]} \ddot{b}(\xi)} \quad\textup{and}\quad C_{2, \mu, \lambda} \coloneqq \frac{C_{1,\mu,\lambda}}{2} \max_{\xi \in [\min\{ \tilde{\mu}, \tilde{\lambda} \}, \max\{ \tilde{\mu}, \tilde{\lambda}\}]} \ddot{b}(\xi).
    \end{align*}
\end{lemma}
\begin{proof}
    We first recall that $b(\cdot)$ is convex and twice differentiable. Therefore, we have that:
    \begin{align*}
        b(\tilde{\lambda}) \le b(\tilde{\mu}) + \mu (\tilde{\lambda} - \tilde{\mu}) + (\tilde{\lambda} - \tilde{\mu})^2 \max_{\xi \in [\min\{ \tilde{\mu}, \tilde{\lambda} \}, \max\{ \tilde{\mu}, \tilde{\lambda}\}]} \frac{\ddot{b}(\xi)}{2}.
    \end{align*}
    Plugging this result within $d(\mu,\lambda)$, we obtain:
    \begin{align}\label{eq:kl-smooth-eq-1}
        d(\mu,\lambda) \le (\tilde{\lambda} - \tilde{\mu})^2 \max_{\xi \in [\min\{ \tilde{\mu}, \tilde{\lambda} \}, \max\{ \tilde{\mu}, \tilde{\lambda}\}]} \frac{\ddot{b}(\xi)}{2}.
    \end{align}
    We now recall that $\dot{b}(\tilde{\mu}) = \mu$. Suppose that $\tilde{\mu} < \tilde{\lambda}$. Then, by the mean value theorem of integration, we have that:
    \begin{align*}
        \ddot{b}(\tilde{c}) = \frac{\lambda - \mu}{\tilde{\lambda} - \tilde{\mu}} \text{ for $\tilde{c} \in [\tilde{\mu}, \tilde{\lambda}]$}. \implies \tilde{\lambda} - \tilde{\mu} = \frac{\lambda - \mu}{\ddot{b}(\tilde{c})}.
    \end{align*}
    Similarly, when $\tilde{\lambda} < \tilde{\mu}$, we have:
    \begin{align*}
        \ddot{b}(\tilde{c}) = \frac{\mu - \lambda}{\tilde{\mu} - \tilde{\lambda}} \text{ for $\tilde{c} \in [\tilde{\lambda}, \tilde{\mu}]$}. \implies \tilde{\mu} - \tilde{\lambda} = \frac{\mu - \lambda}{\ddot{b}(\tilde{c})}.
    \end{align*}
    Chaining these results, we obtain:
    \begin{align*}
        (\tilde{\lambda} - \tilde{\mu})^2 \le \frac{(\lambda - \mu)^2}{\left(\min_{\xi \in [\min\{ \tilde{\mu}, \tilde{\lambda} \}, \max \{ \tilde{\mu}, \tilde{\lambda} \}  ]} \ddot{b}(b^{-1}(\xi)) \right)^2}.
    \end{align*}
    Notice that, since $\Theta$ is an open interval, that min is well-defined and different from zero.
    
    Plugging this upper bound within Equation~\ref{eq:kl-smooth-eq-1}, we can conclude the proof:
    \begin{align*}
        d(\mu, \lambda) & \le (\lambda-\mu)^2 \frac{\max_{\xi \in [\min\{ \tilde{\mu}, \tilde{\lambda} \}, \max\{ \tilde{\mu}, \tilde{\lambda}\}]} \ddot{b}(\xi)}{2\left( \min_{\xi \in [\min\{ \tilde{\mu}, \tilde{\lambda} \}, \max \{ \tilde{\mu}, \tilde{\lambda} \}  ]} \ddot{b}(\xi) \right)^2}.
    \end{align*}    
\end{proof}

As a corollary of \Cref{lemma:local-lip}, we have that the parameter space and the KL divergence are Lipschitz on any compact subset of the parameter space $\Theta$. Again, this result is well-known and we report a proof for completeness. We remark that, since we assumed $\Theta$ to be strictly contained in an open interval, then, we enjoy this stronger Lipschitzianity result. 

\begin{corollary}[Lipschitzianity over a compact set]\label{coroll:kl-lip}
     Let $\widetilde{\Theta} \subset \Theta$ be a compact set.
     Then, there exists constants $C_1, C_2 > 0$ such that, for all $\mu, \lambda \in \widetilde{\Theta}$, it holds that:
    \begin{align*}
        & |\tilde{\mu} - \tilde{\lambda}| \le C_1 |\mu - \lambda| \quad \textup{and}\quad d(\mu, \lambda) \le C_2 (\lambda - \mu)^2.
    \end{align*}
    Furthermore, there exists $D_1, D_2 > 0$ such that $|\tilde{\mu} - \tilde{\lambda}| < D_1$ and $d(\mu, \lambda) < D_2$.  
\end{corollary}
\begin{proof}
    Since $\widetilde{\Theta}$ is compact, there exists $\mu_{\textup{min}}, \mu_{\textup{min}}$ s.t. $\mu \in [\mu_{\textup{min}}, \mu_{\textup{max}}]$ for all $\mu \in \widetilde{\Theta}$. In particular, let $\mu_{\textup{min}} = \min_{\mu \in \widetilde{\Theta}} \mu$ and $\mu_{\textup{max}} = \max_{\mu \in \widetilde{\Theta}} \mu$. 
    Denote by $\widetilde{\Phi} = \{ \dot{b}^{-1}(\mu): \forall \mu \in  [\mu_{\textup{min}}, \mu_{\textup{max}}]\}$. Since $[\mu_{\textup{min}}, \mu_{\textup{max}}]$ is compact and $\dot{b}^{-1}(\cdot)$ is a continuous function, $\widetilde{\Phi}$ is compact as well. 

    Now, from \Cref{lemma:kl-diff}, we know that, for all $\mu, \lambda \in \widetilde{\Theta}$, it holds that:
    \begin{align*}
        |\tilde{\lambda} - \tilde{\mu}| & \le \frac{|\lambda - \mu|}{\min_{\xi \in [\min\{\tilde{\mu}, \tilde{\lambda} \}, \max \{\tilde{\mu}, \tilde{\lambda} \}  ]} \ddot{b}(\xi)} \\ 
        & \le \frac{|\lambda - \mu|}{\inf_{\xi \in \widetilde{\Phi}}\ddot{b}(\xi)} \\ 
        & = \frac{|\lambda - \mu|}{\min_{\xi \in \widetilde{\Phi}}\ddot{b}(\xi)} \\
        & \coloneqq C_1 |\lambda - \mu|
    \end{align*}
    where the inf can be replaced with a min since the optimization set is compact and $\ddot{b}(\cdot)$ is continuous over $\widetilde{\Theta}$. This  is due to the fact that $\widetilde{\Theta} \subset \Theta$ and $\Theta$ is an open interval. Hence, the exponential family is regular and the function $b(\cdot)$ is $\mathcal{C}^{\infty}$, see Theorem 5.8 in~\cite{lehmann2006theory}. Notice, furthermore, that the regularity of the exponential family also implies that $\ddot{b}$ is strictly positive over the considered domain. Finally, taking $D_1 \coloneqq C_1 |\mu_{\textup{max}} - \mu_{\textup{min}}|$ shows that $|\tilde{\mu} - \tilde{\lambda}|$ is bounded.

    We now analyze divergence $d(\cdot, \cdot)$ using similar arguments. From \Cref{lemma:kl-diff}, we have that:
    \begin{align*}
        d(\mu, \lambda) & \le \frac{\max_{\xi \in [\min\{ \tilde{\mu}, \tilde{\lambda} \}, \max\{ \tilde{\mu}, \tilde{\lambda}\}]} \ddot{b}(\xi)}{\min_{\xi \in [\min\{ \tilde{\mu}, \tilde{\lambda} \}, \max\{ \tilde{\mu}, \tilde{\lambda}\}]} \ddot{b}(\xi)} (\lambda - \mu)^2 \\ 
        & \le \frac{\sup_{\xi \in \widetilde{\Theta}} \ddot{b}(\xi)}{\inf_{\xi \in \widetilde{\Theta}} \ddot{b}(\xi)} (\lambda - \mu)^2 \\ & 
        = \frac{\max_{\xi \in \widetilde{\Theta}} \ddot{b}(\xi)}{\min_{\xi \in \widetilde{\Theta}} \ddot{b}(\xi)} (\lambda - \mu)^2 \\ & 
        \le \frac{\max_{\xi \in \widetilde{\Theta}} \ddot{b}(\xi)}{\min_{\xi \in \widetilde{\Theta}} \ddot{b}(\xi)} \left( \mu_{\textup{max}} - \mu_{\textup{min}} \right) |\lambda - \mu| \\ & 
        \coloneqq C_2 |\lambda - \mu|.
    \end{align*}
    Finally, taking $D_2 \coloneqq C_2 |\mu_{\textup{max}} - \mu_{\textup{min}}|$ shows that $d(\mu, \lambda)$ is bounded.
\end{proof}

\subsection{Results on the infimum}\label{app:inf}

Here, we simply state that we can split the infimum by considering unions of the optimization sets.

\begin{lemma}\label{lemma:split-optim}
    Consider $\mathbb{X} \subseteq \mathbb{R}^d$, and $f: \mathbb{X} \rightarrow \mathbb{R}$. Let $\mathbb{X}_1, \mathbb{X}_2: \mathbb{X}_1 \cup \mathbb{X}_2 = \mathbb{X}$. Then, it holds that:
    \begin{align*}
        \inf_{x \in \mathbb{X}} f(x) = \min\left\{ \inf_{x \in \mathbb{X}_1} f(x), \inf_{x \in \mathbb{X}_2} f(x)\right\}.
    \end{align*}
\end{lemma}

\subsection{Correspondences}\label{app:corresp}

We start with some preliminary definitions. 

Let $\mathbb{U}$ be a topological space and let $U \in \mathbb{U}$ s.t. $U \ne \emptyset$. A function $f: U \rightarrow \mathbb{R}$ is inf-compact on $U$ if the level sets $$\mathcal{D}_f(\lambda, U) = \{ y \in U: f(y) \le \lambda \}$$ are compact for all $\lambda \in \mathbb{R}$. Furthermore, it is upper semicontinuous if all the strict level sets $$\mathcal{D}^{<}_f(\lambda, U)= \{y \in U: f(y) < \lambda \}$$ are open.

Next, consider a correspondence $\phi: \mathbb{X} \rightrightarrows \mathbb{Y}$. Then, for any $Z \subseteq \mathbb{X}$, let $$\Gr_{Z}(\phi)=\{ (x,y) \in Z \times \mathbb{Y}: y \in \phi(x) \}.$$
Consider a function $f: \mathbb{X} \times \mathbb{Y} \rightarrow \mathbb{R}$. Then, $f$ is $\mathbb{K}$-inf-compact on $\Gr_{\mathbb{X}}(\phi)$, if all for all compact subsets $K$ of $\mathbb{X}$, it holds that $f$ is inf-compact on $\Gr_K(\phi)$.

We now state Berge's Maximum Theorem \citep{berge1959espaces}. 

\begin{theorem}[Berge's Maximum Theorem]\label{thm:berge}
Let $\mathbb{X}, \mathbb{Y}$ be Haussdorf topological spaces. Let $f: \mathbb{X} \times \mathbb{Y} \rightarrow \mathbb{R}$ be a continuous function, and let $\phi: \mathbb{X} \rightrightarrows \mathbb{Y}$ be a continuous and compact-valued correspondence. Then, let $f^*: \mathbb{X} \rightarrow \mathbb{R}$ and $\phi^*: \mathbb{X} \rightrightarrows \mathbb{Y}$ be defined as follows:
\begin{align*}
    & f^*(x) = \max_{y \in \phi(x)} f(x,y) \\
    & \phi^*(x) = \argmax_{y \in \phi(x)} f(x,y).
\end{align*}
Then, $f^*$ is continuous over $\mathbb{X}$ and $\phi^*$ is upper hemicontinuous and compact-valued over $\mathbb{X}$.    
\end{theorem}

Next, we introduce the following result that extends Berge's maximum theorem to non-compact and only lower hemicontinuous correspondences. 

\begin{theorem}[{\cite[Theorem 1.2]{feinberg2014berge}}]\label{thm:feinberg}
    Let $\mathbb{X}$ be a compactly generated topological space and $\mathbb{Y}$ be a Hausdorff topological space. Let $\phi: \mathbb{X} \rightrightarrows \mathbb{Y}$ be a lower hemicontinuous correspondence, and let $f: \mathbb{X} \times \mathbb{Y} \rightarrow \mathbb{R}$ be $\mathbb{K}$-inf-compact and upper semi-continuous on $\Gr_{\mathbb{X}}(\phi)$. Then, let $f^*: \mathbb{X} \rightarrow \mathbb{R}$ and $\phi^*: \mathbb{X} \rightrightarrows \mathbb{Y}$ be defined as follows:
    \begin{align*}
        & f^*(x) = \sup_{y \in \phi(x)} f(x,y) \\
        & \phi^*(x) = \argmax_{y \in \phi(x)} f(x,y).
    \end{align*}
    Then, $f^*$ is continuous and and $\phi^*$ is upper hemicontinuous and compact-valued. 
\end{theorem}

We recall that, when dealing with topological subspaces $\mathbb{X}$ of $\mathbb{R}^d$ with the inherited euclidean topology, then $\mathbb{X}$ are metric topological spaces, and hence, Haussdorf and compactly generated. Therefore, \Cref{thm:berge} and \Cref{thm:feinberg} can be applied. 

Finally, we report the following results that we use to prove the continuity and compactness of (some) correspondence within our analysis.

\begin{theorem}[{\cite[Proposition 1.4.14]{aubin1999set}}]\label{th:composition_correspondence}
    Let $\mathbb{X}\subseteq \Reals^{n_x}$, $\mathbb{Y}\subseteq \Reals^{n_y}$ and $\mathbb{Z}\subseteq \Reals^{n_z}$ be three sets and let $U:\mathbb{X}\rightrightarrows \mathbb{Z}$ be a compact-valued continuous correspondence. Then let $g:\textnormal{Graph}(U)\to\mathbb{Y}$ be a continuous function. Then the correspondence $G:\mathbb{X}\rightrightarrows\mathbb{Y}$ defined as $G(x)=\cup_{u\in U(x)}\{g(x,u)\}$ is continuous and compact valued.
\end{theorem}

\begin{theorem}[{\cite[Proposition 1.4.8]{aubin1999set}}]\label{th:uhc-closed-graph}
    The graph of an upper hemicontinuous set-valued map $F: \mathbb{X} \mapstoto \mathbb{Y}$ with closed domain and closed values is closed. The converse is true if we assume that $\mathbb{Y}$ is compact.
\end{theorem}

\end{document}